\PassOptionsToPackage{table,xcdraw,usenames,dvipsnames}{xcolor}
\documentclass{article}

\usepackage{microtype}
\usepackage{graphicx}
\usepackage{subcaption}
\usepackage{booktabs} % for professional tables

\usepackage{hyperref}

\usepackage[preprint]{icml2026}

\usepackage{amsmath}
\usepackage{amssymb}
\usepackage{mathtools}
\usepackage{amsthm}

\usepackage[capitalize,noabbrev]{cleveref}
\usepackage[ruled,algo2e]{algorithm2e}
\renewcommand{\algorithmicrequire}{\textbf{Input:}}  % Use Input in the format of Algorithm
\renewcommand{\algorithmicensure}{\textbf{Output:}} % Use Output in the format of Algorithm
\SetKwInOut{Input}{Input}\SetKwInOut{Output}{Output}
\theoremstyle{plain}
\newtheorem{theorem}{Theorem}[section]

\newtheorem{axiom}[theorem]{Axiom}
\newtheorem{corollary}[theorem]{Corollary}
\theoremstyle{definition}
\newtheorem{definition}[theorem]{Definition}

\theoremstyle{remark}

\usepackage[textsize=tiny]{todonotes}

\usepackage{xspace}
\newcommand{\ourmethod}{{\fontfamily{lmtt}\selectfont \textbf{ColParse}}\xspace}

\usepackage[normalem]{ulem}
\usepackage{bbm}
\usepackage{bbding}

\definecolor{ForestGreen}{RGB}{34,139,34}
\definecolor{myyellow}{RGB}{181, 181, 27}

\newcommand{\blue}[1]{$_{\color{BlueGreen}\downarrow #1}$}
\newcommand{\red}[1]{$_{\color{RedOrange}\uparrow #1}$}
\newcommand{\blueup}[1]{$_{\color{BlueGreen}\uparrow #1}$}
\definecolor{darksalmon}{rgb}{0.91, 0.59, 0.48}
\definecolor{emerald}{rgb}{0.31, 0.78, 0.47}
\definecolor{green(pigment)}{rgb}{0.0, 0.65, 0.31}
\definecolor{amaranth}{rgb}{0.9, 0.17, 0.31}
\definecolor{iris}{rgb}{0.35, 0.31, 0.81}
\definecolor{uu}{rgb}{0.95, 0.51, 0.51}
\definecolor{spirodiscoball}{rgb}{0.06, 0.75, 0.99}

\usepackage{multirow}
\usepackage{makecell}
\usepackage{enumerate}
\usepackage{enumitem}
\usepackage{pifont}

\usepackage{algorithm2e}

\usepackage[table]{xcolor}
\usepackage{graphicx}
\usepackage{booktabs}
\usepackage{makecell}
\usepackage{amsmath}
\usepackage{amssymb}
\usepackage{multirow}

\usepackage{longtable}
\usepackage{longtable}
\usepackage{rotating} % For rotating column headers
\usepackage{cuted} % Alternative to \onecolumn for full-width strips

\icmltitlerunning{Beyond the Grid: Layout-Informed Multi-Vector Retrieval with Parsed Visual Document Representations}

\begin{document}

\twocolumn[
  \icmltitle{Beyond the Grid: Layout-Informed Multi-Vector Retrieval \\ with Parsed Visual Document Representations}

  % It is OKAY to include author information, even for blind submissions: the
  % style file will automatically remove it for you unless you've provided
  % the [accepted] option to the icml2026 package.

  % List of affiliations: The first argument should be a (short) identifier you
  % will use later to specify author affiliations Academic affiliations
  % should list Department, University, City, Region, Country Industry
  % affiliations should list Company, City, Region, Country

  % You can specify symbols, otherwise they are numbered in order. Ideally, you
  % should not use this facility. Affiliations will be numbered in order of
  % appearance and this is the preferred way.
\icmlsetsymbol{equal}{*}
\icmlsetsymbol{correspondence}{\dag}

\begin{icmlauthorlist}
\icmlauthor{Yibo Yan}{aaa,bbb,ccc}
\icmlauthor{Mingdong Ou}{correspondence,bbb}
\icmlauthor{Yi Cao}{bbb}
\icmlauthor{Xin Zou}{aaa,ccc}
\icmlauthor{Shuliang Liu}{aaa,ccc} 
\icmlauthor{Jiahao Huo}{aaa,bbb} 
\icmlauthor{Yu Huang}{aaa,bbb} \\
% \icmlauthor{Philip S. Yu}{ddd}
\icmlauthor{James Kwok}{ccc}
\icmlauthor{Xuming Hu}{correspondence,aaa,ccc}
\end{icmlauthorlist}

\icmlaffiliation{aaa}{The Hong Kong University of Science and Technology (Guangzhou)}
\icmlaffiliation{bbb}{Alibaba Cloud Computing}
\icmlaffiliation{ccc}{The Hong Kong University of Science and Technology}
% \icmlaffiliation{ddd}{University of Illinois, Chicago}

\icmlcorrespondingauthor{Yibo Yan}{yanyibo70@gmail.com}
\icmlcorrespondingauthor{Mingdong Ou}{mingdong.omd@alibaba-inc.com}
\icmlcorrespondingauthor{Xuming Hu}{xuminghu@hkust-gz.edu.cn}

  % You may provide any keywords that you find helpful for describing your
  % paper; these are used to populate the "keywords" metadata in the PDF but
  % will not be shown in the document
  \icmlkeywords{Machine Learning, ICML}

  \vskip 0.3in
]

% this must go after the closing bracket ] following \twocolumn[ ...

% This command actually creates the footnote in the first column listing the
% affiliations and the copyright notice. The command takes one argument, which
% is text to display at the start of the footnote. The \icmlEqualContribution
% command is standard text for equal contribution. Remove it (just {}) if you
% do not need this facility.

% Use ONE of the following lines. DO NOT remove the command.
% If you have no special notice, KEEP empty braces:
\printAffiliationsAndNotice{}  % no special notice (required even if empty)
% Or, if applicable, use the standard equal contribution text:
% \printAffiliationsAndNotice{\icmlEqualContribution}

\begin{abstract}
  Harnessing the full potential of visually-rich documents requires retrieval systems that understand not just text, but intricate layouts, a core challenge in Visual Document Retrieval (VDR). The prevailing multi-vector architectures, while powerful, face a \textit{crucial storage bottleneck} that current optimization strategies, such as embedding merging, pruning, or introducing abstract tokens, \textit{fail to resolve without compromising performance or ignoring vital layout cues}. To address this, we introduce \ourmethod, a novel paradigm that leverages a document parsing model to generate a small set of layout-informed sub-image embeddings, which are then fused with a global page-level vector to create a compact and structurally-aware multi-vector representation. Extensive experiments demonstrate that \ourmethod reduces storage requirements by over 95\% while simultaneously yielding significant performance gains across numerous benchmarks and base models. \ourmethod thus bridges the critical gap between the fine-grained accuracy of multi-vector retrieval and the practical demands of large-scale deployment, offering a new path towards efficient and interpretable multimodal information systems.
\end{abstract}

\section{Introduction}
\label{sec:introduction}
Visual Document Retrieval (VDR), the task of retrieving relevant document pages from a large-scale corpus, has become a cornerstone of modern information retrieval \cite{mei2025survey,yan2026unlocking}. 
Unlike natural image retrieval, visual documents, such as academic papers, financial reports, and invoices, are defined by a dense interplay of textual content, intricate layouts, and graphical elements, as illustrated in \autoref{fig:visual_document_comparison}. 
To effectively capture this fine-grained detail, the field has predominantly converged on multi-vector retrieval architectures \cite{faysse2024colpali,gunther2025jina,nomicembedmultimodal2025}. 
These models represent each document page as a set of patch-level embeddings and employ a late-interaction mechanism, such as \texttt{MaxSim}, to compute relevance \cite{khattab2020colbert,santhanam2022colbertv2}. 
This paradigm excels at aligning specific query phrases with corresponding visual or textual regions within a document, a capability essential for the high-precision information-seeking tasks inherent to VDR.
% \cite{zhang2025roles,kim2025hybrid}.

\begin{figure}[t!]
    \centering
    \includegraphics[width=0.8 \linewidth]{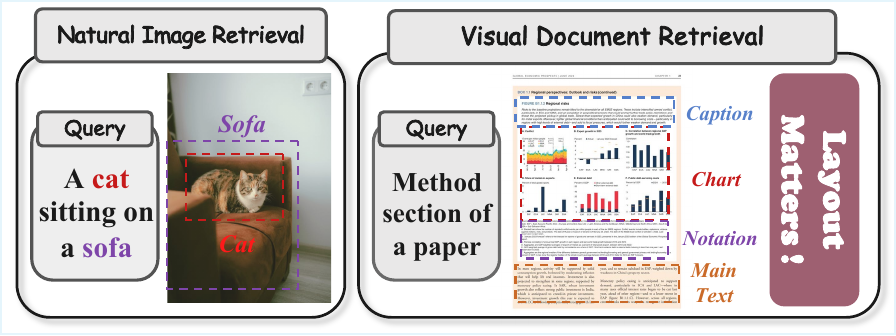}
    \vspace{-0.8em}
    \caption{Comparison of natural image retrieval versus VDR.}
    \label{fig:visual_document_comparison}
    \vspace{-8mm}
\end{figure}

Despite their superior performance, the widespread adoption of multi-vector VDR models is hindered by a critical bottleneck: \textit{prohibitive storage overhead} \cite{jayaram2024muvera,shrestha2024espn,liu2023understanding}. 
Storing hundreds or even thousands of embedding vectors for every page makes large-scale deployment practically challenging. 
To address this, the research community has explored several optimization strategies, as illustrated in \autoref{fig:optimization_paradigm_comparison}. 
\ding{182} One line of work involves \textbf{merging} patch embeddings, where methods like Light-ColPali~\cite{ma2025towards} use clustering techniques to aggregate similar vectors. 
However, this approach often leads to \textit{a dilution of fine-grained information, resulting in unstable performance}. 
\ding{183} Another direction is \textbf{pruning}, where frameworks such as DocPruner~\cite{yan2025docpruner} aim to discard redundant embeddings. 
These methods \textit{struggle to maintain performance under aggressive compression}. 
\ding{184} A third paradigm, exemplified by MetaEmbed~\cite{xiao2025metaembed}, \textbf{introduces a set of abstract, learnable tokens} to form a compact multi-vector representation. 
While innovative, these tokens lack an explicit grounding in the document's inherent layout structure, limiting their ability to capture crucial layout-specific semantics.

\begin{figure*}[t!]
    \centering
    \includegraphics[width=0.9 \linewidth]{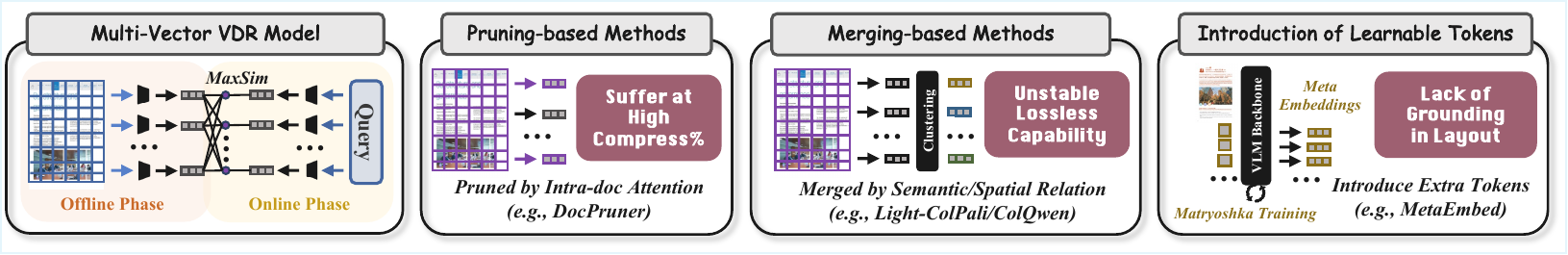}
    \vspace{-0.8em}
    \caption{The illustration of a multi-vector VDR model and three primary optimization strategies for its efficiency bottleneck.}
    \label{fig:optimization_paradigm_comparison}
    \vspace{-6mm}
\end{figure*}

To address the limitations of existing approaches, we introduce \ourmethod, \textbf{a novel paradigm for constructing multi-vector representations that is fundamentally aligned with the structural nature of visual documents}. 
Instead of operating on a uniform grid of patches or abstract tokens, \ourmethod first employs a specialized document parsing model to intelligently segment each document page into a small set of $k$ semantically meaningful, layout-informed sub-images (\textit{e.g.,} tables, figures, paragraphs), where $k$ is typically less than 10. 
These $k$ sub-images are then individually encoded by a standard single-vector retrieval model to yield $k$ local vectors. 
In parallel, the entire document page is encoded to generate one global vector that captures the overall context. 
Finally, we fuse these representations by weighted element-wise adding the global vector to each of the $k$ local vectors. 
This process results in $k$ fused vectors for each document, which integrate both fine-grained, layout-specific details and holistic page-level context.

We conducted comprehensive experiments on 24 diverse VDR datasets \cite{meng2025vlm2vec} to validate the effectiveness and robustness of our proposed framework. 
\ourmethod consistently delivers substantial performance improvements, achieving an average gain of over 10 points in nDCG@5 when applied to 10 different mainstream single-vector models. 
This demonstrates its remarkable flexibility as a training-free, plug-and-play module. 
By deeply integrating the unique structural properties of visual documents with the powerful mechanism of multi-vector retrieval, \ourmethod establishes a new trade-off between retrieval performance and storage efficiency. Our main contributions are as follows:
\begin{itemize}[leftmargin=*]
    \item[\ding{182}] \textbf{A Novel Paradigm for Multi-Vector Construction:} We introduce the first layout-informed paradigm for constructing multi-vector representations in VDR, which overcomes the storage efficiency bottleneck of conventional multi-vector models by leveraging document parsing.
    \item[\ding{183}] \textbf{A Flexible and Robust Framework:} Our method is designed as a training-free, plug-and-play framework that demonstrates robust and significant performance gains across a wide array of existing single-vector models, highlighting its versatility and ease of adoption.
    \item[\ding{184}] \textbf{Superior Performance with Enhanced Interpretability:} \ourmethod provides inherent interpretability by enabling retrieval results to be traced back to specific, parsed layout components, which significantly enhances its practicality and potential for real-world industrial applications.
\end{itemize}

\section{Related Work}
\label{sec:related_work}

\subsection{Visual Document Retrieval}
\label{sec:related_work_vdr}

VDR has become a crucial task for understanding visually-rich documents, moving beyond traditional \textbf{OCR-based} pipelines that often \textit{lose critical layout information} \cite{zhang2025ocr,most2025lost}. 
The advent of Vision-Language Models (VLMs) introduced end-to-end \textbf{single-vector} approaches (\textit{e.g.,} DSE \cite{ma2024dse}, GME \cite{zhang2024gme}, and UniSE \cite{liu2025any}), but these frequently \textit{struggle to capture the fine-grained semantics} required for dense documents. 
A significant leap forward was made with the \textbf{multi-vector} paradigm, pioneered by ColPali~\citep{faysse2024colpali}, which represents pages as numerous patch-level embeddings and employs late interaction for superior matching. 
Recent efforts have sought to optimize this paradigm at various levels: \textbf{model-level}, by exploring bidirectional architectures like ModernVBERT~\citep{teiletche2025modernvbert}; \textbf{data-level}, through advanced data synthesis and hard-negative mining as seen in works like Llama Nemoretriever Colembed~\citep{xu2025llama}; and \textbf{training-level}, via new objectives and multi-task frameworks such as jina-embeddings-v4~\citep{gunther2025jina}. 
Despite their performance, these multi-vector models introduce a \textit{severe storage bottleneck}. 
% Our proposed method, \ourmethod, specifically addresses this limitation by leveraging document parsing to generate a compact and structurally-aware multi-vector representation.

\subsection{Mutli-Vector Retrieval}
\label{sec:related_work_multi_vector}
The multi-vector paradigm, first popularized in \textbf{text retrieval} by ColBERT~\citep{khattab2020colbert}, represents documents as sets of token-level embeddings to enable fine-grained matching through a late-interaction mechanism \cite{qian2022multi,lee2023rethinking}.
This approach was further refined in the text domain by models like BGE-M3-Embedding~\citep{chen2024m3} and Jina-ColBERT-v2~\citep{jha2024jina}.
The paradigm was successfully adapted for \textbf{multimodal retrieval} by ColPali~\citep{faysse2024colpali}, shifting the focus to visual documents, which are inherently more complex than natural images.
Despite their superior performance, these models face \textit{a critical efficiency bottleneck from the prohibitive storage cost of patch-level embeddings} \cite{liu2023understanding,shrestha2024espn,park2025scv}.
Current optimization efforts fall into three main categories, each with inherent drawbacks. 
(i) \textbf{Pruning} redundant embeddings, as seen in DocPruner~\citep{yan2025docpruner} and Prune-then-Merge \cite{yan2026sculpting}, often \textit{struggles to maintain performance under aggressive compression}.
(ii) \textbf{Merging} similar embeddings via clustering, exemplified by Light-ColPali~\citep{ma2025towards}, can \textit{dilute fine-grained information, leading to unstable performance}.
(iii) \textbf{Introducing abstract, learnable tokens}, pioneered by MetaEmbed~\citep{xiao2025metaembed} and CausalEmbed \cite{huo2026causalembed}, creates compact representations that, however, \textit{lack an explicit grounding in the document’s inherent layout structure}.
In contrast, \ourmethod addresses these limitations by leveraging document parsing to generate a compact set of layout-informed embeddings.
% In contrast, \ourmethod directly addresses these limitations. 
% By leveraging document parsing to generate a compact set of layout-informed embeddings, it avoids the information dilution of merging and the performance degradation of aggressive pruning, while ensuring its representations are semantically grounded in the document’s structure, unlike abstract tokens.

\subsection{Document Parsing VLM}
\label{sec:related_work_doc_parsing_vlm}
Document parsing VLMs have emerged as critical tools for converting visually-rich document images into structured formats like LaTeX or Markdown \cite{zhang2024documentparsing,ouyang2025omnidocbench,zhang2025docr}. 
Early models, such as Nougat~\citep{blecher2023nougat} and Donut \cite{kim2022ocr}, adopted an end-to-end, sequence-to-sequence approach but often struggled with the computational cost of high-resolution inputs.
To balance accuracy and efficiency, a more recent multi-stage paradigm has gained traction.
This is exemplified by models like MinerU2.5~\citep{niu2025mineru25}, which first performs efficient layout analysis on a downsampled image before conducting targeted, high-resolution recognition on cropped regions.
This coarse-to-fine strategy, also seen in models like Dolphin~\citep{feng2025dolphin} and MonkeyOCR~\citep{zhang2025monkeyocr}, effectively mitigates the O(N\textsuperscript{2}) complexity of processing high-resolution images end-to-end.
For \ourmethod, we select MinerU2.5 as our document parser, given its state-of-the-art accuracy and efficiency. A quantitative comparison with other document parsing models will be presented in Section~\ref{sec:hyperparameter_analysis}.

\section{Methodology}
\label{sec:method}

In this section, we first formalize the task of VDR within the multi-vector paradigm. We then introduce the \ourmethod~framework, detailing its multi-stage process for generating compact, layout-informed document representations.
See our pseudo-code in \autoref{app:algo_workflow}.

\subsection{Task Formulation}
\label{sec:task_formulation}

The primary goal of VDR is, given a textual query $q$, to retrieve a ranked list of relevant document pages from a large-scale corpus $\mathcal{C} = \{d_1, d_2, \dots, d_{|\mathcal{C}|}\}$.
In the conventional multi-vector retrieval paradigm, a document page $d$ is first rendered as an image and then uniformly partitioned into a grid of $N_p$ patches, $\{p_j\}_{j=1}^{N_p}$. A VLM, serving as an encoder $\Phi(\cdot)$, maps each patch $p_j$ into a $D$-dimensional embedding, resulting in a large set of patch-level document embeddings $\mathbf{D}_{\text{grid}} = \{\mathbf{d}_j\}_{j=1}^{N_p}$, where each $\mathbf{d}_j \in \mathbb{R}^D$. Concurrently, the same encoder maps the textual query $q$ into a set of $N_q$ token-level embeddings $\mathbf{Q} = \{\mathbf{q}_i\}_{i=1}^{N_q}$, where each $\mathbf{q}_i \in \mathbb{R}^D$.
The relevance score $s(q, d)$ between the query and the document is then computed using a late-interaction mechanism, typically \texttt{MaxSim}, as defined below:
\begin{equation}
\label{eq:maxsim_grid}
s(q, d) = \sum_{i=1}^{N_q} \max_{j=1}^{N_p} (\mathbf{q}_i^\top \mathbf{d}_j).
\end{equation}
where vectors are assumed to be L2-normalized. While this grid-based approach excels at fine-grained matching, it incurs a prohibitive storage cost of $O(N_p \times D)$ per document page, as $N_p$ can be in the hundreds or thousands.

The objective of our work is to address this critical bottleneck. We aim to replace the large, layout-agnostic set $\mathbf{D}_{\text{grid}}$ with a highly compact, structurally-aware multi-vector representation $\mathbf{D}_{\text{\ourmethod}}$, which contains only $k$ vectors, where $k \ll N_p$. This new representation should significantly reduce storage requirements to $O(k \times D)$ while simultaneously enhancing retrieval performance by being explicitly grounded in the document's semantic layout.

\begin{figure}[!t]
  \centering
  \includegraphics[width=\linewidth]{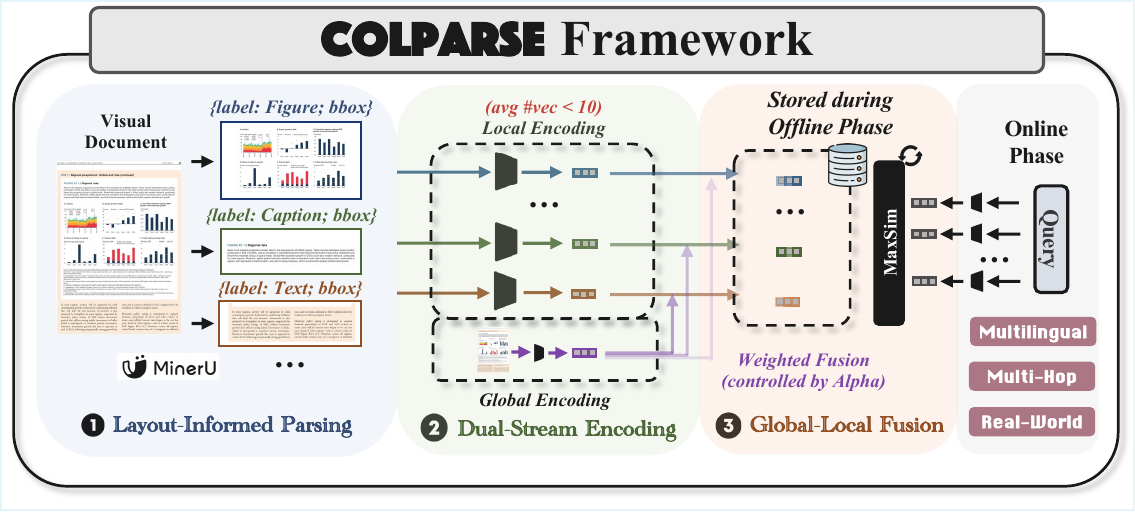}
  \vspace{-1em}
  \caption{The simplified illustration of \ourmethod framework.}
   \label{fig:colparse_framework}
   \vspace{-1.3em}
\end{figure}

\subsection{The \ourmethod Framework}
\label{sec:colparse_framework}

\ourmethod~is a plug-and-play framework that revolutionizes the construction of multi-vector representations by moving ``beyond the grid.'' Instead of relying on uniform patches, it leverages structural understanding to generate a compact and semantically rich set of embeddings. As shown in \autoref{fig:colparse_framework}, our framework operates offline in a three-stage pipeline for each document image $d \in \mathbb{R}^{H \times W \times 3}$: (1) Layout-Informed Document Parsing, (2) Dual-Stream Encoding, and (3) Global-Local Fusion.

\subsubsection{Layout-Informed Document Parsing}
\label{sec:stage1_parsing}

The foundational step of \ourmethod~is to deconstruct the document image into its constituent semantic components. We employ a specialized, off-the-shelf document parsing model, $\Psi_{\text{parse}}(\cdot)$, which functions as a layout detector. For a given document image $d$, the parser identifies a set of $k$ distinct layout regions, outputting their corresponding bounding boxes and content type labels: $[\{b_j, c_j\}_{j=1}^k] = \Psi_{\text{parse}}(d).$
Here, $b_j = (x_{j1}, y_{j1}, x_{j2}, y_{j2})$ is the bounding box for the $j$-th region, defining its coordinates within the original image. $c_j$ is a categorical label from a predefined set of content types $\mathcal{C}_{\text{types}} = \{\text{`title'}, \text{`table'}, \text{`figure'}, \dots\}$, indicating the semantic nature of the region.
Using these bounding boxes, we crop the original image $d$ to extract a set of $k$ sub-images $\mathcal{S}_d$. The number of sub-images, $k$, is dynamically determined by the parser based on the document's complexity and is typically very small (\textit{e.g.,} $k < 10$). The process can be formulated as $\mathcal{S}_d = \{s_1, s_2, \dots, s_k\} \quad \text{where } s_j = \text{Crop}(d, b_j).$
Each sub-image $s_j \in \mathbb{R}^{H_j \times W_j \times 3}$ is of variable size, flexibly conforming to the dimensions of the detected layout component. This intelligent, content-aware segmentation forms the basis for our compact representation.

\subsubsection{Dual-Stream Encoding}
\label{sec:stage2_encoding}

Once the document is parsed, we generate embeddings using a standard single-vector retrieval model, $\Phi_{\text{enc}}(\cdot): \mathbb{R}^{H' \times W' \times 3} \to \mathbb{R}^D$, which serves as the base encoder. This stage operates in two parallel streams to capture both local, layout-specific details and global, page-level context.

\paragraph{Local Encoding.} Each of the $k$ variable-sized sub-images $s_j$ from $\mathcal{S}_d$ is resized and independently passed through the encoder $\Phi_{\text{enc}}(\cdot)$ to produce a corresponding $D$-dimensional local vector, $\mathbf{v}_{\text{local}}^{(j)}$. This process yields a set of $k$ local embeddings, each representing a distinct semantic unit: $\mathbf{D}_{\text{local}} = \{\mathbf{v}_{\text{local}}^{(j)} \in \mathbb{R}^D \mid \mathbf{v}_{\text{local}}^{(j)} = \Phi_{\text{enc}}(s_j), \forall s_j \in \mathcal{S}_d\}_{j=1}^k.$

\paragraph{Global Encoding.} In parallel, the \textit{entire, un-cropped} document page image $d$ is passed through the same encoder $\Phi_{\text{enc}}(\cdot)$ to generate a single $D$-dimensional global vector, $\mathbf{v}_{\text{global}}$. This vector serves as a holistic summary of the page, capturing the overall context and relationships between layout components: $\mathbf{v}_{\text{global}} = \Phi_{\text{enc}}(d) \in \mathbb{R}^D.$
This dual-stream design ensures our final representation benefits from both the specificity of individual layout elements and the broader context of the entire page.

\subsubsection{Global-Local Fusion for Final Representation}
\label{sec:stage3_fusion}

% The final offline stage is crucial: it fuses the local, layout-specific representations with the global, page-level context to form a cohesive and powerful multi-vector set. 
Instead of a simple summation, we employ a weighted fusion strategy that allows for a tunable balance between these two critical streams of information.
For each of the $k$ local vectors, $\mathbf{v}_{\text{local}}^{(j)}$, we inject the holistic context captured by the single global vector, $\mathbf{v}_{\text{global}}$, using a balancing factor $\alpha \in [0, 1]$. 
The resulting fused vector, $\mathbf{d}_{\text{fused}}^{(j)} \in \mathbb{R}^D$, synergistically integrates both fine-grained detail and overarching context. 
The fusion is performed as a weighted element-wise addition: $\mathbf{d}_{\text{fused}}^{(j)} = \alpha \cdot \mathbf{v}_{\text{global}} + (1-\alpha) \cdot \mathbf{v}_{\text{local}}^{(j)}, \quad \forall j \in \{1, \dots, k\}.$
% This hyperparameter $\alpha$ acts as a control knob governing the contribution of each component. 
% A higher value of $\alpha$ emphasizes the global context, making the representation more attuned to the overall theme and structure of the page. 
% This can be advantageous for broad, topic-level queries. 
% Conversely, a lower value of $\alpha$ prioritizes the specific, fine-grained information from the sub-image, enhancing precision for highly localized queries targeting a specific table, figure, or paragraph. 
This weighted mechanism provides the flexibility to tailor the representation's focus.
% , a capability that simple addition lacks.

This fusion operation is performed for all $k$ local vectors, producing the final multi-vector representation for document $d$, denoted as $\mathbf{D}_{\text{\ourmethod}} = \{\mathbf{d}_{\text{fused}}^{(j)}\}_{j=1}^k \subset \mathbb{R}^{k \times D}$.
This set, containing only $k$ structurally-aware and context-enriched vectors, is then stored for online retrieval, achieving our goal of massive storage reduction.

\subsubsection{Late-Interaction Scoring with \ourmethod}
\label{sec:scoring}

During the online retrieval phase, the relevance score between a query $q$ and a document $d$ is computed efficiently using the compact representation $\mathbf{D}_{\text{\ourmethod}}$. The query is first encoded into its token-level embeddings $\mathbf{Q} = \{\mathbf{q}_i \in \mathbb{R}^D\}_{i=1}^{N_q}$ as standard. The \texttt{MaxSim} score is then calculated over the $k$ fused document vectors:
\vspace{-2mm}
\begin{equation}
\label{eq:maxsim_colparse}
s_{\text{\ourmethod}}(q, d) = \sum_{i=1}^{N_q} \max_{j=1}^{k} (\mathbf{q}_i^\top \mathbf{d}_{\text{fused}}^{(j)}).
\end{equation}
\vspace{-1mm}
By replacing the search over $N_p$ grid-based vectors with a search over just $k$ layout-informed vectors, \ourmethod~not only dramatically reduces the storage footprint but also focuses the late-interaction mechanism on the most semantically salient parts of the document.

\subsection{Theoretical Foundation}
\label{sec:theorectical_foundation}

We provide a theoretical justification for \ourmethod~from an Information Bottleneck (IB) perspective. We demonstrate that the framework's architecture serves as a principled surrogate for solving the intractable IB objective in VDR by (1) disentangling source information via parsing and (2) refining it with contextual side-information.

\subsubsection{The VDR Compression Problem as an Information Bottleneck}

Let $D$ be a random variable representing a document image and $Q$ be a random variable for a query, drawn from an unknown distribution $P(Q)$. Let $R$ be a relevance variable, a function of $D$ and $Q$. The goal is to learn a compression function $g$ that maps a document $D$ to a compact representation $Z = g(D)$ by solving the IB Lagrangian:
\begin{equation}
    \min_{g} \mathcal{L}(Z) = I(Z; D) - \beta \mathbb{E}_{Q}[I(Z; R(D, Q))]
    \label{eq:ib_lagrangian}.
\end{equation}
This objective seeks to minimize the information $Z$ retains about the source $D$ (compression) while maximizing the information it preserves about the relevance $R$ (prediction). The expectation $\mathbb{E}_{Q}$ over the unknown query distribution makes this problem intractable at indexing time.

\subsubsection{Information Disentanglement via Parsing}

\ourmethod's first stage, parsing, transforms the input space. Let $\Psi_{\text{parse}}(D) = \{S_1, \dots, S_k\}$ be the set of sub-images (semantic regions) derived from document $D$. These regions form a partition of the document's core semantic content. By the chain rule of mutual information, the information in the original document can be expressed through its components: $I(D; R) = I(S_1, S_2, \dots, S_k; R)$.
We posit the \textbf{Semantic Concentration Axiom}: for any given query $Q$, the relevance signal $R$ is predominantly determined by a single primary semantic region $S_{j^*} \in \{S_j\}$, where $j^*$ is the index of the most relevant region. This implies near conditional independence for the remaining regions: $I(S_{\neg j^*}; R | S_{j^*}) \approx 0, \quad \text{where } S_{\neg j^*} = \{S_j\}_{j \neq j^*}.$
This axiom leads to the approximation: $I(D; R) \approx I(S_{j^*}; R).$
This decomposition transforms the problem from compressing monolithic variable $D$ to creating a set of representations, one for each potential primary channel $S_j$. This provides a structural prior to IB problem, justifying creation of a multi-vector set $\{g_j(S_j)\}_{j=1}^k$ instead of a single vector $g(D)$.

\subsubsection{Contextual Refinement via Synergistic Fusion}

\ourmethod~generates two sets of intermediate representations: local vectors $\{V_j = \Phi_{\text{enc}}(S_j)\}_{j=1}^k$ and a global vector $V_{\text{global}} = \Phi_{\text{enc}}(D)$. Each of these encoding steps is itself an information bottleneck, subject to the Data Processing Inequality (DPI): $I(V_j; R) \le I(S_j; R),\quad \forall j \in \{1, \dots, k\}$; $I(V_{\text{global}}; R) \le I(D; R)$.
The core of \ourmethod~is the fusion step, which creates the final representation set $\{Z_j = V_j + V_{\text{global}}\}_{j=1}^k$. To analyze its benefit, consider the joint information held by the local-global pair $(V_j, V_{\text{global}})$ about the relevance $R$. Using the chain rule:
\begin{equation}
    I(V_j, V_{\text{global}}; R) = I(V_j; R) + \underbrace{I(V_{\text{global}}; R | V_j)}_{\text{\textbf{Contextual Information Gain}}}.
    \label{eq:contextual_gain}
\end{equation}
This gain term quantifies the new information about relevance that the global context provides, given the local region. It is non-zero if global context helps disambiguate local content.
The fusion function $f(V_j, V_{\text{global}}) = V_j + V_{\text{global}}$ creates the vector $Z_j$. By DPI, this fusion is also a bottleneck:
\begin{equation}
    I(Z_j; R) = I(V_j + V_{\text{global}}; R) \le I(V_j, V_{\text{global}}; R).
\end{equation}
The objective of this fusion is to craft a compact vector $Z_j$ that is more informative than the local vector $V_j$ alone, by capturing a significant portion of the contextual gain. The net improvement in information for region $j$ is:
\begin{equation}
    \Delta I_j = I(Z_j; R) - I(V_j; R).
\end{equation}
The success of \ourmethod~relies on this fusion being effective, \textit{i.e.,} ensuring $\Delta I_j > 0$. This holds when the fusion operation successfully encodes the contextual information gain. The simple vector addition serves as a parameter-free, computationally efficient mechanism to achieve this. The final representation $\mathbf{D}_{\text{\ourmethod}} = \{Z_j\}_{j=1}^k$ is thus a set of contextually-refined, disentangled vectors that more effectively preserve query-relevant information, providing a superior solution to the VDR compression problem.

% See more theoretical analysis in Appendix \ref{app:more_theoretical_analysis}.
See more theoretical analysis in \autoref{app:more_theoretical_analysis}.

\section{Experiment}
\label{sec:experiment}

\subsection{Experimental Setup}
\label{sec:experimental_setup}

\begin{figure*}[!ht]
  \centering
  \includegraphics[width=0.8\linewidth]{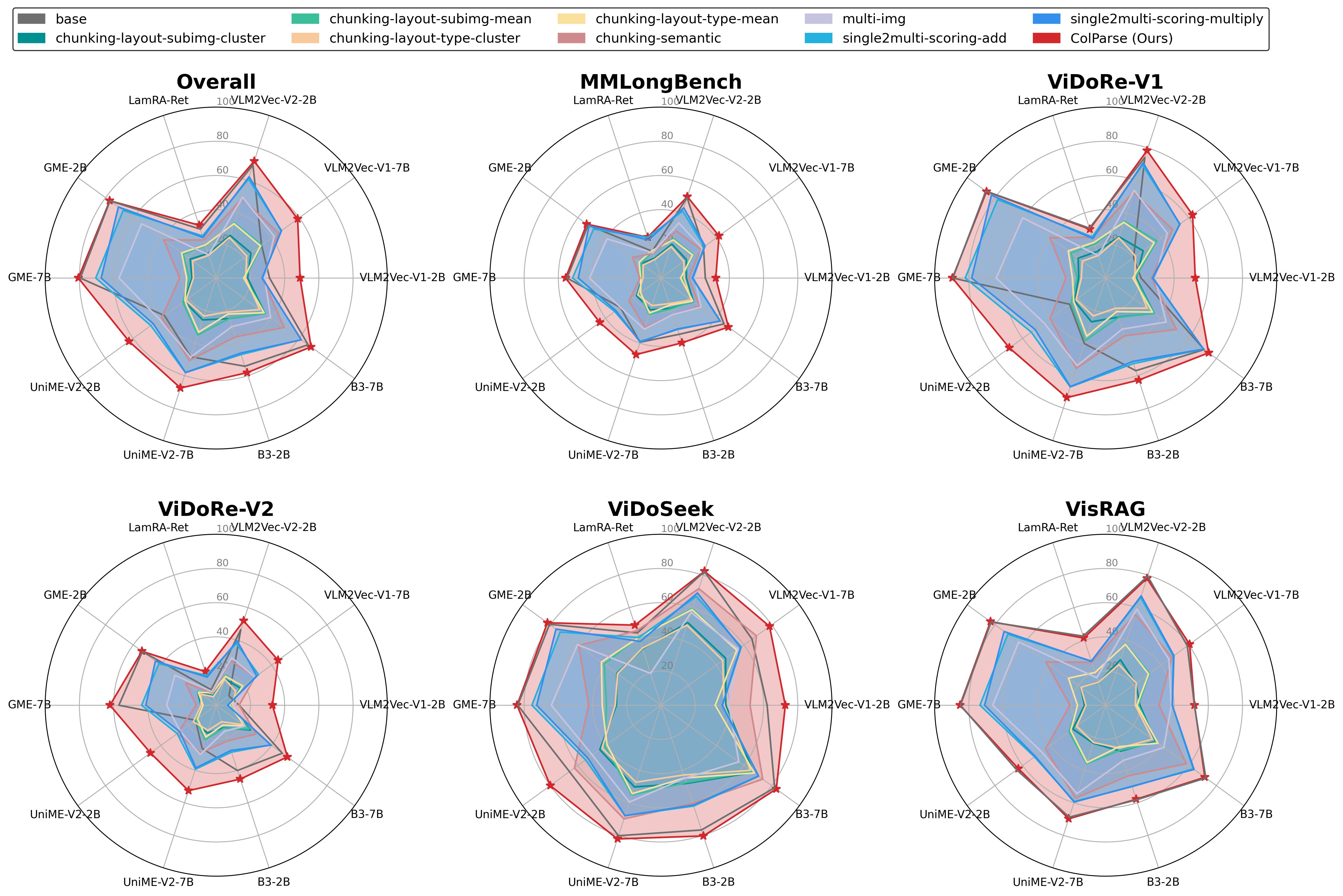}
  % \vspace{-2em}
  \caption{The performance comparison (evaluated by nDCG@5) between \ourmethod and baselines on five VDR benchmarks across ten mainstream single-vector multimodal retrieval models. Refer to \autoref{tab:full_results_mmlongbench_vidorev1} and \autoref{tab:full_results_vidorev2_vidoseek_visrag} for detailed result records due to the space limit.}
   \label{fig:main_result_visualization}
   \vspace{-1em}
\end{figure*}

\paragraph{Benchmarks and Evaluation.}
To ensure a comprehensive evaluation, we assess the performance of \ourmethod across five mainstream VDR benchmark suites, encompassing a total of 24 diverse datasets: \textbf{ViDoRe-V1} \cite{faysse2024colpali}, \textbf{ViDoRe-V2} \cite{mace2025vidorev2}, \textbf{VisRAG} \cite{yu2024visrag}, \textbf{ViDoSeek} \cite{wang2025vidorag}, and \textbf{MMLongBench} \cite{ma2024mmlongbench}. We validate the versatility and plug-and-play nature of our framework by applying it to ten prominent single-vector retrieval models: \textbf{VLM2Vec-V1-2B/7B} \cite{jiang2024vlm2vec}, \textbf{VLM2Vec-V2-2B} \cite{meng2025vlm2vec}, \textbf{LamRA-Ret} \cite{liu2025lamra}, \textbf{GME-2B/7B} \cite{zhang2024gme}, \textbf{UniME-V2-2B/7B} \cite{gu2025unimev2}, \textbf{B3-2B/7B} \cite{thirukovalluru2025breaking}. See details of benchmarks and models in Appendix \ref{app:benchmark_details} and \ref{app:model_details}, respectively. In alignment with established VDR practices \cite{wasserman2025real,Quentin2025vidorev3}, we use the Normalized Discounted Cumulative Gain at 5 (nDCG@5) as the evaluation metric.
\vspace{-4mm}
\paragraph{Baselines.}
To demonstrate the superiority of \ourmethod, we compare it against a diverse set of baselines organized into five distinct categories. 
\textbf{\ding{182} Base:} This category represents the original performance of the ten single-vector models without any multi-vector adaptation. 
\textbf{\ding{183} Multi-img:} In this approach, all sub-images parsed from a document are fed simultaneously into the base models, leveraging their native support for multi-image inputs to generate a single representative vector.
\textbf{\ding{184} Chunking-layout:} This category explores strategies that operate directly on the final layer of token embeddings from the base model. Guided by the layout parsing results, these tokens are chunked and aggregated, while any tokens outside the parsed bounding boxes are discarded. We evaluate four variants: \texttt{type-cluster} (tokens of the same content type are merged via semantic clustering), \texttt{type-mean} (tokens of the same content type are merged via mean pooling), \texttt{subimg-cluster} (tokens from the same sub-image region are clustered), and \texttt{subimg-mean} (tokens from the same sub-image region are pooled).
\textbf{\ding{185} Chunking-semantic:} In contrast to layout-guided methods, this baseline performs hierarchical clustering on the entire set of final-layer tokens to generate a fixed number of vectors (defaulting to 10), operating agnostically to the document's layout structure.
\textbf{\ding{186} Single2multi:} This category mimics \ourmethod by encoding each parsed sub-image separately but employs alternative scoring mechanisms instead of our vector fusion. The two variants are \texttt{scoring-add} and \texttt{scoring-multiply}, where the final document score is computed by either adding or multiplying the individual query-sub\_img similarity scores with the global\_img-sub\_img similarity scores, respectively.

\paragraph{Implementation Details.}
To ensure a fair comparison and reproducibility, our entire evaluation pipeline is built upon the MMEB codebase\footnote{\url{https://github.com/TIGER-AI-Lab/VLM2Vec}}. We employ MinerU2.5\footnote{\url{https://github.com/opendatalab/MinerU}} \cite{niu2025mineru25} as the unified document parsing model across all experiments due to its state-of-the-art performance and efficiency (See details of MinerU2.5 in Appendix \ref{app:mineru25_details}). For \ourmethod, we choose $\alpha$ ranging from 0.1 to 0.9 with an interval of 0.1, and we select the optimal hyperparameter for each base model to report the final results (validated in Section \ref{sec:hyperparameter_analysis}). Experiments were conducted on a cluster of NVIDIA A100 (80G) GPUs. The complete codebase will be made publicly available upon acceptance.

\subsection{Experimental Analysis}
\label{sec:experimental_analysis}

\subsubsection{Main Result}
\label{sec:main_result}

\textbf{\ourmethod consistently outperforms both single-vector base models and existing multi-vector optimization baselines across diverse benchmarks.} 
As shown in \autoref{tab:full_results_mmlongbench_vidorev1} and \autoref{tab:full_results_vidorev2_vidoseek_visrag} (Appendix \ref{app:main_results}), \ourmethod achieves a remarkable average nDCG@5 gain of 31.64 points for VLM2Vec-V1-2B and 42.69 points for its 7B counterpart on the ViDoRe-V1 benchmark. 
\autoref{fig:main_result_visualization} further visualizes this superiority, where the red envelope representing \ourmethod consistently forms the outermost boundary across all ten mainstream models. 
This suggests that layout-informed sub-image representations capture critical fine-grained details that are typically ``diluted'' in monolithic global embeddings.

\textbf{The framework demonstrates exceptional versatility and robustness as a training-free, plug-and-play module for various VLM-based embeddings.} 
Across ten distinct models including VLM2Vec, GME, UniME, and B3, \ourmethod consistently yields performance improvements regardless of the model's architecture or parameter scale. 
Notably, even for the high-performing GME-7B model, \ourmethod maintains state-of-the-art results while other optimization strategies like semantic chunking (\texttt{c-sem}) cause a drastic performance drop from 89.36 to 23.21 on ViDoRe-V1. 
This universality implies that layout awareness is a fundamental, model-agnostic enhancement for visual document understanding that can be unlocked at representation level.

\textbf{Layout-informed decomposition is substantially more effective for multi-vector construction than traditional token-level chunking or clustering.} 
Quantitative results reveal that baselines like \texttt{cl-t-c} (type-clustering) and \texttt{cl-s-m} (sub-image mean pooling) often lead to significant performance degradation; for instance, VLM2Vec-V2-2B's performance plummets from 74.16 to 24.85 when using token-level clustering. 
In contrast, \ourmethod boosts the same model to 78.41 by preserving the visual integrity of semantic regions rather than aggregating abstract token embeddings. 
This phenomenon indicates that maintaining the raw visual-semantic alignment within parsed regions is superior to post-hoc heuristic aggregation of late-stage features.

\textbf{\ourmethod exhibits superior efficacy in handling complex, long-form documents that require multi-hop reasoning.} 
On the MMLongBench dataset (\autoref{tab:full_results_mmlongbench_vidorev1}), \ourmethod elevates the average nDCG@5 of VLM2Vec-V1-2B from 25.93 to 32.07 and UniME-V2-2B from 29.31 to 44.21, outperforming all other compression baselines. 
This performance leap is particularly evident in the cases which require cross-page information locating. 
We speculate that by intelligently segmenting pages into key layout components (\textit{e.g.,} tables, figures), \ourmethod reduces the ``signal-to-noise'' ratio during the late-interaction phase, allowing the model to focus on the most semantically salient regions.

\textbf{The global-local fusion strategy is critical for providing necessary contextual grounding to isolated semantic regions.} 
Comparative analysis with the \texttt{single2multi} (\texttt{s2m-add/mul}) baselines in \autoref{tab:full_results_vidorev2_vidoseek_visrag} shows \ourmethod consistently leads on challenging tasks like VisRAG and ViDoSeek, achieving 51.96 on VisRAG with VLM2Vec-V1-2B compared to only 38.94 for \texttt{s2m-add}. 
This significant gap highlights local sub-images alone often lack the holistic context (\textit{e.g.,} a table without its preceding paragraph's context) needed for accurate retrieval. 
The synergy achieved through weighted fusion allows \ourmethod to retain both regional specificity and page-level semantics.

\subsubsection{Variant Study}
\label{sec:variant_study}

\begin{figure}[!t]
  \centering
  \includegraphics[width=0.9\linewidth]{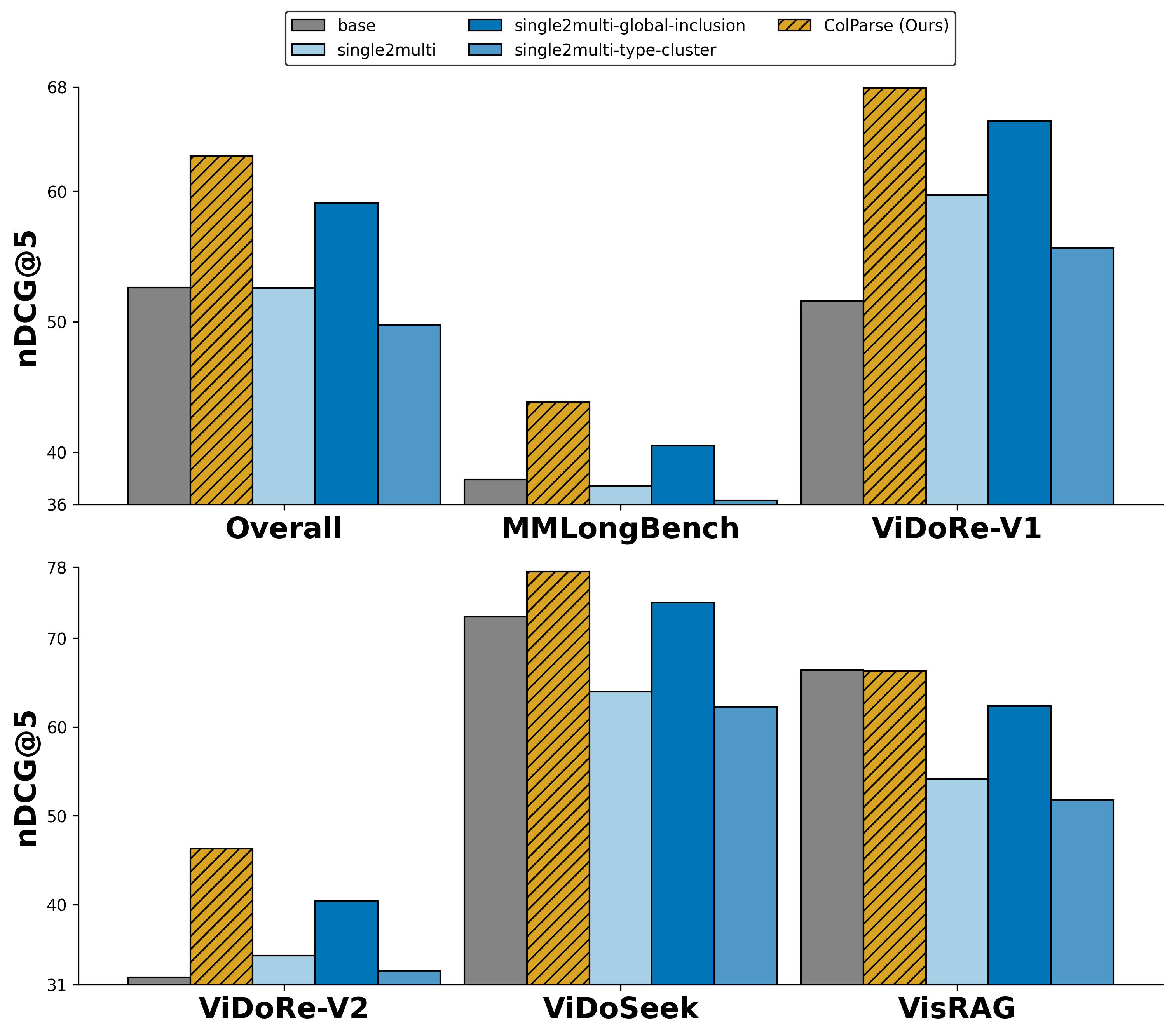}
  % \vspace{-1em}
  \caption{Variant study of \ourmethod and its variants. }
   \label{fig:variant_bar}
   \vspace{-6mm}
\end{figure}

We evaluate three specific variants: (i) \textbf{single2multi} (\texttt{s2m}) variant serves as the baseline layout-decomposed representation, which extracts and encodes sub-images as independent vectors without any global information. (ii) \textbf{single2multi-type-cluster} (\texttt{s2m-t-c}) variant extends \texttt{s2m} by aggregating all sub-image embeddings of the same semantic category into a single representative type-level vector to further compress the representation. (iii) \textbf{single2multi-global-inclusion} (\texttt{s2m-g-i}) variant appends the original holistic page-level embedding to the \texttt{s2m} sub-image set as a separate global vector to provide page-wide context.

\textbf{Synergistic global-local fusion is significantly more effective than simple global vector inclusion for contextualizing layout-aware representations.} As shown in \autoref{fig:variant_bar}, \ourmethod consistently outperforms the \texttt{s2m-g-i} variant across all ten base models and benchmarks, for instance, achieving a gain of 2.44 points in nDCG@5 over \texttt{s2m-g-i} for VLM2Vec-V1-2B on ViDoRe-V1. This superiority is further visualized in the radar plots of \autoref{fig:variant_result_visualization} (Appendix \ref{app:variant_study}), where \ourmethod (red envelope) consistently covers the largest area, particularly in dense tasks like VisRAG where it leads \texttt{s2m-g-i} by over 2 points on the 7B model. This suggests that element-wise fusion serves as a deeper conditioning mechanism than simple inclusion, allowing local features to be fundamentally ``re-weighted'' by the global semantic environment of the document.

See more analysis in Appendix \ref{app:variant_study} due to the space limit.

\begin{figure}[!t]
  \centering
  \includegraphics[width=0.8\linewidth]{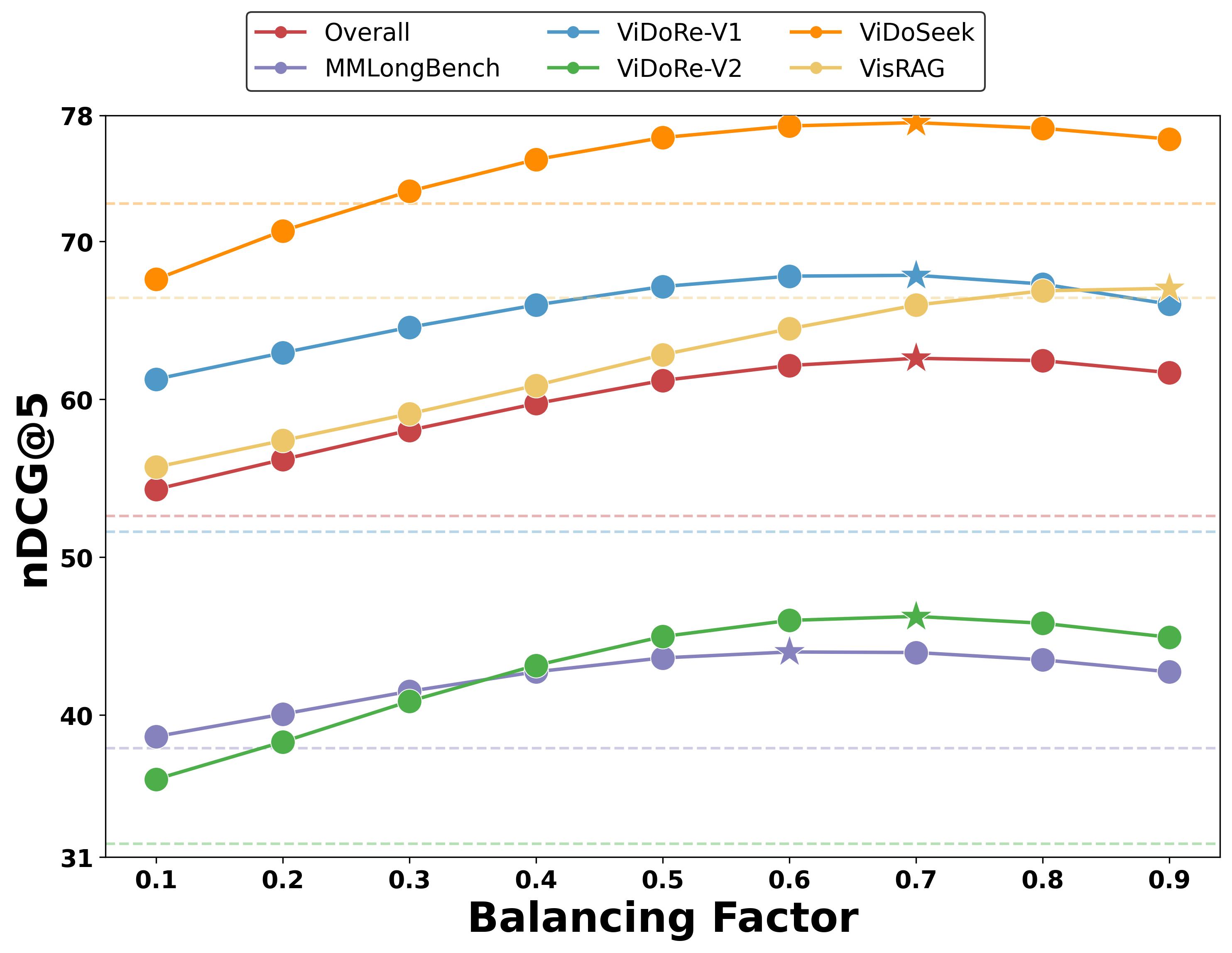}
  % \vspace{-1em}
  \caption{The comparison of the \textit{average} performance of \ourmethod across different balancing factors. The dash lines refer to the base results; and the star points refer to the best-performing balancing factors. See \autoref{fig:balancing_factor_all_line} for the model-level comparisons.}
   \label{fig:balancing_factor_average_line}
   \vspace{-6mm}
\end{figure}

\subsubsection{Hyperparameter Analysis}
\label{sec:hyperparameter_analysis}

\paragraph{Effect of balancing factor.} 
We investigate the sensitivity of \ourmethod to the balancing factor $\alpha$, as illustrated in \autoref{fig:balancing_factor_average_line} (details in Appendix \ref{app:hyperparameter_balancing_factor}). 
First, the retrieval performance is robust across the entire range of $\alpha \in [0.1, 0.9]$, consistently surpassing the single-vector baseline for all tested models. 
For instance, in VLM2Vec-V1-2B, even the lowest $\alpha$ of $0.1$ achieves an overall score of $35.97$, which is still a substantial improvement over the $31.18$ base result. 
Second, the performance typically follows a convex trajectory, with the optimal balance point generally leaning towards the global context to effectively ground the isolated layout components. 
Quantitatively, most models reach their peak performance within the range of $\alpha \in [0.6, 0.8]$. 
Finally, a moderate synergistic fusion is essential, as either excessive local specificity or excessive global dominance leads to sub-optimal results. 
In B3-7B, performance steadily climbs from $62.09$ at $\alpha=0.1$ to its peak of $68.51$ at $\alpha=0.7$, before exhibiting a slight decline as the representation becomes overly dominated by the global vector at $\alpha=0.9$.

\begin{figure}[!t]
  \centering
  \includegraphics[width=0.8\linewidth]{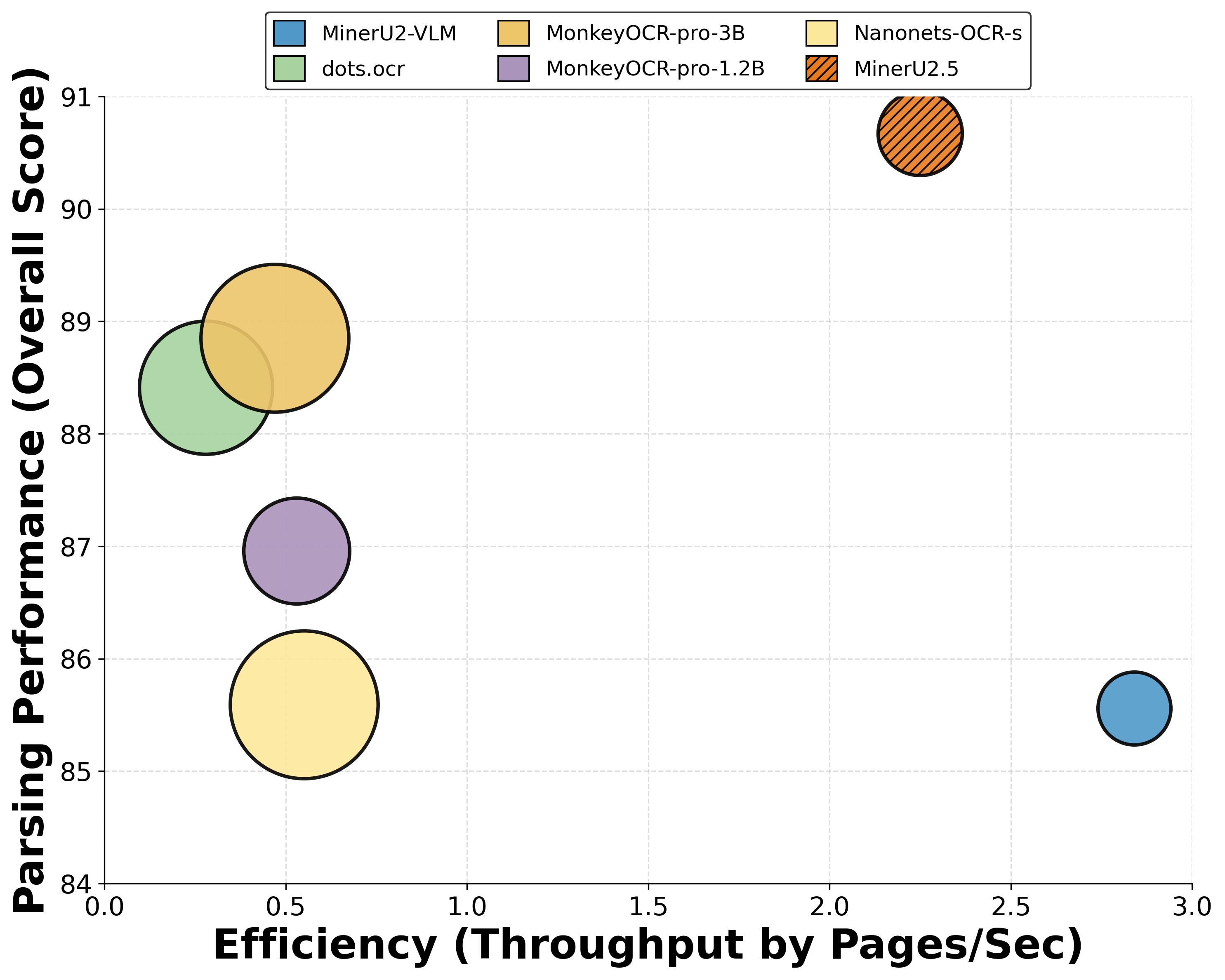}
  \vspace{-1mm}
  \caption{Comparison of MinerU2.5 against its counterparts. The y-axis represents the overall score of OmniDocBench \cite{ouyang2025omnidocbench}, the x-axis shows end-to-end throughput (Pages/Sec), and bubble size indicates the parameter size.}
   \label{fig:mineru25_comparison}
   \vspace{-2mm}
\end{figure}

\paragraph{Effect of document parsing model.} 
We evaluate the capability of various document parsing models and conclude that MinerU2.5 provides the optimal balance between parsing fidelity and inference efficiency, as shown in \autoref{fig:mineru25_comparison} (Details in Appendix \ref{app:hyperparameter_document_parsing_model}). 
First, MinerU2.5 demonstrates superior parsing accuracy across all semantic and structural dimensions compared to existing specialized VLMs. 
Quantitatively, it achieves a state-of-the-art Overall score of 90.67 on OmniDocBench \cite{ouyang2025omnidocbench}, significantly outperforming the best baseline MonkeyOCR-pro-3B (88.85) and maintaining the lowest error rates in both text (0.047) and reading order (0.044) recognition. 
Second, MinerU2.5 achieves industrial-grade throughput, ensuring the practicality of the \ourmethod\ pipeline for large-scale document corpora. 
For example, it delivers an end-to-end processing speed of 2.25 Pages/sec, which is $4\times$ faster than high-parameter models like Nanonets-OCR-s (0.55 Pages/sec). 
Consequently, we select MinerU2.5 as our unified parser as it sits at the optimal Pareto front, offering the highest parsing quality while sustaining a remarkable throughput.

\subsubsection{Efficiency Analysis}
\label{sec:efficiency_analysis}

\begin{table}[!t]
  \centering
  \caption{Efficiency analysis of \ourmethod on the best performing model GME-7B and its multi-vector counterpart (w/ original setting). \textbf{Performance} denotes the overall score of MMEB-visdoc, \textbf{Storage} refers to the number of vectors stored per document, and \textbf{Latency} represents the average encoding time per document. * denotes the multi-vector model is trained w/ aligned configuration. {\color{RedOrange}Orange} arrows denote better and {\color{BlueGreen}blue} ones denote worse.}
  \label{tab:efficiency_comparison}
  % \vspace{-0.8em}
  \renewcommand\tabcolsep{8pt}
  \renewcommand\arraystretch{1}
  \footnotesize 
  \begin{tabular}{l|ccc} 
    \Xhline{1.2pt}
    \rowcolor{CadetBlue!20} 
    \textbf{Model} & \textbf{Performance} & \textbf{Storage} & \textbf{Latency} \\ 
    \Xhline{1pt}
    \textbf{GME-7b} & 79.50 & 1.00 & 0.30 \\
    \rowcolor{gray!10}
    \textbf{+\ourmethod} & 80.61\red{1.11} & 5.90\blueup{4.9} & 0.81\blueup{0.51} \\
    \textbf{ColQwen\textsuperscript{*}} & 80.02\red{0.52} & 768.00\blueup{767} & 0.41\blueup{0.11} \\
    \Xhline{1.2pt}
  \end{tabular}
  \vspace{-6mm}
\end{table}

To ensure a fair comparison, we evaluate \ourmethod\ using GME-7B, the best-performing single-vector model, and ColQwen, a multi-vector baseline with an aligned architecture and training configuration. 
We conclude that our paradigm achieves superior retrieval performance while drastically slashing the storage overhead inherent in traditional multi-vector models. 
Specifically, as shown in \autoref{tab:efficiency_comparison}, \ourmethod\ yields an overall score of 80.61 on MMEB-visdoc, outperforming the multi-vector counterpart ColQwen (80.02) while requiring only 5.9 vectors per page (See per-dataset records in Appendix \ref{app:efficiency_analysis}), a massive storage reduction of over 99\% compared to ColQwen's 768 vectors. 
Furthermore, although \ourmethod\ introduces a marginal increase in encoding latency, it remains a highly practical solution for real-world deployment. 
While the per-document latency rises to 0.81s, it is still significantly lower than the average 7s delay of conventional OCR-based pipelines, and the offline parsing stage can be optimized through parallel processing to mitigate indexing overhead.

\subsubsection{Case Study}
\label{sec:case_study}

\begin{figure}[!h]
  \centering
  \includegraphics[width=0.9\linewidth]{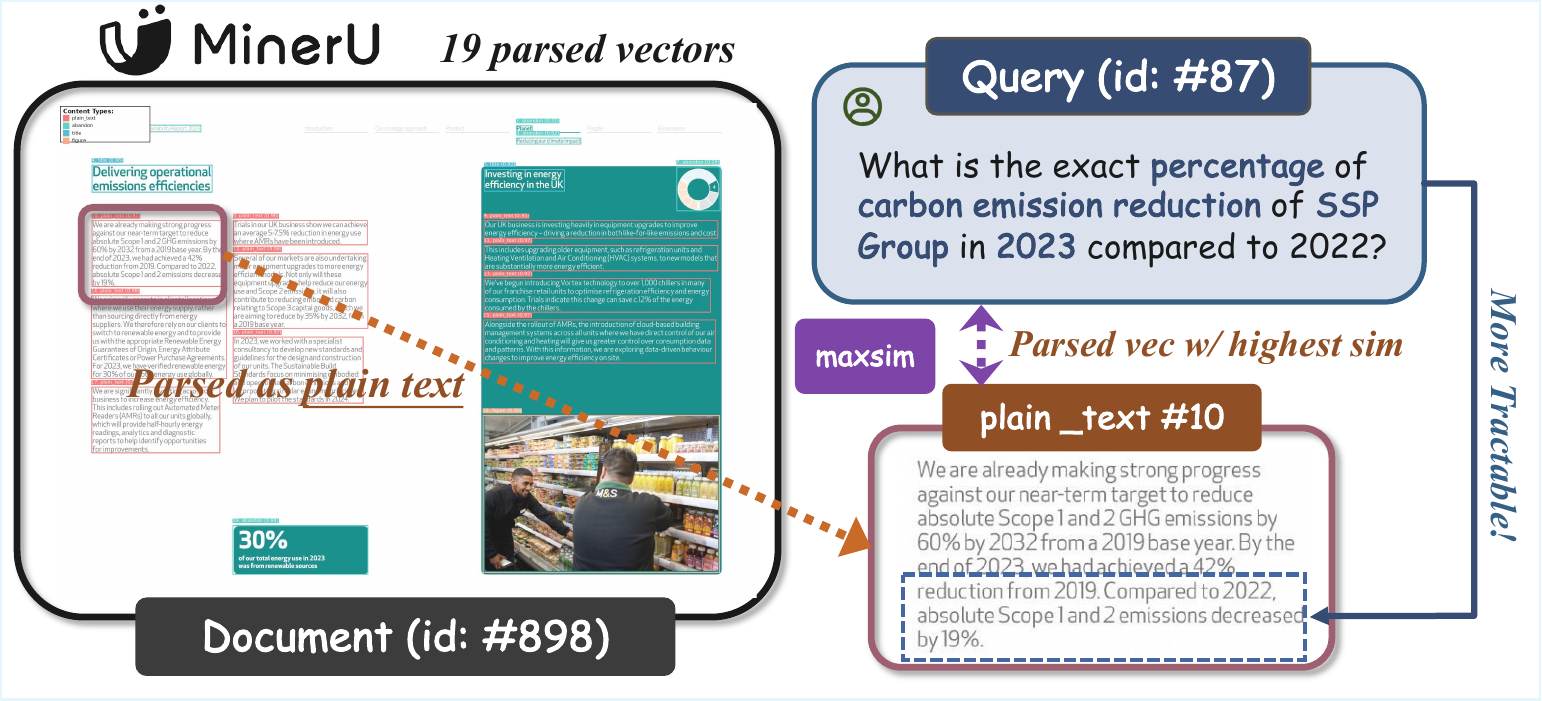}
  \vspace{-2mm}
  \caption{The illustration of a representative case.}
   \label{fig:colparse_case_main_text}
   \vspace{-4mm}
\end{figure}

We conduct a case study in Figure~\ref{fig:colparse_case_main_text} to illustrate the interpretability of \ourmethod beyond its superior retrieval performance. When handling a query requiring specific details, such as carbon emission reduction percentages, \ourmethod not only retrieves the correct document page but also pinpoints the exact parsed sub-vector (\textit{e.g.,} a specific text block) that yields the highest \texttt{MaxSim} score. This layout-informed granularity enables fine-grained information back-tracing, allowing the system to present the specific evidence region directly to the user. Such explainability is highly valuable in practical industrial scenarios, such as financial auditing or legal review, where the ability to verify the source of information is as critical as retrieval accuracy.

\section{Conclusion}
\label{sec:conclusion}
\vspace{-2mm}
In this paper, we introduced \ourmethod, a novel layout-informed paradigm designed to overcome the critical storage efficiency bottleneck in multi-vector VDR. 
Our framework uniquely generates a compact multi-vector representation by fusing layout-aware sub-image embeddings from a document parser with a holistic global vector. 
Extensive experiments demonstrated that \ourmethod achieves over 95\% storage compression while consistently improving retrieval performance across base models and datasets. 
Ultimately, \ourmethod charts a new course for the field, establishing that a deep understanding of document structure is critical for practical multimodal information systems.

\clearpage
\section*{Impact Statement}

\paragraph{Ethical Considerations.}
We believe that our proposed \ourmethod framework raises no new ethical concerns. Its motivation is to advance the efficiency and performance of VDR systems in a principled and resource-conscious manner. By leveraging existing document parsing technologies to create more compact and interpretable representations, our method promotes responsible AI development, adhering to established ethical guidelines in information retrieval research without relying on sensitive or proprietary data.

\paragraph{Societal Implications.}
\ourmethod introduces a new paradigm for multimodal information retrieval by shifting from storage-intensive, grid-based representations to a lightweight, layout-informed approach. It fundamentally resolves the critical conflict between fine-grained retrieval accuracy and the practical storage costs of large-scale deployment. By reducing storage requirements by over 95\% while simultaneously enhancing performance, \ourmethod significantly lowers the barrier for deploying state-of-the-art visual document understanding systems. This has the potential to democratize access to advanced information retrieval in diverse domains, including academic research, enterprise knowledge management, and digital archives. Furthermore, its inherent interpretability, which links retrieval results to specific document components, fosters greater transparency and user trust in AI-powered information systems.

% In the unusual situation where you want a paper to appear in the
% references without citing it in the main text, use \nocite
% \nocite{langley00}

\bibliography{colparse}
\bibliographystyle{icml2026}

%%%%%%%%%%%%%%%%%%%%%%%%%%%%%%%%%%%%%%%%%%%%%%%%%%%%%%%%%%%%%%%%%%%%%%%%%%%%%%%
%%%%%%%%%%%%%%%%%%%%%%%%%%%%%%%%%%%%%%%%%%%%%%%%%%%%%%%%%%%%%%%%%%%%%%%%%%%%%%%
% APPENDIX
%%%%%%%%%%%%%%%%%%%%%%%%%%%%%%%%%%%%%%%%%%%%%%%%%%%%%%%%%%%%%%%%%%%%%%%%%%%%%%%
%%%%%%%%%%%%%%%%%%%%%%%%%%%%%%%%%%%%%%%%%%%%%%%%%%%%%%%%%%%%%%%%%%%%%%%%%%%%%%%
\newpage
\clearpage
\appendix
% \onecolumn

\section{Algorithm Workflow}
\label{app:algo_workflow}

We formalize the complete workflow of our proposed \ourmethod framework in two distinct algorithms. 
\textbf{Algorithm~\ref{alg:offline_indexing}} details the offline indexing process, where \ourmethod{} generates a highly compact set of document embeddings through its sequential three-stage process.
Subsequently, \textbf{Algorithm~\ref{alg:online_retrieval}} illustrates the online retrieval phase, where the final relevance score is efficiently computed via a \texttt{MaxSim} operation using this compressed set of embeddings.

\begin{algorithm}[!ht]
\caption{The Offline Indexing Process of \ourmethod}\label{alg:offline_indexing}
\SetKwComment{Comment}{\# }{}
\SetCommentSty{textrm}

\Input{
    A document image $d \in \mathbb{R}^{H \times W \times 3}$; \\
    A document parser model $\Psi_{\text{parse}}$; \\
    A single-vector encoder $\Phi_{\text{enc}}: \mathbb{R}^{H' \times W' \times 3} \to \mathbb{R}^D$
}
\Output{
    A compact multi-vector representation $\mathbf{D}_{\text{\ourmethod}} \subset \mathbb{R}^{k \times D}$
}
\BlankLine

\tcc{\textcolor{blue}{Stage 1: Layout-Informed Document Parsing}}
$[\{b_j, c_j\}]_{j=1}^k \leftarrow \Psi_{\text{parse}}(d)$ \tcp*{Get $k$ bboxes and content types}
$\mathcal{S}_d \leftarrow \emptyset$\;
\For{$j \leftarrow 1$ \KwTo $k$}{
    $s_j \leftarrow \text{Crop}(d, b_j)$ \tcp*{Crop doc image $d$ using bbox $b_j$}
    $\mathcal{S}_d \leftarrow \mathcal{S}_d \cup \{s_j\}$\;
}
\BlankLine

\tcc{\textcolor{blue}{Stage 2: Dual-Stream Encoding}}
$\mathbf{D}_{\text{local}} \leftarrow \emptyset$\;
\For{each sub-image $s_j \in \mathcal{S}_d$}{
    $\mathbf{v}_{\text{local}}^{(j)} \leftarrow \Phi_{\text{enc}}(s_j)$ \tcp*{Encode local region}
    $\mathbf{D}_{\text{local}} \leftarrow \mathbf{D}_{\text{local}} \cup \{\mathbf{v}_{\text{local}}^{(j)}\}$\;
}
$\mathbf{v}_{\text{global}} \leftarrow \Phi_{\text{enc}}(d)$ \tcp*{Encode entire page for global context}
\BlankLine

\tcc{\textcolor{blue}{Stage 3: Global-Local Fusion}}
$\mathbf{D}_{\text{\ourmethod}} \leftarrow \emptyset$\;
\For{each local vector $\mathbf{v}_{\text{local}}^{(j)} \in \mathbf{D}_{\text{local}}$}{
    $\mathbf{d}_{\text{fused}}^{(j)} \leftarrow \mathbf{v}_{\text{local}}^{(j)} + \mathbf{v}_{\text{global}}$ \tcp*{Fuse by element-wise addition}
    $\mathbf{D}_{\text{\ourmethod}} \leftarrow \mathbf{D}_{\text{\ourmethod}} \cup \{\mathbf{d}_{\text{fused}}^{(j)}\}$\;
}
\BlankLine

\KwRet $\mathbf{D}_{\text{\ourmethod}}$\;
\end{algorithm}

\begin{algorithm}[!ht]
\caption{The Online Retrieval Process with \ourmethod}\label{alg:online_retrieval}
\SetKwComment{Comment}{\# }{}
\SetCommentSty{textrm}

\Input{
    A textual query $q$; \\
    A pre-computed compact representation $\mathbf{D}_{\text{\ourmethod}} = \{\mathbf{d}_{\text{fused}}^{(j)}\}_{j=1}^k$; \\
    An encoder $\Phi_{\text{enc}}$
}
\Output{
    The relevance score $s_{\text{\ourmethod}}(q, d)$
}
\BlankLine

\tcc{\textcolor{blue}{Step 1: Encode Query}}
$\mathbf{Q} \leftarrow \Phi_{\text{enc}}(q)$ \tcp*{Encode $q$ into $N_q$ token vectors $\{\mathbf{q}_i\}$}
\BlankLine

\tcc{\textcolor{blue}{Step 2: Late-Interaction Scoring (\texttt{MaxSim})}}
$score \leftarrow 0$\;
\For{each query vector $\mathbf{q}_i \in \mathbf{Q}$}{
    $max\_sim \leftarrow -\infty$\;
    \For{each fused document vector $\mathbf{d}_{\text{fused}}^{(j)} \in \mathbf{D}_{\text{\ourmethod}}$}{
        $sim \leftarrow {\mathbf{q}_i}^\top \mathbf{d}_{\text{fused}}^{(j)}$ \tcp*{Assuming L2-normalized vectors}
        $max\_sim \leftarrow \max(max\_sim, sim)$\;
    }
    $score \leftarrow score + max\_sim$ \tcp*{Aggregate max similarity}
}
\BlankLine

\KwRet $score$\;
\end{algorithm}

\section{More Theoretical Analysis}
\label{app:more_theoretical_analysis}

This section provides a detailed theoretical exposition of the concepts introduced in Section~\ref{sec:method}, grounding the \ourmethod~framework in fundamental principles of information theory.

\subsection{Information-Theoretic Preliminaries}

We begin by defining the core concepts used in our analysis.

\begin{definition}[Mutual Information]
The mutual information $I(X; Y)$ between two random variables $X$ and $Y$ measures their mutual dependence. It is defined as:
\begin{equation}
    I(X; Y) = \sum_{x \in \mathcal{X}} \sum_{y \in \mathcal{Y}} p(x, y) \log \frac{p(x, y)}{p(x)p(y)}.
\end{equation}
where $p(x, y)$ is the joint probability distribution, and $p(x)$ and $p(y)$ are the marginal distributions. $I(X; Y) = 0$ if and only if $X$ and $Y$ are independent.
\end{definition}

\begin{definition}[Conditional Mutual Information]
The conditional mutual information $I(X; Y | Z)$ measures the mutual information between $X$ and $Y$ given that a third variable $Z$ is known:
\begin{equation}
    I(X; Y | Z) = \mathbb{E}_{z \sim p(z)} [I(X; Y | Z=z)].
\end{equation}
\end{definition}

\begin{theorem}[Chain Rule for Mutual Information]
For a set of random variables $\{X_1, \dots, X_n\}$ and another variable $Y$, the chain rule states:
\begin{equation}
    I(X_1, \dots, X_n; Y) = \sum_{i=1}^{n} I(X_i; Y | X_1, \dots, X_{i-1}).
\end{equation}
This rule is fundamental for decomposing the information content of a complex system.
\end{theorem}

\begin{theorem}[Data Processing Inequality (DPI)]
\label{thm:dpi}
For any Markov chain of random variables $X \to Y \to Z$, where $Z$ is conditionally independent of $X$ given $Y$, the following inequality holds:
\begin{equation}
    I(X; Z) \le I(X; Y) \text{ and } I(X; Z) \le I(Y; Z).
\end{equation}
This theorem formalizes the notion that post-processing (the step from $Y$ to $Z$) cannot increase information about the original source $X$.
\end{theorem}

\subsection{The Information Bottleneck (IB) Principle in VDR}
As stated in Section~\ref{sec:theorectical_foundation}, the VDR compression task can be framed as an IB problem \cite{tishby2000information}. The objective is to find a compressed representation $Z$ of a document $D$ that maximizes information about a relevance variable $R$, while minimizing information about the source $D$ itself.

\paragraph{Proof of Intractability.} The IB Lagrangian (Eq.~\ref{eq:ib_lagrangian} in the main text) requires computing an expectation over the distribution of all possible queries, $P(Q)$.
\begin{equation}
    \mathcal{L}(Z) = I(Z; D) - \beta \int_{q \in \mathcal{Q}} P(q) I(Z; R(D, q)) dq.
\end{equation}
Since $P(Q)$ is unknown and potentially infinite at the time of document indexing, this objective cannot be directly optimized. Therefore, practical methods must rely on principled approximations or surrogates for this ideal objective. \ourmethod~provides such a surrogate.

\subsection{Justification for Structural Disentanglement}
\ourmethod's parsing stage, $\Psi_{\text{parse}}(D) = \{S_1, \dots, S_k\}$, is justified by the Semantic Concentration Axiom. We now provide a more formal justification.

\begin{axiom}[Semantic Concentration]
\label{axiom:concentration}
For a given query $Q=q$, there exists a primary semantic region $S_{j^*} \in \{S_j\}$ that contains almost all the information required to determine relevance. The remaining regions $S_{\neg j^*} = \{S_j\}_{j \neq j^*}$ provide negligible additional information.
\begin{equation}
    I(S_{\neg j^*}; R | S_{j^*}, Q=q) \approx 0.
\end{equation}
\end{axiom}

\begin{proof}[Justification]
This axiom is an empirical assumption about the nature of user queries and documents. For a query ``What were the revenues in Q3 2023?'', the answer is almost certainly contained entirely within a single financial table. Information in other regions (e.g., the abstract, a methodology figure) is conditionally irrelevant once the correct table is identified.
\end{proof}

\begin{corollary}[Information Equivalence of Decomposed Representation]
Under the Semantic Concentration Axiom, the mutual information between the entire document and the relevance variable is approximately equal to the maximum information contained in any single semantic region.
\begin{equation}
    I(D; R) \approx \max_{j \in \{1, \dots, k\}} I(S_j; R).
\end{equation}
\end{corollary}
\begin{proof}
From the chain rule, $I(D; R) = I(S_1, \dots, S_k; R)$. For a specific query $q$, let $j^*$ be the index of the primary region. We have:
\begin{equation}
    I(D; R | Q=q) = I(S_{j^*}; R | Q=q) + I(S_{\neg j^*}; R | S_{j^*}, Q=q).
\end{equation}
Applying Axiom~\ref{axiom:concentration}, the second term vanishes: $I(D; R | Q=q) \approx I(S_{j^*}; R | Q=q)$. Taking the expectation over $P(Q)$, and using the property that $\mathbb{E}[\max(X_i)] \ge \max(\mathbb{E}[X_i])$, we arrive at the approximation that the total information is well-represented by the information in the best possible channel, justifying the multi-vector approach.
\end{proof}

\subsection{Justification for Synergistic Fusion}
The fusion stage combines local vectors $\{V_j = \Phi_{\text{enc}}(S_j)\}$ with a global vector $V_{\text{global}} = \Phi_{\text{enc}}(D)$ to produce the final representation $\{Z_j = V_j + V_{\text{global}}\}$.

\begin{definition}[Contextual Information Gain]
\label{def:gain}
The contextual information gain for region $j$ is the additional information about relevance $R$ provided by the global context $V_{\text{global}}$, given that the local information $V_j$ is already known.
\begin{equation}
    G_j^{\text{context}} \triangleq I(V_{\text{global}}; R | V_j).
\end{equation}
\end{definition}

\begin{theorem}[Information in the Fused Representation]
The information contained in the fused vector $Z_j = V_j + V_{\text{global}}$ is upper-bounded by the joint information of its components.
\begin{equation}
    I(Z_j; R) \le I(V_j, V_{\text{global}}; R).
\end{equation}
\end{theorem}
\begin{proof}
The fused vector $Z_j$ is a deterministic function of $V_j$ and $V_{\text{global}}$. This forms the Markov chain $(V_j, V_{\text{global}}) \to Z_j \to R$. Applying the Data Processing Inequality (Theorem~\ref{thm:dpi}) to this chain directly yields the result.
\end{proof}

\begin{corollary}[Condition for Information Improvement]
The fusion step is beneficial (i.e., $Z_j$ is more informative than $V_j$ alone) if and only if the fusion function successfully captures a non-zero portion of the contextual information gain.
\begin{equation}
    \Delta I_j = I(Z_j; R) - I(V_j; R) > 0 \iff I(Z_j; R | V_j) > 0.
\end{equation}
\end{corollary}
\begin{proof}
From the chain rule, $I(Z_j, V_j; R) = I(V_j; R) + I(Z_j; R | V_j)$. Since $Z_j$ is a function of $V_j$ and $V_{\text{global}}$, knowing $V_j$ does not make $Z_j$ fully determined. The term $I(Z_j; R | V_j)$ represents the information that the \textit{variation} in $Z_j$ (caused by $V_{\text{global}}$) provides about $R$, even when $V_j$ is fixed. A positive net improvement $\Delta I_j > 0$ directly requires this conditional term to be positive, which in turn means the fusion must have encoded some of the contextual gain $G_j^{\text{context}}$. The vector addition $V_j + V_{\text{global}}$ is a simple, effective function for this purpose, as it non-linearly interacts with the query vector during the dot product scoring: $\mathbf{q}^\top(\mathbf{v}_j + \mathbf{v}_{\text{global}})$, allowing the model to utilize both local and global signals.
\end{proof}

\clearpage
\section{More Experimental Analysis}
\label{app:more_experimental_analysis}

\subsection{Benchmark Details}
\label{app:benchmark_details}

To ensure a comprehensive and robust evaluation of our framework, we anchor our experiments on five mainstream benchmark suites for VDR, all of which are integrated within the visdoc section of the MMEB \cite{meng2025vlm2vec}. The following benchmarks collectively cover a diverse range of document types, query complexities, and retrieval scenarios, providing a multifaceted view of model performance.

\begin{itemize}[leftmargin=1.5em,itemsep=-0.1em,topsep=0.2em]
    \item[$\blacktriangleright$] \textbf{ViDoRe-V1 \cite{faysse2024colpali}\footnote{\url{https://huggingface.co/collections/vidore/vidore-benchmark}}:} As a foundational benchmark for page-level VDR, ViDoRe-V1 was one of the first to systematically evaluate systems on visually-rich documents. It combines repurposed academic VQA datasets with practical, topic-specific tasks, highlighting the inherent shortcomings of traditional text-only retrieval systems on documents containing complex layouts, tables, and figures.

    \item[$\blacktriangleright$] \textbf{ViDoRe-V2 \cite{mace2025vidorev2}\footnote{\url{https://huggingface.co/collections/vidore/vidore-benchmark-v2}}:} As a successor to ViDoRe-V1, ViDoRe-V2 aims to raise the bar by introducing more challenging and realistic retrieval scenarios to address the performance saturation observed on the original. Its core contributions include the use of long-form, cross-document, and multilingual queries generated via a hybrid synthetic and human-in-the-loop process, which reduces extractive bias and more accurately reflects real-world user interactions.

    \item[$\blacktriangleright$] \textbf{VisRAG \cite{yu2024visrag}\footnote{\url{https://huggingface.co/collections/openbmb/visrag}}:} The VisRAG benchmark is constructed to specifically evaluate vision-based RAG pipelines by aggregating and refining multiple existing VQA datasets. Its primary contribution is the unification of a wide spectrum of document types—including scientific figures, charts, infographics, and presentation slides—under a single evaluation framework, coupled with a crucial filtering process to remove context-dependent questions and ensure its suitability for open-retrieval tasks.

    \item[$\blacktriangleright$] \textbf{ViDoSeek \cite{wang2025vidorag}\footnote{\url{https://huggingface.co/datasets/Qiuchen-Wang/ViDoSeek}}:} ViDoSeek is a novel benchmark designed to evaluate end-to-end RAG systems on visually-rich documents that require complex reasoning. Its main contribution lies in providing a large document corpus where each query corresponds to a unique answer, which allows for a more realistic and rigorous evaluation of both the retrieval and subsequent reasoning stages in a large-scale setting.

    \item[$\blacktriangleright$] \textbf{MMLongBench \cite{ma2024mmlongbench}\footnote{\url{https://huggingface.co/datasets/yubo2333/MMLongBench-Doc}}:} MMLongBench is specifically designed to assess the long-context, multi-modal understanding capabilities of LVLMs. It stands out by using lengthy documents (averaging 47.5 pages) and featuring a significant portion of cross-page questions that require multi-hop reasoning, as well as unanswerable questions to probe for model hallucination, thus rigorously testing a model's ability to locate and synthesize information from extensive contexts.
\end{itemize}

\subsection{Model Details}
\label{app:model_details}

We select ten representative single-vector multimodal retrieval models from recent literature to serve as the base models for our experiments. These models, built upon various architectures and pre-training paradigms, provide a comprehensive testbed for evaluating the versatility and effectiveness of our proposed framework.

\begin{itemize}[leftmargin=1.5em,itemsep=-0.1em,topsep=0.2em]
    \item[$\blacktriangleright$] \textbf{VLM2Vec-V1-2B/7B \cite{jiang2024vlm2vec}\footnote{\url{https://huggingface.co/TIGER-Lab/VLM2Vec-Full}}:} As a pioneering work in universal multimodal embeddings, VLM2Vec introduces a contrastive training framework to adapt any VLM for a wide range of tasks. Its core contribution is reformulating diverse multimodal tasks (\textit{e.g.,} classification, VQA, retrieval) into a unified instruction-following ranking problem, enabling the model to learn general-purpose embeddings for both images and text.

    \item[$\blacktriangleright$] \textbf{VLM2Vec-V2-2B \cite{meng2025vlm2vec}\footnote{\url{https://huggingface.co/VLM2Vec/VLM2Vec-V2.0}}:} This model extends its predecessor by broadening the scope of multimodal embeddings to include videos and visual documents, in addition to images and text. Its primary contribution is the introduction of a more comprehensive benchmark and a unified training strategy that allows a single model to effectively learn representations across static, temporal, and structured visual data formats.

    \item[$\blacktriangleright$] \textbf{LamRA-Ret-7B \cite{liu2025lamra}\footnote{\url{https://huggingface.co/code-kunkun/LamRA-Ret}}:} LamRA explores repurposing generative Large Multimodal Models for retrieval tasks, unifying diverse retrieval scenarios under a single instruction-following framework. Its key innovation is a two-stage training strategy that first pre-trains the model on language-only tasks before multimodal instruction tuning, progressively adapting the generative model for retrieval.

    \item[$\blacktriangleright$] \textbf{GME-2B/7B \cite{zhang2024gme}\footnote{\url{https://huggingface.co/collections/Alibaba-NLP/gme-models}}:} The General Multimodal Embedder (GME) framework focuses on improving universal multimodal retrieval by leveraging a more diverse mix of training data, including single-modal, cross-modal, and fused-modal examples. Its core contribution is a novel data synthesis pipeline for creating large-scale, high-quality fused-modal data, which significantly enhances the model's ability to handle complex queries and retrieve visual documents.

    \item[$\blacktriangleright$] \textbf{UniME-V2-2B/7B \cite{gu2025unimev2}\footnote{\url{https://huggingface.co/collections/TianchengGu/unime-v2}}:} UniME-V2 enhances representation learning by leveraging an MLLM as a "judge" to generate soft semantic matching scores for query-candidate pairs. This MLLM-as-a-Judge mechanism facilitates more effective hard negative mining and allows the embedding model to learn finer-grained semantic distinctions, significantly improving its discriminative capacity.

    \item[$\blacktriangleright$] \textbf{B3-2B/7B \cite{thirukovalluru2025breaking}\footnote{\url{https://huggingface.co/raghavlite}}:} Breaking the Batch Barrier (B3) introduces a novel batch construction strategy for contrastive learning that curates high-quality batches rich in hard negatives. Instead of random sampling, it uses a teacher model and graph-based community detection to group mutually challenging examples together, thereby improving training efficiency and achieving state-of-the-art performance even with significantly smaller batch sizes.
\end{itemize}

\subsection{MinerU2.5 Details}
\label{app:mineru25_details}

% MinerU2.5 is a 1.2B-parameter vision-language model specifically designed for document parsing tasks, achieving state-of-the-art recognition accuracy while maintaining exceptional computational efficiency. 
To resolve the trade-off between the immense computational overhead ($O(N^2)$ complexity) and information loss associated with directly processing high-resolution document images, MinerU2.5 innovatively employs a decoupled, \textit{coarse-to-fine} two-stage strategy:

\begin{enumerate}[leftmargin=1.5em,itemsep=-0.1em,topsep=0.2em]
    \item \textbf{Stage I: Global Layout Analysis.} In this stage, the model first resizes the input document image to a medium-resolution thumbnail (\textit{e.g.,} 1036$\times$1036 pixels). It then performs a fast, global layout analysis on this thumbnail to identify all structural elements (such as paragraphs, tables, formulas, and figures) and their positions at a low computational cost.

    \item \textbf{Stage II: Local Content Recognition.} Guided by the layout information detected in the first stage, the model precisely crops the respective semantic regions from the \textbf{original high-resolution} image. Subsequently, it performs parallel, fine-grained content recognition (\textit{e.g.,} text OCR, table structuring, formula transcription) on these native-resolution cropped patches. This preserves high recognition accuracy while avoiding redundant computations on the entire high-resolution image.
\end{enumerate}

% This decoupled design not only architecturally bypasses the efficiency bottleneck but also provides an accurate and reliable basis for the semantic region segmentation required by \ourmethod.

\textbf{Algorithm~\ref{alg:layout_parsing}} details the layout-informed image splitting process used in \ourmethod. 
% This procedure intelligently decomposes a monolithic document page into semantically coherent sub-images. 
% It prioritizes layout detection via MinerU's DocLayoutYOLO but includes a robust grid-based fallback mechanism and filters regions based on their relative importance (area ratio).

\begin{algorithm}[!ht]
\caption{Layout-Informed Image Splitting for \ourmethod}\label{alg:layout_parsing}
\SetKwComment{Comment}{\# }{}
\SetCommentSty{textrm}

\Input{
    A document image $d \in \mathbb{R}^{H \times W \times 3}$; \\
    A layout detector model $\Psi_{\text{split}}$ (e.g., DocLayoutYOLO); \\
    Minimum area ratio threshold $\tau$ (default 0.01); \\
    Maximum sub-images count $N_{\text{max}}$ (default 20); \\
    Grid fallback parameters: $R_{\text{grid}}, C_{\text{grid}}$
}
\Output{
    A list of cropped sub-images $\mathcal{S}_d$; \\
    A list of content type labels $\mathcal{C}_d$ (optional)
}
\BlankLine

\tcc{\textcolor{blue}{Step 1: Semantic Layout Detection}}
$\text{TotalArea} \leftarrow H \times W$\;
$\mathcal{B} \leftarrow \emptyset, \mathcal{C}_d \leftarrow \emptyset$\;

\If{$\Psi_{\text{split}}$ is available}{
    $\mathcal{R} \leftarrow \Psi_{\text{split}}.\text{predict}(d)$ \tcp*{Returns list of \{bbox, category, score\}}
    \If{$\mathcal{R}$ is not empty}{
        \For{each region $r \in \mathcal{R}$}{
            $b \leftarrow (x_1, y_1, x_2, y_2)$ from $r.\text{poly}$\;
            $c \leftarrow \text{MapCategoryID}(r.\text{category\_id})$\;
            $\mathcal{B} \leftarrow \mathcal{B} \cup \{(b, c, \text{centerY}(b), \text{centerX}(b))\}$\;
        }
    }
}

\BlankLine
\tcc{\textcolor{blue}{Step 2: Fallback \& Sorting Mechanism}}
\uIf{$\mathcal{B}$ is empty}{
    $\mathcal{B} \leftarrow \text{GridBasedSplit}(H, W, R_{\text{grid}}, C_{\text{grid}})$ \tcp*{Fallback to grid}
}
\Else{
    $\mathcal{B} \leftarrow \text{SortByReadingOrder}(\mathcal{B})$ \tcp*{Sort by vertical bands, then horizontal}
}

\BlankLine
\tcc{\textcolor{blue}{Step 3: Filtering, Cropping and Output}}
$\mathcal{S}_d \leftarrow \emptyset, \text{count} \leftarrow 0$\;
\For{each $(b, c) \in \mathcal{B}$}{
    \If{$\text{count} \ge N_{\text{max}}$}{\textbf{break}\;}
    $\text{Area} \leftarrow \text{width}(b) \times \text{height}(b)$\;
    \If{$\text{Area} / \text{TotalArea} \ge \tau$}{
        $s \leftarrow \text{Crop}(d, b)$ \tcp*{Extract region from original image}
        $\mathcal{S}_d \leftarrow \mathcal{S}_d \cup \{s\}$\;
        $\mathcal{C}_d \leftarrow \mathcal{C}_d \cup \{c\}$\;
        $\text{count} \leftarrow \text{count} + 1$\;
    }
}

\BlankLine
\KwRet $\mathcal{S}_d, \mathcal{C}_d$\;
\end{algorithm}

\subsection{Main Results}
\label{app:main_results}

Refer to \autoref{tab:full_results_mmlongbench_vidorev1} and \autoref{tab:full_results_vidorev2_vidoseek_visrag} for all results of \ourmethod and baselines across five benchmarks.

\onecolumn

% We recommend wrapping this in a \onecolumn environment for two-column papers.
\begin{center}
\small % Use smaller font for wide tables
\setlength\tabcolsep{2.5pt} % Adjust column separation
\renewcommand\arraystretch{1.2} % Adjust row height
\begin{longtable}{p{2.6cm}|cc|c|*{10}{c}|c}
\caption{Performance comparison on MMLongBench and ViDoRe-V1 benchmarks. For each model block, we bold the best-performing optimization method in each column (except for the base result). The average scores for optimizations are shown with relative gains ($\color{RedOrange}\uparrow\color{black}$/$\color{BlueGreen}\downarrow\color{black}$) compared to the base model.}
\label{tab:full_results_mmlongbench_vidorev1} \\
\Xhline{1.5pt}
\rowcolor{gray!20}
\multirow{2}{*}{\textbf{Method}} & \multicolumn{3}{c|}{\textbf{MMLongBench}} & \multicolumn{11}{c}{\textbf{ViDoRe-V1}} \\
\cmidrule(lr){2-4} \cmidrule(lr){5-15}
\rowcolor{gray!20}
& \textbf{Doc} & \textbf{Page} & \textbf{Avg.} & \textbf{Arxiv} & \textbf{DocV} & \textbf{InfoV} & \textbf{Shift} & \textbf{TabF} & \textbf{TatD} & \textbf{S-AI} & \textbf{S-En} & \textbf{S-HC} & \textbf{S-Gov} & \textbf{Avg.} \\
\Xhline{1.5pt}
\endfirsthead
\Xhline{1.5pt}
\rowcolor{gray!20}
\multicolumn{15}{c}{\tablename\ \thetable\ -- \textit{Continued from previous page}} \\
\Xhline{1.5pt}
\rowcolor{gray!20}
\multirow{2}{*}{\textbf{Method}} & \multicolumn{3}{c|}{\textbf{MMLongBench}} & \multicolumn{11}{c}{\textbf{ViDoRe-V1}} \\
\cmidrule(lr){2-4} \cmidrule(lr){5-15}
\rowcolor{gray!20}
& \textbf{Doc} & \textbf{Page} & \textbf{Avg.} & \textbf{Arxiv} & \textbf{DocV} & \textbf{InfoV} & \textbf{Shift} & \textbf{TabF} & \textbf{TatD} & \textbf{S-AI} & \textbf{S-En} & \textbf{S-HC} & \textbf{S-Gov} & \textbf{Avg.} \\
\Xhline{1.5pt}
\endhead
\Xhline{1.5pt}
\multicolumn{15}{r}{\textit{Continued on next page}} \\
\endfoot
\Xhline{1.5pt}
\endlastfoot
%
% Block 1: VLM2Vec-V1-2B
VLM2Vec-V1-2B & 25.62 & 26.23 & 25.93 & 17.80 & 13.98 & 39.41 & 9.18 & 36.32 & 10.56 & 16.39 & 15.96 & 23.56 & 24.11 & 20.73 \\
\hline
\rowcolor{gray!10}
\texttt{s2m-add} & 21.54 & 15.08 & 18.31\blue{7.62} & 35.07 & 15.61 & 52.15 & 6.62 & 36.51 & 10.39 & 26.23 & 31.22 & 30.29 & 33.61 & 27.77\red{7.04} \\
\texttt{s2m-mul} & 22.07 & 15.10 & 18.59\blue{7.34} & 34.91 & 16.12 & 52.61 & 6.57 & 36.65 & 10.34 & 23.90 & 30.89 & 28.88 & 33.60 & 27.45\red{6.72} \\
\rowcolor{gray!10}
\texttt{cl-t-c} & 16.57 & 10.59 & 13.58\blue{12.35} & 13.97 & 3.97 & 21.62 & 8.73 & 23.20 & 11.56 & 12.94 & 28.27 & 19.16 & 26.72 & 17.01\blue{3.72} \\
\texttt{cl-t-m} & 14.35 & 8.67 & 11.51\blue{14.42} & 16.47 & 2.98 & 26.97 & 16.27 & 18.48 & 8.59 & 14.40 & 23.03 & 18.71 & 13.10 & 15.90\blue{4.83} \\
\rowcolor{gray!10}
\texttt{cl-s-c} & 18.29 & 11.54 & 14.92\blue{11.01} & 18.83 & 4.87 & 22.80 & 13.49 & 24.64 & 12.63 & 15.45 & 24.88 & 20.78 & 20.85 & 17.92\blue{2.81} \\
\texttt{cl-s-m} & 15.45 & 9.06 & 12.26\blue{13.67} & 15.73 & 2.54 & 27.46 & 15.68 & 18.99 & 9.76 & 13.52 & 24.73 & 22.59 & 20.72 & 17.17\blue{3.56} \\
\rowcolor{gray!10}
\texttt{c-sem} & 18.76 & 13.88 & 16.32\blue{9.61} & 29.33 & 5.87 & 35.70 & 22.38 & 35.38 & 13.95 & 23.98 & 37.88 & 33.91 & 30.44 & 26.88\red{6.15} \\
\texttt{multi-img} & 23.61 & 15.85 & 19.73\blue{6.20} & 37.78 & 13.96 & 54.20 & 11.35 & 40.50 & 9.10 & 20.36 & 32.15 & 36.75 & 32.04 & \textbf{28.82}\red{8.09} \\
\hline
\rowcolor{gray!10}
\ourmethod & \textbf{34.31} & \textbf{29.83} & \textbf{32.07}\red{6.14} & \textbf{47.66} & \textbf{28.12} & \textbf{69.23} & \textbf{47.11} & \textbf{57.05} & \textbf{20.43} & \textbf{62.24} & \textbf{63.77} & \textbf{65.51} & \textbf{62.54} & \textbf{52.37}\red{31.64} \\
\Xhline{1.5pt}
% Block 2: VLM2Vec-V1-7B
VLM2Vec-V1-7B & 23.85 & 37.63 & 30.74 & 28.07 & 17.93 & 44.47 & 2.06 & 16.78 & 5.86 & 17.93 & 25.04 & 28.90 & 14.59 & 20.16 \\
\hline
\rowcolor{gray!10}
\texttt{s2m-add} & 34.57 & 28.11 & 31.34\red{0.60} & 50.30 & 25.66 & 66.73 & 38.21 & 63.75 & 23.49 & 70.27 & 61.87 & 70.68 & 66.38 & 53.73\red{33.57} \\
\texttt{s2m-mul} & 35.29 & 28.49 & 31.89\red{1.15} & 50.46 & 26.34 & 67.28 & 36.45 & 64.13 & 23.84 & 69.43 & 61.94 & 68.35 & 66.99 & 53.52\red{33.36} \\
\rowcolor{gray!10}
\texttt{cl-t-c} & 18.75 & 14.31 & 16.53\blue{14.21} & 17.10 & 7.71 & 26.41 & 21.58 & 29.62 & 14.61 & 18.52 & 27.98 & 26.09 & 26.65 & 21.63\red{1.47} \\
\texttt{cl-t-m} & 25.34 & 20.46 & 22.90\blue{7.84} & 28.06 & 7.00 & 47.43 & 29.48 & 45.92 & 19.42 & 27.66 & 47.56 & 44.96 & 52.94 & 35.04\red{14.88} \\
\rowcolor{gray!10}
\texttt{cl-s-c} & 22.25 & 14.48 & 18.37\blue{12.37} & 21.80 & 8.11 & 30.92 & 28.36 & 25.91 & 19.61 & 27.40 & 40.30 & 36.25 & 32.58 & 27.12\red{6.96} \\
\texttt{cl-s-m} & 26.31 & 20.38 & 23.35\blue{7.39} & 28.58 & 8.85 & 49.12 & 32.67 & 46.07 & 20.85 & 31.96 & 50.86 & 50.85 & 49.99 & 36.98\red{16.82} \\
\rowcolor{gray!10}
\texttt{c-sem} & 31.36 & 26.16 & 28.76\blue{1.98} & 45.77 & 15.38 & 59.20 & 37.46 & 57.05 & 30.17 & 47.00 & 60.89 & 64.22 & 66.76 & 48.39\red{28.23} \\
\texttt{multi-img} & 33.77 & 25.60 & 29.69\blue{1.05} & 49.40 & 19.55 & 62.09 & 28.19 & 66.34 & 17.19 & 41.89 & 51.44 & 60.84 & 48.89 & 44.58\red{24.42} \\
\hline
\rowcolor{gray!10}
\ourmethod & \textbf{43.34} & \textbf{40.58} & \textbf{41.96}\red{11.22} & \textbf{60.47} & \textbf{34.42} & \textbf{70.39} & \textbf{53.67} & \textbf{77.12} & \textbf{31.33} & \textbf{74.81} & \textbf{69.64} & \textbf{80.79} & \textbf{75.89} & \textbf{62.85}\red{42.69} \\
\Xhline{1.5pt}
% Block 3: VLM2Vec-V2-2B
VLM2Vec-V2-2B & 48.55 & 50.34 & 49.45 & 78.98 & 38.51 & 82.21 & 64.57 & 87.64 & 44.68 & 85.06 & 82.99 & 89.89 & 87.08 & 74.16 \\
\hline
\rowcolor{gray!10}
\texttt{s2m-add} & 43.33 & 39.00 & 41.17\blue{8.28} & 66.60 & 38.47 & 72.80 & 58.28 & 65.85 & \textbf{54.50} & 90.10 & 84.97 & 83.93 & 80.36 & 69.59\blue{4.57} \\
\texttt{s2m-mul} & 45.72 & 40.66 & 43.19\blue{6.26} & 68.04 & 39.80 & 75.45 & 58.87 & 69.90 & 54.31 & 90.93 & 84.53 & 84.56 & 82.40 & 70.88\blue{3.28} \\
\rowcolor{gray!10}
\texttt{cl-t-c} & 20.08 & 18.48 & 19.28\blue{30.17} & 28.47 & 6.46 & 29.32 & 20.51 & 40.06 & 17.75 & 20.23 & 35.23 & 27.90 & 22.54 & 24.85\blue{49.31} \\
\texttt{cl-t-m} & 25.04 & 21.94 & 23.49\blue{25.96} & 44.82 & 7.28 & 42.98 & 31.23 & 38.96 & 21.19 & 26.38 & 47.14 & 45.34 & 37.44 & 34.28\blue{39.88} \\
\rowcolor{gray!10}
\texttt{cl-s-c} & 23.25 & 18.43 & 20.84\blue{28.61} & 29.34 & 8.76 & 27.77 & 22.92 & 40.48 & 21.16 & 20.26 & 30.91 & 24.61 & 27.64 & 25.39\blue{48.77} \\
\texttt{cl-s-m} & 26.08 & 22.30 & 24.19\blue{25.26} & 44.95 & 7.16 & 45.42 & 19.48 & 39.29 & 25.44 & 31.57 & 47.18 & 48.84 & 40.13 & 34.95\blue{39.21} \\
\rowcolor{gray!10}
\texttt{c-sem} & 29.94 & 27.99 & 28.97\blue{20.48} & 61.30 & 17.28 & 62.19 & 40.55 & 55.53 & 33.02 & 56.52 & 62.33 & 68.54 & 70.10 & 52.74\blue{21.42} \\
\texttt{multi-img} & 38.69 & 29.00 & 33.85\blue{15.60} & 65.39 & 25.80 & 70.54 & 27.92 & 71.56 & 33.79 & 55.93 & 62.93 & 72.27 & 53.14 & 53.93\blue{20.23} \\
\hline
\rowcolor{gray!10}
\ourmethod & \textbf{49.49} & \textbf{50.53} & \textbf{50.01}\red{0.56} & \textbf{80.17} & \textbf{46.33} & \textbf{83.53} & \textbf{72.76} & \textbf{86.74} & 52.40 & \textbf{91.36} & \textbf{85.83} & \textbf{95.47} & \textbf{89.52} & \textbf{78.41}\red{4.25} \\
\Xhline{1.5pt}
% Block 4: LamRA-Ret
LamRA-Ret & 19.78 & 13.24 & 16.51 & 29.31 & 19.56 & 63.00 & 15.83 & 51.44 & 7.70 & 21.10 & 29.81 & 37.18 & 31.95 & 30.69 \\
\hline
\rowcolor{gray!10}
\texttt{s2m-add} & \textbf{32.18} & 17.82 & 25.00\red{8.49} & 9.80 & 14.37 & 46.06 & 19.49 & 28.13 & \textbf{19.16} & 22.79 & 30.98 & 37.98 & 24.81 & 25.36\blue{5.33} \\
\texttt{s2m-mul} & 30.52 & 16.37 & 23.45\red{6.94} & 9.71 & 14.98 & 45.61 & 17.37 & 27.91 & 17.20 & 20.45 & 28.32 & 38.60 & 23.01 & 24.32\blue{6.37} \\
\rowcolor{gray!10}
\texttt{cl-t-c} & 14.88 & 8.09 & 11.49\blue{5.02} & 5.57 & 2.22 & 19.76 & 13.94 & 17.55 & 6.59 & 14.23 & 20.69 & 17.83 & 20.52 & 13.89\blue{16.80} \\
\texttt{cl-t-m} & 20.43 & 13.28 & 16.86\red{0.35} & 7.47 & 3.85 & 30.91 & 17.54 & 13.41 & 12.37 & 27.06 & 35.63 & 39.20 & 31.88 & 21.93\blue{8.76} \\
\rowcolor{gray!10}
\texttt{cl-s-c} & 15.32 & 9.05 & 12.19\blue{4.32} & 6.78 & 2.62 & 20.53 & 13.79 & 18.30 & 9.11 & 17.05 & 27.03 & 25.60 & 22.25 & 16.31\blue{14.38} \\
\texttt{cl-s-m} & 19.84 & 12.91 & 16.38\blue{0.13} & 7.75 & 4.82 & 32.16 & 15.61 & 14.02 & 10.23 & 24.81 & 34.66 & 37.31 & 27.20 & 20.86\blue{9.83} \\
\rowcolor{gray!10}
\texttt{c-sem} & 19.62 & 10.69 & 15.16\blue{1.35} & 13.51 & 5.91 & 34.43 & 9.13 & 32.30 & 15.94 & \textbf{26.32} & \textbf{36.02} & 39.82 & \textbf{37.01} & 25.04\blue{5.65} \\
\texttt{multi-img} & 23.32 & 7.71 & 15.52\blue{0.99} & 6.90 & 6.10 & 21.26 & 8.00 & 22.26 & 13.60 & 13.91 & 15.91 & 20.30 & 9.67 & 13.79\blue{16.90} \\
\hline
\rowcolor{gray!10}
\ourmethod & 30.74 & \textbf{19.50} & \textbf{25.12}\red{8.61} & \textbf{17.27} & \textbf{20.61} & \textbf{58.35} & \textbf{21.39} & \textbf{39.69} & 13.34 & 25.57 & 35.27 & \textbf{43.21} & 26.29 & \textbf{30.10}\blue{0.59} \\
\Xhline{1.5pt}
% Block 5: GME-2B
GME-2B & 52.07 & 53.14 & 52.61 & 82.59 & 56.46 & 88.97 & 89.72 & 93.20 & 70.33 & 98.49 & 92.15 & 98.15 & 95.65 & 86.57 \\
\hline
\rowcolor{gray!10}
\texttt{s2m-add} & 50.92 & 46.52 & 48.72\blue{3.89} & 71.85 & 43.00 & 83.58 & 72.17 & 77.65 & 69.45 & 93.29 & 89.80 & 89.92 & 90.50 & 78.12\blue{8.45} \\
\texttt{s2m-mul} & \textbf{53.82} & 50.06 & 51.94\blue{0.67} & 76.87 & 47.23 & 85.49 & 81.53 & 84.77 & \textbf{74.63} & 95.72 & \textbf{93.07} & 93.85 & 92.18 & 82.53\blue{4.04} \\
\rowcolor{gray!10}
\texttt{cl-t-c} & 15.54 & 10.26 & 12.90\blue{39.71} & 14.96 & 4.70 & 25.32 & 16.93 & 24.97 & 12.01 & 12.02 & 19.02 & 19.26 & 19.34 & 16.85\blue{69.72} \\
\texttt{cl-t-m} & 17.95 & 13.61 & 15.78\blue{36.83} & 30.20 & 3.58 & 33.07 & 30.13 & 31.16 & 16.11 & 22.27 & 33.99 & 36.32 & 33.28 & 27.01\blue{59.56} \\
\rowcolor{gray!10}
\texttt{cl-s-c} & 16.03 & 12.59 & 14.31\blue{38.30} & 15.90 & 5.88 & 22.23 & 26.42 & 26.16 & 15.36 & 20.67 & 18.52 & 25.26 & 22.00 & 19.84\blue{66.73} \\
\texttt{cl-s-m} & 16.39 & 12.39 & 14.39\blue{38.22} & 31.79 & 4.01 & 30.55 & 19.79 & 31.61 & 14.18 & 24.92 & 30.34 & 34.24 & 28.50 & 24.99\blue{61.58} \\
\rowcolor{gray!10}
\texttt{c-sem} & 23.78 & 17.15 & 20.47\blue{32.14} & 45.19 & 11.49 & 51.57 & 25.34 & 42.78 & 27.26 & 51.26 & 44.73 & 49.30 & 53.58 & 40.25\blue{46.32} \\
\texttt{multi-img} & 45.10 & 32.87 & 38.99\blue{13.62} & 67.63 & 27.89 & 74.78 & 47.02 & 78.34 & 43.20 & 59.99 & 63.76 & 73.68 & 63.19 & 59.95\blue{26.62} \\
\hline
\rowcolor{gray!10}
\ourmethod & 53.06 & \textbf{54.24} & \textbf{53.65}\red{1.04} & \textbf{82.39} & \textbf{54.11} & \textbf{88.93} & \textbf{88.59} & \textbf{92.33} & 70.65 & \textbf{97.75} & 92.30 & \textbf{97.91} & \textbf{96.10} & \textbf{86.11}\blue{0.46} \\
\Xhline{1.5pt}
% Block 6: GME-7B
GME-7B & 54.01 & 55.80 & 54.91 & 87.59 & 56.05 & 91.96 & 94.25 & 93.72 & 76.26 & 99.63 & 95.45 & 99.63 & 99.06 & 89.36 \\
\hline
\rowcolor{gray!10}
\texttt{s2m-add} & 53.57 & 49.76 & 51.67\blue{3.24} & 75.91 & 46.41 & 85.01 & 80.64 & 83.47 & 74.66 & 95.72 & 92.93 & 94.35 & 92.17 & 82.13\blue{7.23} \\
\texttt{s2m-mul} & 50.73 & 45.95 & 48.34\blue{6.57} & 72.29 & 43.38 & 83.23 & 72.85 & 78.28 & 69.64 & 92.92 & 90.23 & 90.05 & 90.17 & 78.30\blue{11.06} \\
\rowcolor{gray!10}
\texttt{cl-t-c} & 12.53 & 8.39 & 10.46\blue{44.45} & 7.32 & 5.16 & 17.36 & 18.86 & 16.81 & 12.83 & 13.28 & 16.18 & 20.30 & 13.97 & 14.21\blue{75.15} \\
\texttt{cl-t-m} & 14.64 & 9.03 & 11.84\blue{43.07} & 6.78 & 3.80 & 26.36 & 24.67 & 17.82 & 15.52 & 19.12 & 26.65 & 24.74 & 28.60 & 19.41\blue{69.95} \\
\rowcolor{gray!10}
\texttt{cl-s-c} & 13.08 & 8.93 & 11.01\blue{43.90} & 7.25 & 3.85 & 16.83 & 17.08 & 18.68 & 13.51 & 19.17 & 20.53 & 19.15 & 21.64 & 15.77\blue{73.59} \\
\texttt{cl-s-m} & 13.62 & 7.93 & 10.78\blue{44.13} & 6.88 & 4.73 & 22.77 & 18.40 & 18.43 & 10.43 & 16.80 & 21.03 & 26.40 & 20.86 & 16.67\blue{72.69} \\
\rowcolor{gray!10}
\texttt{c-sem} & 15.42 & 9.22 & 12.32\blue{42.59} & 15.02 & 6.13 & 27.52 & 18.27 & 36.39 & 22.53 & 20.40 & 28.91 & 30.81 & 26.13 & 23.21\blue{66.15} \\
\texttt{multi-img} & 47.50 & 36.01 & 41.76\blue{13.15} & 72.48 & 33.71 & 78.84 & 45.50 & 84.88 & 45.65 & 64.59 & 68.82 & 75.90 & 69.50 & 63.99\blue{25.37} \\
\hline
\rowcolor{gray!10}
\ourmethod & \textbf{54.96} & \textbf{56.51} & \textbf{55.74}\red{0.83} & \textbf{87.35} & \textbf{57.91} & \textbf{90.76} & \textbf{95.35} & \textbf{95.44} & \textbf{75.92} & \textbf{99.63} & \textbf{94.67} & \textbf{99.63} & \textbf{98.89} & \textbf{89.56}\red{0.20} \\
\Xhline{1.5pt}
% Block 7: UniME-V2-2B
UniME-V2-2B & 18.52 & 40.10 & 29.31 & 36.52 & 12.43 & 42.41 & 14.09 & 51.11 & 7.39 & 20.23 & 32.96 & 24.21 & 19.25 & 26.06 \\
\hline
\rowcolor{gray!10}
\texttt{s2m-add} & 36.66 & 30.03 & 33.35\red{4.04} & 50.78 & 25.04 & 58.61 & 37.71 & 54.90 & 36.67 & 68.00 & 68.50 & 69.21 & 68.39 & 53.78\red{27.72} \\
\texttt{s2m-mul} & 36.06 & 28.97 & 32.52\red{3.21} & 50.76 & 23.93 & 58.23 & 32.85 & 54.90 & 34.94 & 63.01 & 64.67 & 66.31 & 61.07 & 51.07\red{25.01} \\
\rowcolor{gray!10}
\texttt{cl-t-c} & 19.45 & 14.04 & 16.75\blue{12.56} & 19.03 & 7.04 & 25.79 & 22.55 & 30.85 & 14.57 & 13.25 & 21.82 & 32.30 & 26.67 & 21.39\blue{4.67} \\
\texttt{cl-t-m} & 16.70 & 12.63 & 14.67\blue{14.64} & 21.88 & 4.30 & 33.79 & 26.88 & 20.47 & 9.65 & 15.83 & 36.40 & 34.18 & 27.10 & 23.05\blue{3.01} \\
\rowcolor{gray!10}
\texttt{cl-s-c} & 19.23 & 16.02 & 17.63\blue{11.68} & 19.53 & 6.89 & 23.75 & 21.96 & 28.29 & 15.97 & 17.34 & 31.06 & 33.58 & 27.73 & 22.61\blue{3.45} \\
\texttt{cl-s-m} & 17.77 & 13.19 & 15.48\blue{13.83} & 23.67 & 5.33 & 36.60 & 32.03 & 20.17 & 11.46 & 19.46 & 37.16 & 36.87 & 26.76 & 24.95\blue{1.11} \\
\rowcolor{gray!10}
\texttt{c-sem} & 25.11 & 20.89 & 23.00\blue{6.31} & 38.24 & 13.31 & 51.27 & 38.76 & 40.04 & 21.25 & 33.42 & 51.86 & 60.03 & 52.80 & 40.10\red{14.04} \\
\texttt{multi-img} & 33.99 & 24.39 & 29.19\blue{0.12} & 50.79 & 19.08 & 60.61 & 30.15 & 58.91 & 24.36 & 43.08 & 46.76 & 61.06 & 50.79 & 44.56\red{18.50} \\
\hline
\rowcolor{gray!10}
\ourmethod & \textbf{44.22} & \textbf{44.19} & \textbf{44.21}\red{14.90} & \textbf{62.39} & \textbf{37.69} & \textbf{73.33} & \textbf{72.35} & \textbf{77.45} & \textbf{38.83} & \textbf{82.50} & \textbf{75.80} & \textbf{89.35} & \textbf{85.84} & \textbf{69.55}\red{43.49} \\
\Xhline{1.5pt}
% Block 8: UniME-V2-7B
UniME-V2-7B & 33.19 & 45.72 & 39.46 & 63.23 & 24.91 & 65.25 & 11.16 & 41.54 & 14.18 & 41.89 & 40.56 & 57.44 & 42.78 & 40.29 \\
\hline
\rowcolor{gray!10}
\texttt{s2m-add} & 40.28 & 37.61 & 38.95\blue{0.51} & 55.17 & 32.79 & 69.82 & 58.54 & 65.47 & \textbf{45.46} & 84.23 & 81.72 & 87.38 & 89.16 & 66.97\red{26.68} \\
\texttt{s2m-mul} & 40.57 & 38.14 & 39.36\blue{0.10} & 55.29 & 33.72 & 70.38 & 56.55 & 65.82 & 44.48 & 84.10 & 81.55 & 86.00 & \textbf{89.89} & 66.78\red{26.49} \\
\rowcolor{gray!10}
\texttt{cl-t-c} & 20.22 & 13.70 & 16.96\blue{22.50} & 24.63 & 5.79 & 29.62 & 27.36 & 21.81 & 15.81 & 17.79 & 34.01 & 27.97 & 28.61 & 23.34\blue{16.95} \\
\texttt{cl-t-m} & 23.20 & 19.90 & 21.55\blue{17.91} & 34.15 & 5.21 & 49.47 & 34.87 & 32.02 & 18.91 & 35.36 & 47.97 & 43.55 & 56.22 & 35.77\blue{4.52} \\
\rowcolor{gray!10}
\texttt{cl-s-c} & 20.91 & 18.25 & 19.58\blue{19.88} & 23.98 & 8.28 & 28.30 & 26.59 & 28.51 & 21.08 & 25.63 & 32.59 & 30.66 & 43.46 & 26.91\blue{13.38} \\
\texttt{cl-s-m} & 24.81 & 19.71 & 22.26\blue{17.20} & 35.09 & 6.60 & 49.60 & 36.77 & 33.75 & 20.75 & 37.45 & 48.31 & 55.85 & 60.28 & 38.45\blue{1.84} \\
\rowcolor{gray!10}
\texttt{c-sem} & 31.71 & 29.01 & 30.36\blue{9.10} & 59.37 & 19.84 & 64.17 & 42.06 & 50.69 & 34.28 & 70.63 & 66.71 & 72.14 & 75.85 & 55.57\red{15.28} \\
\texttt{multi-img} & 34.79 & 28.05 & 31.42\blue{8.04} & 53.99 & 23.06 & 67.47 & 40.81 & 70.58 & 30.36 & 52.38 & 65.49 & 72.19 & 62.65 & 53.90\red{13.61} \\
\hline
\rowcolor{gray!10}
\ourmethod & \textbf{45.90} & \textbf{48.26} & \textbf{47.08}\red{7.62} & \textbf{64.78} & \textbf{37.43} & \textbf{78.51} & \textbf{76.74} & \textbf{81.47} & 43.69 & \textbf{89.32} & \textbf{82.68} & \textbf{92.74} & 88.13 & \textbf{73.55}\red{33.26} \\
\Xhline{1.5pt}
% Block 9: B3-2B
B3-2B & 37.10 & 32.07 & 34.59 & 57.00 & 29.38 & 68.09 & 48.31 & 71.55 & 18.09 & 74.13 & 64.64 & 75.44 & 63.13 & 56.98 \\
\hline
\rowcolor{gray!10}
\texttt{s2m-add} & 35.86 & 27.36 & 31.61\blue{2.98} & 47.93 & 24.57 & 64.85 & 39.99 & 50.16 & \textbf{33.16} & 66.24 & 63.90 & 69.75 & 65.45 & 52.60\blue{4.38} \\
\texttt{s2m-mul} & 35.66 & 26.97 & 31.32\blue{3.27} & 47.89 & 24.47 & 64.82 & 35.42 & 49.59 & 32.43 & 63.95 & 63.09 & 66.79 & 64.47 & 51.29\blue{5.69} \\
\rowcolor{gray!10}
\texttt{cl-t-c} & 17.88 & 13.56 & 15.72\blue{18.87} & 19.64 & 4.53 & 32.73 & 10.60 & 20.15 & 9.10 & 12.68 & 26.80 & 30.88 & 18.64 & 18.58\blue{38.40} \\
\texttt{cl-t-m} & 18.36 & 14.95 & 16.66\blue{17.93} & 18.38 & 5.34 & 33.44 & 13.13 & 19.91 & 12.42 & 14.05 & 32.81 & 30.84 & 26.31 & 20.66\blue{36.32} \\
\rowcolor{gray!10}
\texttt{cl-s-c} & 20.70 & 16.46 & 18.58\blue{16.01} & 16.92 & 5.67 & 34.77 & 16.79 & 22.56 & 13.04 & 22.41 & 34.47 & 37.31 & 36.75 & 24.07\blue{32.91} \\
\texttt{cl-s-m} & 20.41 & 17.63 & 19.02\blue{15.57} & 19.74 & 4.72 & 35.90 & 17.36 & 20.74 & 15.77 & 19.76 & 38.43 & 36.44 & 33.89 & 24.28\blue{32.70} \\
\rowcolor{gray!10}
\texttt{c-sem} & 23.80 & 21.20 & 22.50\blue{12.09} & 34.98 & 11.15 & 51.26 & 21.98 & 37.45 & 19.41 & 29.16 & 49.03 & 54.62 & 46.73 & 35.58\blue{21.40} \\
\texttt{multi-img} & 28.93 & 16.01 & 22.47\blue{12.12} & 41.03 & 13.38 & 46.20 & 20.46 & 46.89 & 10.26 & 31.74 & 38.83 & 32.44 & 32.80 & 31.40\blue{25.58} \\
\hline
\rowcolor{gray!10}
\ourmethod & \textbf{42.06} & \textbf{37.60} & \textbf{39.83}\red{5.24} & \textbf{56.47} & \textbf{30.91} & \textbf{66.69} & \textbf{67.42} & \textbf{69.33} & 29.42 & \textbf{79.88} & \textbf{72.67} & \textbf{83.24} & \textbf{71.41} & \textbf{62.74}\red{5.76} \\
\Xhline{1.5pt}
% Block 10: B3-7B
B3-7B & 46.09 & 45.10 & 45.60 & 68.95 & 43.38 & 79.86 & 66.56 & 84.12 & 37.06 & 81.01 & 81.25 & 88.57 & 81.30 & 71.21 \\
\hline
\rowcolor{gray!10}
\texttt{s2m-add} & 44.95 & 40.38 & 42.67\blue{2.93} & 59.11 & 38.45 & 75.42 & 69.63 & 70.95 & 51.71 & \textbf{88.55} & 81.68 & 86.10 & 86.34 & 70.79\blue{0.42} \\
\texttt{s2m-mul} & 45.11 & 40.61 & 42.86\blue{2.74} & 59.50 & 38.63 & 75.74 & 69.25 & 71.18 & \textbf{51.72} & 87.36 & 82.07 & 85.17 & 86.23 & 70.69\blue{0.52} \\
\rowcolor{gray!10}
\texttt{cl-t-c} & 23.96 & 19.73 & 21.85\blue{23.75} & 25.53 & 8.05 & 47.76 & 13.93 & 32.90 & 19.17 & 29.29 & 44.67 & 45.02 & 33.61 & 29.99\blue{41.22} \\
\texttt{cl-t-m} & 24.79 & 21.29 & 23.04\blue{22.56} & 31.72 & 9.21 & 52.63 & 18.92 & 28.84 & 22.25 & 31.48 & 47.83 & 51.71 & 43.89 & 33.85\blue{37.36} \\
\rowcolor{gray!10}
\texttt{cl-s-c} & 25.66 & 21.29 & 23.48\blue{22.12} & 24.40 & 11.69 & 48.70 & 23.29 & 33.17 & 25.56 & 34.54 & 42.09 & 55.93 & 52.12 & 35.15\blue{36.06} \\
\texttt{cl-s-m} & 25.52 & 20.72 & 23.12\blue{22.48} & 24.78 & 11.31 & 49.68 & 26.51 & 32.71 & 26.20 & 37.63 & 46.29 & 51.53 & 44.47 & 35.11\blue{36.10} \\
\rowcolor{gray!10}
\texttt{c-sem} & 29.88 & 25.98 & 27.93\blue{17.67} & 56.08 & 24.05 & 65.81 & 25.04 & 52.18 & 33.15 & 59.78 & 57.76 & 67.14 & 70.35 & 51.13\blue{20.08} \\
\texttt{multi-img} & 35.37 & 22.45 & 28.91\blue{16.69} & 52.60 & 22.38 & 56.80 & 27.96 & 65.24 & 16.05 & 53.57 & 47.96 & 43.69 & 51.39 & 43.76\blue{27.45} \\
\hline
\rowcolor{gray!10}
\ourmethod & \textbf{49.11} & \textbf{48.39} & \textbf{48.75}\red{3.15} & \textbf{67.68} & \textbf{42.17} & \textbf{79.02} & \textbf{78.06} & \textbf{81.64} & 47.60 & 85.17 & \textbf{82.04} & \textbf{92.00} & \textbf{88.73} & \textbf{74.41}\red{3.20} \\
\Xhline{1.5pt}
\end{longtable}
\end{center}

\twocolumn
% \clearpage

\onecolumn

% We recommend wrapping this in a \onecolumn environment for two-column papers.
\begin{center}
\tiny % Use smaller font for wide tables
\setlength\tabcolsep{4.0pt} % Adjust column separation
\renewcommand\arraystretch{1.2} % Adjust row height
\begin{longtable}{p{2.5cm}|cccc|c|cc|c|cccccc|c}
\caption{Performance comparison on ViDoRe-V2, ViDoSeek, and VisRAG benchmarks. For each model block, we bold the best-performing optimization method in each column (except for the base result). The average scores for optimizations are shown with relative gains ($\color{RedOrange}\uparrow\color{black}$/$\color{BlueGreen}\downarrow\color{black}$) compared to the base model.}
\label{tab:full_results_vidorev2_vidoseek_visrag} \\
\Xhline{1.5pt}
\rowcolor{gray!20}
\multirow{2}{*}{\textbf{Method}} & \multicolumn{5}{c|}{\textbf{ViDoRe-V2}} & \multicolumn{3}{c|}{\textbf{ViDoSeek}} & \multicolumn{7}{c}{\textbf{VisRAG}} \\
\cmidrule(lr){2-6} \cmidrule(lr){7-9} \cmidrule(lr){10-16}
\rowcolor{gray!20}
& \textbf{Bio-L} & \textbf{Eco-R} & \textbf{ESG-H} & \textbf{ESG-M} & \textbf{Avg.} & \textbf{Doc} & \textbf{Page} & \textbf{Avg.} & \textbf{Arxiv} & \textbf{Chart} & \textbf{InfoV} & \textbf{MP-Doc} & \textbf{Plot} & \textbf{Slide} & \textbf{Avg.} \\
\Xhline{1.5pt}
\endfirsthead
\Xhline{1.5pt}
\rowcolor{gray!20}
\multicolumn{16}{c}{\tablename\ \thetable\ -- \textit{Continued from previous page}} \\
\Xhline{1.5pt}
\rowcolor{gray!20}
\multirow{2}{*}{\textbf{Method}} & \multicolumn{5}{c|}{\textbf{ViDoRe-V2}} & \multicolumn{3}{c|}{\textbf{ViDoSeek}} & \multicolumn{7}{c}{\textbf{VisRAG}} \\
\cmidrule(lr){2-6} \cmidrule(lr){7-9} \cmidrule(lr){10-16}
\rowcolor{gray!20}
& \textbf{Bio-L} & \textbf{Eco-R} & \textbf{ESG-H} & \textbf{ESG-M} & \textbf{Avg.} & \textbf{Doc} & \textbf{Page} & \textbf{Avg.} & \textbf{Arxiv} & \textbf{Chart} & \textbf{InfoV} & \textbf{MP-Doc} & \textbf{Plot} & \textbf{Slide} & \textbf{Avg.} \\
\Xhline{1.5pt}
\endhead
\Xhline{1.5pt}
\multicolumn{16}{r}{\textit{Continued on next page}} \\
\endfoot
\Xhline{1.5pt}
\endlastfoot
%
% Block 1: VLM2Vec-V1-2B
VLM2Vec-V1-2B & 6.88 & 14.15 & 12.25 & 20.54 & 13.46 & 56.40 & 67.73 & 62.07 & 41.68 & 58.21 & 70.79 & 42.74 & 23.83 & 74.07 & 51.89 \\
\hline
\rowcolor{gray!10}
\texttt{s2m-add} & 10.40 & 8.59 & 4.68 & 3.46 & 6.78\blue{6.68} & 40.78 & 31.23 & 36.01\blue{26.06} & 28.94 & 44.79 & 59.76 & 30.86 & 9.67 & 59.59 & 38.94\blue{12.95} \\
\texttt{s2m-mul} & 10.00 & 8.40 & 5.09 & 3.51 & 6.75\blue{6.71} & 40.63 & 31.00 & 35.82\blue{26.25} & 28.85 & 44.64 & 60.48 & 30.55 & 9.83 & 59.52 & 38.98\blue{12.91} \\
\rowcolor{gray!10}
\texttt{cl-t-c} & 14.86 & 14.49 & 4.14 & 4.48 & 9.49\blue{3.97} & 42.17 & 31.65 & 36.91\blue{25.16} & 13.56 & 18.52 & 25.66 & 13.11 & 1.41 & 31.91 & 17.36\blue{34.53} \\
\texttt{cl-t-m} & 13.34 & 4.60 & 8.78 & 5.15 & 7.97\blue{5.49} & 36.71 & 26.80 & 31.76\blue{30.31} & 8.29 & 18.27 & 37.74 & 11.85 & 1.13 & 32.46 & 18.29\blue{33.60} \\
\rowcolor{gray!10}
\texttt{cl-s-c} & 16.38 & 10.03 & 4.93 & 4.64 & 9.00\blue{4.46} & 43.89 & 31.95 & 37.92\blue{24.15} & 13.03 & 27.01 & 28.61 & 18.82 & 1.54 & 31.31 & 20.05\blue{31.84} \\
\texttt{cl-s-m} & 13.46 & 6.50 & 6.89 & 3.74 & 7.65\blue{5.81} & 37.20 & 26.86 & 32.03\blue{30.04} & 8.21 & 22.04 & 37.37 & 12.65 & 1.15 & 33.31 & 19.12\blue{32.77} \\
\rowcolor{gray!10}
\texttt{c-sem} & 20.32 & 11.01 & 11.86 & 10.44 & 13.41\blue{0.05} & 56.19 & 47.96 & 52.08\blue{9.99} & 16.53 & 44.39 & 43.88 & 24.77 & 3.86 & 54.32 & 31.29\blue{20.60} \\
\texttt{multi-img} & 15.59 & 14.90 & 5.70 & 5.95 & 10.54\blue{2.92} & 45.24 & 34.26 & 39.75\blue{22.32} & 30.20 & 41.53 & 60.40 & 29.53 & 8.44 & 60.10 & 38.37\blue{13.52} \\
\hline
\rowcolor{gray!10}
\ourmethod & \textbf{30.33} & \textbf{29.55} & \textbf{33.21} & \textbf{38.33} & \textbf{32.86}\red{19.40} & \textbf{75.23} & \textbf{70.19} & \textbf{72.71}\red{10.64} & \textbf{38.18} & \textbf{60.09} & \textbf{69.44} & \textbf{48.29} & \textbf{18.83} & \textbf{76.95} & \textbf{51.96}\red{0.07} \\
\Xhline{1.5pt}
% Block 2: VLM2Vec-V1-7B
VLM2Vec-V1-7B & 4.93 & 13.74 & 6.82 & 11.27 & 9.19 & 54.26 & 77.39 & 65.83 & 52.58 & 69.83 & 71.43 & 52.86 & 34.24 & 73.22 & 59.03 \\
\hline
\rowcolor{gray!10}
\texttt{s2m-add} & 29.49 & 38.26 & 31.73 & 22.80 & 30.57\red{21.38} & 62.50 & 53.64 & 58.07\blue{7.76} & 45.56 & 48.98 & 68.39 & 49.01 & 11.01 & 70.61 & 48.93\blue{10.10} \\
\texttt{s2m-mul} & 28.79 & 37.40 & 29.66 & 21.43 & 29.32\red{20.13} & 62.34 & 52.73 & 57.54\blue{8.29} & 45.29 & 50.21 & 69.10 & 49.54 & 11.02 & 70.82 & 49.33\blue{9.70} \\
\rowcolor{gray!10}
\texttt{cl-t-c} & 17.52 & 18.22 & 11.68 & 8.82 & 14.06\red{4.87} & 49.02 & 40.25 & 44.64\blue{21.19} & 15.12 & 21.70 & 36.61 & 19.07 & 2.22 & 37.89 & 22.10\blue{36.93} \\
\texttt{cl-t-m} & 22.08 & 14.23 & 24.74 & 15.61 & 19.17\red{9.98} & 58.38 & 51.13 & 54.76\blue{11.07} & 18.75 & 30.49 & 58.40 & 22.24 & 3.90 & 52.32 & 31.02\blue{28.01} \\
\rowcolor{gray!10}
\texttt{cl-s-c} & 19.74 & 22.03 & 14.36 & 12.74 & 17.22\red{8.03} & 51.57 & 41.97 & 46.77\blue{19.06} & 15.24 & 17.35 & 36.59 & 24.24 & 1.79 & 35.59 & 21.80\blue{37.23} \\
\texttt{cl-s-m} & 23.08 & 16.08 & 20.53 & 14.47 & 18.54\red{9.35} & 58.59 & 50.84 & 54.72\blue{11.11} & 18.76 & 31.65 & 58.91 & 24.90 & 3.10 & 52.63 & 31.66\blue{27.37} \\
\rowcolor{gray!10}
\texttt{c-sem} & 29.88 & 31.92 & 30.24 & 20.34 & 28.10\red{18.91} & 72.03 & 66.00 & 69.02\red{3.19} & 39.77 & 57.94 & 63.94 & 41.41 & 14.39 & 69.66 & 47.85\blue{11.18} \\
\texttt{multi-img} & 30.91 & 39.93 & 23.80 & 17.07 & 27.93\red{18.74} & 62.03 & 49.71 & 55.87\blue{9.96} & 45.32 & 44.84 & 64.95 & 40.71 & 10.55 & 67.95 & 45.72\blue{13.31} \\
\hline
\rowcolor{gray!10}
\ourmethod & \textbf{42.63} & \textbf{42.89} & \textbf{50.55} & \textbf{42.86} & \textbf{44.73}\red{35.54} & \textbf{78.34} & \textbf{78.61} & \textbf{78.48}\red{12.65} & \textbf{54.43} & \textbf{70.30} & \textbf{69.27} & \textbf{58.49} & \textbf{33.46} & \textbf{77.98} & \textbf{60.66}\red{1.63} \\
\Xhline{1.5pt}
% Block 3: VLM2Vec-V2-2B
VLM2Vec-V2-2B & 44.45 & 45.77 & 48.77 & 46.98 & 46.49 & 80.88 & 83.68 & 82.28 & 77.38 & 82.30 & 86.27 & 71.60 & 66.96 & 92.04 & 79.43 \\
\hline
\rowcolor{gray!10}
\texttt{s2m-add} & 41.02 & 49.68 & 41.29 & 20.26 & 38.06\blue{8.43} & 69.01 & 64.83 & 66.92\blue{15.36} & 62.37 & 73.06 & 78.41 & 75.52 & 22.67 & 85.73 & 66.29\blue{13.14} \\
\texttt{s2m-mul} & 42.43 & 50.96 & 42.90 & 21.91 & 39.55\blue{6.94} & 71.66 & 66.41 & 69.04\blue{13.24} & 63.30 & 73.50 & 79.84 & 76.41 & 23.36 & 87.29 & 67.28\blue{12.15} \\
\rowcolor{gray!10}
\texttt{cl-t-c} & 19.84 & 21.50 & 12.92 & 9.75 & 16.00\blue{30.49} & 48.48 & 49.99 & 49.24\blue{33.04} & 25.86 & 22.73 & 34.54 & 19.89 & 10.59 & 36.62 & 25.04\blue{54.39} \\
\texttt{cl-t-m} & 25.10 & 18.29 & 15.42 & 12.37 & 17.80\blue{28.69} & 57.69 & 59.96 & 58.83\blue{23.45} & 37.64 & 42.56 & 56.47 & 24.76 & 11.54 & 51.93 & 37.48\blue{41.95} \\
\rowcolor{gray!10}
\texttt{cl-s-c} & 23.22 & 18.36 & 10.47 & 12.71 & 16.19\blue{30.30} & 49.73 & 51.97 & 50.85\blue{31.43} & 28.13 & 27.93 & 34.99 & 25.08 & 10.77 & 40.66 & 27.93\blue{51.50} \\
\texttt{cl-s-m} & 27.68 & 18.68 & 17.65 & 9.93 & 18.49\blue{28.00} & 57.31 & 61.50 & 59.41\blue{22.87} & 37.55 & 43.07 & 55.65 & 26.56 & 11.74 & 52.27 & 37.81\blue{41.62} \\
\rowcolor{gray!10}
\texttt{c-sem} & 36.74 & 31.11 & 20.49 & 21.25 & 27.40\blue{19.09} & 71.91 & 71.47 & 71.69\blue{10.59} & 53.17 & 61.34 & 67.94 & 53.96 & 23.78 & 74.81 & 55.83\blue{23.60} \\
\texttt{multi-img} & 30.01 & 45.94 & 23.58 & 17.56 & 29.27\blue{17.22} & 63.17 & 50.95 & 57.06\blue{25.22} & 62.83 & 61.38 & 75.92 & 54.34 & 24.62 & 75.43 & 59.09\blue{20.34} \\
\hline
\rowcolor{gray!10}
\ourmethod & \textbf{50.06} & \textbf{53.76} & \textbf{57.41} & \textbf{46.40} & \textbf{51.91}\red{5.42} & \textbf{80.94} & \textbf{83.87} & \textbf{82.41}\red{0.13} & \textbf{77.18} & \textbf{78.05} & \textbf{84.37} & \textbf{78.07} & \textbf{58.74} & \textbf{91.95} & \textbf{78.06}\blue{1.37} \\
\Xhline{1.5pt}
% Block 4: LamRA-Ret
LamRA-Ret & 10.75 & 9.65 & 6.32 & 11.18 & 9.48 & 60.17 & 28.81 & 44.49 & 11.17 & 63.50 & 59.78 & 33.57 & 29.42 & 57.59 & 42.51 \\
\hline
\rowcolor{gray!10}
\texttt{s2m-add} & 12.36 & \textbf{21.34} & 18.02 & 21.30 & 18.26\red{8.78} & 55.39 & 27.91 & 41.65\blue{2.84} & 2.75 & 33.31 & 44.51 & 28.61 & 4.86 & 49.08 & 27.19\blue{15.32} \\
\texttt{s2m-mul} & 9.91 & 17.36 & 21.00 & 20.88 & 17.29\red{7.81} & 52.53 & 25.95 & 39.24\blue{5.25} & 2.75 & 33.53 & 43.10 & 28.28 & 5.19 & 49.10 & 26.99\blue{15.52} \\
\rowcolor{gray!10}
\texttt{cl-t-c} & 6.69 & 7.44 & 5.85 & 1.05 & 5.26\blue{4.22} & 39.52 & 30.63 & 35.08\blue{9.41} & 4.26 & 19.69 & 21.94 & 9.90 & 1.53 & 26.54 & 13.98\blue{28.53} \\
\texttt{cl-t-m} & 7.75 & 10.84 & 8.20 & 3.97 & 7.69\blue{1.79} & 44.87 & 38.69 & 41.78\blue{2.71} & 5.46 & 22.85 & 38.15 & 17.45 & 2.13 & 35.03 & 20.18\blue{22.33} \\
\rowcolor{gray!10}
\texttt{cl-s-c} & 7.57 & 6.57 & 6.90 & 1.97 & 5.75\blue{3.73} & 38.41 & 30.98 & 34.70\blue{9.79} & 5.19 & 17.67 & 21.50 & 11.02 & 1.52 & 25.81 & 13.79\blue{28.72} \\
\texttt{cl-s-m} & 8.55 & 10.98 & 7.70 & 4.07 & 7.83\blue{1.65} & 45.50 & 38.86 & 42.18\blue{2.31} & 5.45 & 23.56 & 37.71 & 17.17 & 2.24 & 35.58 & 20.29\blue{22.22} \\
\rowcolor{gray!10}
\texttt{c-sem} & 8.67 & 9.41 & 8.69 & 4.43 & 7.80\blue{1.68} & 53.12 & 38.51 & 45.82\red{1.33} & \textbf{9.07} & 40.44 & 41.90 & 18.39 & 3.92 & 44.45 & 26.36\blue{16.15} \\
\texttt{multi-img} & 3.36 & 10.48 & 13.43 & 6.24 & 8.38\blue{1.10} & 28.18 & 10.96 & 19.57\blue{24.92} & 1.22 & 23.90 & 35.05 & 16.83 & 5.71 & 18.51 & 16.87\blue{25.64} \\
\hline
\rowcolor{gray!10}
\ourmethod & \textbf{15.81} & 17.65 & \textbf{23.75} & \textbf{25.48} & \textbf{20.67}\red{11.19} & \textbf{58.55} & \textbf{39.99} & \textbf{49.27}\red{4.78} & 5.91 & \textbf{60.76} & \textbf{54.91} & \textbf{38.77} & \textbf{25.27} & \textbf{62.34} & \textbf{41.33}\blue{1.18} \\
\Xhline{1.5pt}
% Block 5: GME-2B
GME-2B & 54.25 & 50.65 & 59.44 & 49.15 & 53.37 & 81.44 & 79.62 & 80.53 & 81.37 & 81.70 & 91.31 & 85.03 & 63.81 & 93.60 & 82.80 \\
\hline
\rowcolor{gray!10}
\texttt{s2m-add} & 48.33 & 46.46 & 40.98 & 30.02 & 41.45\blue{11.92} & 82.20 & 63.50 & 72.85\blue{7.68} & 70.43 & 77.10 & 85.04 & 80.94 & 20.83 & 89.96 & 70.72\blue{12.08} \\
\texttt{s2m-mul} & 51.15 & 49.41 & 46.69 & 28.65 & 43.98\blue{9.39} & 83.64 & 67.99 & 75.82\blue{4.71} & 74.25 & 77.91 & 88.07 & 84.75 & 22.40 & 92.99 & 73.40\blue{9.40} \\
\rowcolor{gray!10}
\texttt{cl-t-c} & 17.04 & 18.20 & 2.88 & 5.79 & 10.98\blue{42.39} & 34.72 & 28.03 & 31.38\blue{49.15} & 14.40 & 18.17 & 30.07 & 14.03 & 4.30 & 26.48 & 17.91\blue{64.89} \\
\texttt{cl-t-m} & 18.60 & 12.72 & 12.39 & 7.72 & 12.86\blue{40.51} & 45.17 & 40.84 & 43.01\blue{37.52} & 25.41 & 31.73 & 44.53 & 16.56 & 2.77 & 40.18 & 26.86\blue{55.94} \\
\rowcolor{gray!10}
\texttt{cl-s-c} & 17.47 & 19.49 & 10.88 & 5.99 & 13.46\blue{39.91} & 34.42 & 29.29 & 31.86\blue{48.67} & 15.10 & 14.91 & 29.78 & 15.49 & 3.92 & 28.86 & 18.01\blue{64.79} \\
\texttt{cl-s-m} & 21.66 & 11.82 & 14.37 & 5.63 & 13.37\blue{40.00} & 42.98 & 38.99 & 40.99\blue{39.54} & 24.76 & 30.32 & 45.48 & 17.41 & 2.60 & 39.35 & 26.65\blue{56.15} \\
\rowcolor{gray!10}
\texttt{c-sem} & 24.51 & 22.71 & 22.72 & 18.57 & 22.13\blue{31.24} & 65.58 & 53.29 & 59.44\blue{21.09} & 38.80 & 50.88 & 52.09 & 35.27 & 18.64 & 63.25 & 43.16\blue{39.64} \\
\texttt{multi-img} & 32.00 & 41.67 & 25.86 & 19.83 & 29.84\blue{23.53} & 73.00 & 47.58 & 60.29\blue{20.24} & 67.87 & 67.77 & 81.98 & 58.54 & 22.85 & 77.89 & 62.82\blue{19.98} \\
\hline
\rowcolor{gray!10}
\ourmethod & \textbf{55.14} & \textbf{52.08} & \textbf{56.24} & \textbf{51.36} & \textbf{53.71}\red{0.34} & \textbf{83.93} & \textbf{80.16} & \textbf{82.05}\red{1.52} & \textbf{81.65} & \textbf{83.87} & \textbf{91.06} & \textbf{85.89} & \textbf{63.17} & \textbf{93.76} & \textbf{83.23}\red{0.43} \\
\Xhline{1.5pt}
% Block 6: GME-7B
GME-7B & 53.66 & 54.34 & 65.38 & 54.32 & 56.93 & 83.21 & 84.18 & 83.70 & 87.20 & 82.32 & 92.92 & 88.89 & 63.36 & 94.81 & 84.92 \\
\hline
\rowcolor{gray!10}
\texttt{s2m-add} & 50.40 & 47.49 & 47.63 & 29.02 & 43.64\blue{13.29} & 83.22 & 67.67 & 75.45\blue{8.25} & 73.94 & 79.30 & 87.84 & 84.63 & 22.35 & 92.79 & 73.48\blue{11.44} \\
\texttt{s2m-mul} & 48.04 & 47.85 & 39.12 & 29.60 & 41.15\blue{15.78} & 82.30 & 63.00 & 72.65\blue{11.05} & 70.46 & 77.68 & 84.96 & 80.38 & 21.12 & 90.08 & 70.78\blue{14.14} \\
\rowcolor{gray!10}
\texttt{cl-t-c} & 10.12 & 12.81 & 8.12 & 4.38 & 8.86\blue{48.07} & 29.66 & 23.82 & 26.74\blue{56.96} & 3.79 & 11.21 & 21.97 & 11.20 & 2.44 & 18.30 & 11.49\blue{73.43} \\
\texttt{cl-t-m} & 9.38 & 10.36 & 6.42 & 4.63 & 7.70\blue{49.23} & 35.86 & 32.84 & 34.35\blue{49.35} & 3.56 & 21.05 & 35.14 & 12.97 & 1.45 & 24.43 & 16.43\blue{68.49} \\
\rowcolor{gray!10}
\texttt{cl-s-c} & 8.58 & 16.36 & 5.41 & 5.20 & 8.89\blue{48.04} & 28.62 & 23.40 & 26.01\blue{57.69} & 5.46 & 11.14 & 23.14 & 11.35 & 1.54 & 21.57 & 12.37\blue{72.55} \\
\texttt{cl-s-m} & 10.84 & 9.96 & 6.62 & 2.80 & 7.56\blue{49.37} & 34.06 & 29.69 & 31.88\blue{51.82} & 3.93 & 22.24 & 33.74 & 13.14 & 1.42 & 23.76 & 16.37\blue{68.55} \\
\rowcolor{gray!10}
\texttt{c-sem} & 16.30 & 13.79 & 8.34 & 8.42 & 11.71\blue{45.22} & 49.08 & 35.41 & 42.25\blue{41.45} & 7.40 & 29.57 & 32.02 & 15.54 & 5.76 & 35.26 & 20.93\blue{63.99} \\
\texttt{multi-img} & 29.32 & 39.47 & 30.32 & 16.76 & 28.97\blue{27.96} & 75.23 & 52.91 & 64.07\blue{19.63} & 72.68 & 70.18 & 85.20 & 63.27 & 23.75 & 81.21 & 66.05\blue{18.87} \\
\hline
\rowcolor{gray!10}
\ourmethod & \textbf{62.40} & \textbf{59.73} & \textbf{68.43} & \textbf{57.90} & \textbf{62.12}\red{5.19} & \textbf{84.30} & \textbf{83.94} & \textbf{84.12}\red{0.42} & \textbf{87.12} & \textbf{83.94} & \textbf{92.84} & \textbf{89.62} & \textbf{62.53} & \textbf{94.97} & \textbf{85.17}\red{0.25} \\
\Xhline{1.5pt}
% Block 7: UniME-V2-2B
UniME-V2-2B & 9.50 & 15.78 & 15.51 & 19.03 & 14.96 & 54.24 & 77.98 & 66.11 & 61.19 & 65.48 & 76.76 & 59.65 & 45.25 & 77.50 & 64.31 \\
\hline
\rowcolor{gray!10}
\texttt{s2m-add} & 30.82 & 38.23 & 23.31 & 20.59 & 28.24\red{13.28} & 58.58 & 50.31 & 54.45\blue{11.66} & 45.94 & 55.05 & 67.28 & 55.13 & 12.43 & 72.62 & 51.41\blue{12.90} \\
\texttt{s2m-mul} & 29.67 & 34.92 & 21.62 & 19.02 & 26.31\red{11.35} & 56.92 & 48.58 & 52.75\blue{13.36} & 45.21 & 55.60 & 66.84 & 53.32 & 12.49 & 71.71 & 50.86\blue{13.45} \\
\rowcolor{gray!10}
\texttt{cl-t-c} & 24.01 & 15.49 & 14.30 & 12.37 & 16.54\red{1.58} & 46.27 & 38.93 & 42.60\blue{23.51} & 18.00 & 24.42 & 29.87 & 16.38 & 4.20 & 33.42 & 21.05\blue{43.26} \\
\texttt{cl-t-m} & 21.67 & 10.06 & 13.15 & 9.43 & 13.58\blue{1.38} & 42.79 & 35.18 & 38.99\blue{27.12} & 17.18 & 27.33 & 42.30 & 14.99 & 5.25 & 44.29 & 25.22\blue{39.09} \\
\rowcolor{gray!10}
\texttt{cl-s-c} & 23.56 & 15.34 & 12.00 & 12.62 & 15.88\red{0.92} & 48.96 & 39.41 & 44.19\blue{21.92} & 19.80 & 24.93 & 32.56 & 21.55 & 4.33 & 38.90 & 23.68\blue{40.63} \\
\texttt{cl-s-m} & 22.75 & 13.02 & 12.48 & 10.46 & 14.68\blue{0.28} & 43.51 & 37.60 & 40.56\blue{25.55} & 17.20 & 29.53 & 42.92 & 16.67 & 5.00 & 47.36 & 26.45\blue{37.86} \\
\rowcolor{gray!10}
\texttt{c-sem} & 35.10 & 27.35 & 21.50 & 28.81 & 28.19\red{13.23} & 64.50 & 60.58 & 62.54\blue{3.57} & 32.65 & 51.23 & 57.19 & 40.32 & 14.91 & 64.66 & 43.49\blue{20.82} \\
\texttt{multi-img} & 31.08 & 38.51 & 16.09 & 18.02 & 25.93\red{10.97} & 56.06 & 43.06 & 49.56\blue{16.55} & 46.47 & 46.80 & 65.91 & 42.99 & 11.82 & 63.81 & 46.30\blue{18.01} \\
\hline
\rowcolor{gray!10}
\ourmethod & \textbf{44.52} & \textbf{45.86} & \textbf{48.15} & \textbf{51.48} & \textbf{47.50}\red{32.54} & \textbf{80.50} & \textbf{79.46} & \textbf{79.98}\red{13.87} & \textbf{58.51} & \textbf{63.17} & \textbf{74.52} & \textbf{66.58} & \textbf{35.57} & \textbf{80.21} & \textbf{63.09}\blue{1.22} \\
\Xhline{1.5pt}
% Block 8: UniME-V2-7B
UniME-V2-7B & 26.77 & 23.69 & 24.68 & 31.17 & 26.58 & 78.25 & 82.25 & 80.25 & 60.60 & 79.43 & 80.61 & 64.94 & 45.35 & 82.17 & 68.85 \\
\hline
\rowcolor{gray!10}
\texttt{s2m-add} & 41.25 & \textbf{54.38} & 36.91 & 24.88 & 39.36\red{12.78} & 71.49 & 64.32 & 67.91\blue{12.34} & 51.03 & 64.77 & 75.82 & 67.21 & 15.38 & 81.12 & 59.22\blue{9.63} \\
\texttt{s2m-mul} & 41.49 & 53.84 & 35.49 & 23.31 & 38.53\red{11.95} & 71.69 & 64.24 & 67.97\blue{12.28} & 50.89 & 65.45 & 76.45 & 67.54 & 15.32 & 81.68 & 59.56\blue{9.29} \\
\rowcolor{gray!10}
\texttt{cl-t-c} & 22.76 & 16.17 & 8.68 & 14.14 & 15.44\blue{11.14} & 49.38 & 45.31 & 47.35\blue{32.90} & 19.44 & 24.63 & 35.22 & 17.75 & 4.44 & 36.13 & 22.94\blue{45.91} \\
\texttt{cl-t-m} & 27.93 & 14.34 & 19.03 & 16.99 & 19.57\blue{7.01} & 55.94 & 52.84 & 54.39\blue{25.86} & 27.05 & 38.47 & 62.09 & 20.81 & 7.64 & 55.56 & 35.27\blue{33.58} \\
\rowcolor{gray!10}
\texttt{cl-s-c} & 24.75 & 19.62 & 12.41 & 11.69 & 17.12\blue{9.46} & 52.10 & 48.38 & 50.24\blue{30.01} & 20.58 & 21.04 & 34.39 & 23.49 & 5.53 & 35.42 & 23.41\blue{45.44} \\
\texttt{cl-s-m} & 30.78 & 20.40 & 18.14 & 15.62 & 21.24\blue{5.34} & 56.69 & 53.63 & 55.16\blue{25.09} & 28.73 & 37.84 & 62.65 & 23.01 & 6.78 & 57.14 & 36.03\blue{32.82} \\
\rowcolor{gray!10}
\texttt{c-sem} & 37.96 & 29.51 & 27.33 & 21.95 & 29.19\red{2.61} & 71.01 & 68.82 & 69.92\blue{10.33} & 51.40 & 65.80 & 72.53 & 54.42 & 21.98 & 76.81 & 57.16\blue{11.69} \\
\texttt{multi-img} & 35.99 & 45.82 & 24.74 & 14.71 & 30.32\red{3.74} & 65.89 & 53.21 & 59.55\blue{20.70} & 51.79 & 62.63 & 73.67 & 48.06 & 17.25 & 73.00 & 54.40\blue{14.45} \\
\hline
\rowcolor{gray!10}
\ourmethod & \textbf{54.95} & 50.07 & \textbf{54.92} & \textbf{50.14} & \textbf{52.52}\red{25.94} & \textbf{81.16} & \textbf{83.32} & \textbf{82.24}\red{1.99} & \textbf{61.90} & \textbf{77.80} & \textbf{78.41} & \textbf{71.89} & \textbf{44.43} & \textbf{84.68} & \textbf{69.85}\red{1.00} \\
\Xhline{1.5pt}
% Block 9: B3-2B
B3-2B & 38.41 & 31.80 & 45.23 & 45.10 & 40.14 & 78.56 & 74.87 & 76.72 & 51.75 & 66.86 & 70.43 & 45.73 & 36.69 & 77.81 & 58.21 \\
\hline
\rowcolor{gray!10}
\texttt{s2m-add} & 32.14 & \textbf{40.97} & 22.62 & 19.74 & 28.87\blue{11.27} & 67.40 & 56.61 & 62.01\blue{14.71} & 43.01 & 51.27 & 68.01 & 50.92 & 13.40 & 73.78 & 50.07\blue{8.14} \\
\texttt{s2m-mul} & 30.87 & 39.31 & 21.36 & 19.69 & 27.81\blue{12.33} & 66.73 & 56.12 & 61.43\blue{15.29} & 42.74 & 49.75 & \textbf{68.21} & 51.27 & 13.31 & 73.79 & 49.85\blue{8.36} \\
\rowcolor{gray!10}
\texttt{cl-t-c} & 11.90 & 13.53 & 9.61 & 5.70 & 10.19\blue{29.95} & 44.40 & 41.39 & 42.90\blue{33.82} & 14.13 & 37.31 & 41.30 & 17.72 & 3.32 & 44.50 & 26.38\blue{31.83} \\
\texttt{cl-t-m} & 15.16 & 14.20 & 13.12 & 6.61 & 12.27\blue{27.87} & 44.92 & 44.57 & 44.75\blue{31.97} & 14.12 & 32.08 & 42.61 & 19.91 & 1.06 & 43.03 & 25.47\blue{32.74} \\
\rowcolor{gray!10}
\texttt{cl-s-c} & 16.03 & 20.23 & 10.10 & 8.92 & 13.82\blue{26.32} & 48.54 & 47.73 & 48.14\blue{28.58} & 13.09 & 36.05 & 44.30 & 24.60 & 2.76 & 48.77 & 28.26\blue{29.95} \\
\texttt{cl-s-m} & 16.43 & 16.68 & 14.29 & 9.21 & 14.15\blue{25.99} & 48.08 & 47.08 & 47.58\blue{29.14} & 13.51 & 35.91 & 44.17 & 25.47 & 1.20 & 45.45 & 27.62\blue{30.59} \\
\rowcolor{gray!10}
\texttt{c-sem} & 25.26 & 25.23 & 18.89 & 18.29 & 21.92\blue{18.22} & 60.79 & 59.97 & 60.38\blue{16.34} & 32.30 & 54.90 & 61.44 & 39.77 & 10.80 & 61.17 & 43.40\blue{14.81} \\
\texttt{multi-img} & 21.83 & 23.38 & 7.29 & 11.37 & 15.97\blue{24.17} & 52.38 & 37.41 & 44.90\blue{31.82} & 38.52 & 41.26 & 53.67 & 24.33 & 7.97 & 37.56 & 33.89\blue{24.32} \\
\hline
\rowcolor{gray!10}
\ourmethod & \textbf{45.03} & 39.23 & \textbf{48.20} & \textbf{49.30} & \textbf{45.44}\red{5.30} & \textbf{79.98} & \textbf{80.79} & \textbf{80.39}\red{3.67} & \textbf{51.00} & \textbf{62.94} & 67.06 & \textbf{53.86} & \textbf{32.02} & \textbf{80.01} & \textbf{57.82}\blue{0.39} \\
\Xhline{1.5pt}
% Block 10: B3-7B
B3-7B & 47.29 & 44.81 & 50.84 & 48.05 & 47.75 & 82.07 & 82.26 & 82.17 & 65.83 & 76.77 & 84.54 & 68.55 & 52.86 & 85.75 & 72.38 \\
\hline
\rowcolor{gray!10}
\texttt{s2m-add} & 44.47 & 48.61 & 35.86 & 28.82 & 39.44\blue{8.31} & 75.93 & 65.24 & 70.59\blue{11.58} & 56.95 & 69.80 & 78.61 & 69.19 & 19.74 & 85.71 & 63.33\blue{9.05} \\
\texttt{s2m-mul} & 44.79 & \textbf{49.78} & 35.45 & 28.68 & 39.68\blue{8.07} & 76.17 & 64.83 & 70.50\blue{11.67} & 57.17 & 70.52 & 79.24 & 70.67 & 20.17 & 85.50 & 63.88\blue{8.50} \\
\rowcolor{gray!10}
\texttt{cl-t-c} & 18.50 & 22.50 & 20.46 & 15.85 & 19.33\blue{28.42} & 66.30 & 63.26 & 64.78\blue{17.39} & 18.07 & 43.88 & 52.94 & 23.57 & 5.85 & 58.35 & 33.78\blue{38.60} \\
\texttt{cl-t-m} & 22.98 & 22.71 & 22.21 & 14.16 & 20.52\blue{27.23} & 67.41 & 67.43 & 67.42\blue{14.75} & 23.92 & 43.10 & 62.54 & 28.35 & 7.43 & 61.50 & 37.81\blue{34.57} \\
\rowcolor{gray!10}
\texttt{cl-s-c} & 20.76 & 26.27 & 28.71 & 22.95 & 24.67\blue{23.08} & 68.09 & 65.59 & 66.84\blue{15.33} & 16.71 & 47.07 & 53.44 & 31.71 & 7.56 & 60.01 & 36.08\blue{36.30} \\
\texttt{cl-s-m} & 22.17 & 26.43 & 26.15 & 19.85 & 23.65\blue{24.10} & 67.85 & 64.05 & 65.95\blue{16.22} & 19.23 & 43.79 & 53.45 & 35.24 & 8.50 & 61.19 & 36.90\blue{35.48} \\
\rowcolor{gray!10}
\texttt{c-sem} & 32.31 & 30.23 & 33.27 & 18.99 & 28.70\blue{19.05} & 75.30 & 71.83 & 73.57\blue{8.60} & 46.32 & 67.80 & 75.28 & 54.93 & 24.73 & 80.26 & 58.22\blue{14.16} \\
\texttt{multi-img} & 26.48 & 33.14 & 11.25 & 11.91 & 20.70\blue{27.05} & 64.84 & 47.83 & 56.34\blue{25.83} & 51.50 & 46.37 & 63.64 & 34.80 & 11.94 & 44.09 & 42.06\blue{30.32} \\
\hline
\rowcolor{gray!10}
\ourmethod & \textbf{53.72} & 49.50 & \textbf{52.40} & \textbf{50.15} & \textbf{51.44}\red{3.69} & \textbf{83.15} & \textbf{83.60} & \textbf{83.38}\red{1.21} & \textbf{65.97} & \textbf{75.92} & \textbf{80.19} & \textbf{73.87} & \textbf{48.05} & \textbf{86.13} & \textbf{71.69}\blue{0.69} \\
\Xhline{1.5pt}
\end{longtable}
\end{center}

\twocolumn
\clearpage

\subsection{Variant Study}
\label{app:variant_study}

\subsubsection{Algorithm Workflow}

\textbf{Algorithm~\ref{alg:variant_indexing}} presents a unified offline indexing framework for the three variants. After shared layout parsing and dual-stream encoding (Stages 1--2), Stage 3 diverges based on the specified variant type: \texttt{single2multi} retains raw sub-image vectors; \texttt{type\_cluster} aggregates vectors by semantic content types via averaging; \texttt{global\_inclusion} appends the full-document global vector to the local set. 

\begin{algorithm}[!ht]
\caption{Integrated Offline Indexing for \ourmethod Variants}\label{alg:variant_indexing}
\SetKwComment{Comment}{\# }{}
\SetCommentSty{textrm}

\Input{
    A document image $d \in \mathbb{R}^{H \times W \times 3}$; \\
    A document parser model $\Psi_{\text{parse}}$; \\
    A single-vector encoder $\Phi_{\text{enc}}: \mathbb{R}^{H' \times W' \times 3} \to \mathbb{R}^D$; \\
    Mode $M \in \{\texttt{s2m, s2m-t-c, s2m-g-i}\}$
}
\Output{
    A multi-vector representation $\mathbf{D}_{\text{variant}}$
}
\BlankLine

\tcc{\textcolor{blue}{Stage 1: Layout-Informed Document Parsing}}
$[\{b_j, c_j\}]_{j=1}^k \leftarrow \Psi_{\text{parse}}(d)$ \tcp*{Get $k$ bboxes and content types}
$\mathcal{S}_d \leftarrow \emptyset$\;
\For{$j \leftarrow 1$ \KwTo $k$}{
    $s_j \leftarrow \text{Crop}(d, b_j)$ \tcp*{Extract sub-image for each layout component}
    $\mathcal{S}_d \leftarrow \mathcal{S}_d \cup \{s_j\}$\;
}
\BlankLine

\tcc{\textcolor{blue}{Stage 2: Regional Encoding}}
$\mathbf{D}_{\text{local}} \leftarrow \emptyset$\;
\For{each sub-image $s_j \in \mathcal{S}_d$}{
    $\mathbf{v}_{\text{local}}^{(j)} \leftarrow \Phi_{\text{enc}}(s_j)$ \tcp*{Independent regional encoding}
    $\mathbf{D}_{\text{local}} \leftarrow \mathbf{D}_{\text{local}} \cup \{\mathbf{v}_{\text{local}}^{(j)}\}$\;
}

\BlankLine

\tcc{\textcolor{blue}{Stage 3: Variant-specific Representation Construction}}
\uIf{$M = \texttt{s2m}$}{
    $\mathbf{D}_{\text{variant}} \leftarrow \mathbf{D}_{\text{local}}$ \tcp*{Standard layout-decomposed set}
}
\ElseIf{$M = \texttt{s2m-t-c}$}{
    $\mathbf{D}_{\text{variant}} \leftarrow \emptyset$\;
    $\mathcal{T} \leftarrow \text{Unique}(\{c_1, \dots, c_k\})$ \tcp*{Identify unique content types}
    \For{each type $t \in \mathcal{T}$}{
        $\mathbf{v}_{\text{avg}}^{(t)} \leftarrow \text{Mean}(\{\mathbf{v}_{\text{local}}^{(j)} \mid c_j = t\})$ \tcp*{Cluster and average by type}
        $\mathbf{D}_{\text{variant}} \leftarrow \mathbf{D}_{\text{variant}} \cup \{\mathbf{v}_{\text{avg}}^{(t)}\}$\;
    }
}
\ElseIf{$M = \texttt{s2m-g-i}$}{
    $\mathbf{v}_{\text{global}} \leftarrow \Phi_{\text{enc}}(d)$ \tcp*{Encode original page for context}
    $\mathbf{D}_{\text{variant}} \leftarrow \mathbf{D}_{\text{local}} \cup \{\mathbf{v}_{\text{global}}\}$ \tcp*{Append global vector to the set}
}

\BlankLine
\KwRet $\mathbf{D}_{\text{variant}}$\;
\end{algorithm}

\subsubsection{More Analysis}

Due to the limited space of main text, we leave the radar plots of performance comparison between \ourmethod and its variants in \autoref{fig:variant_result_visualization}.

\textbf{The introduction of global page-level context is indispensable for resolving semantic ambiguities within isolated layout components.} Quantitative results in \autoref{tab:variant_results_mmlongbench_vidorev1} demonstrate that adding global context---even via simple inclusion (\texttt{s2m-g-i})---dramatically elevates performance over the local-only \texttt{single2multi} baseline, lifting the VLM2Vec-V1-2B score on ViDoRe-V1 from 34.39 to 49.93. \autoref{fig:variant_bar} highlights that this gap is most pronounced in benchmarks requiring holistic understanding, where local sub-images like tables or charts often lack the necessary contextual headers found elsewhere on the page. We hypothesize that the global vector acts as a ``semantic anchor'' that provides the overarching topic of the document, which is essential for the late-interaction mechanism to accurately align specific query tokens with relevant sub-regions.

\textbf{Maintaining the individual spatial and semantic integrity of layout components is superior to heuristic type-level clustering.} As evidenced by the performance trends in \autoref{tab:variant_results_mmlongbench_vidorev1} and the comparative bars in \autoref{fig:variant_bar}, the \texttt{s2m-t-c} variant typically results in a performance regression compared to the standard \texttt{single2multi}, such as the 1.56-point drop for VLM2Vec-V1-2B on ViDoRe-V1 benchmark. This trend is echoed across the radar charts in \autoref{fig:variant_result_visualization}, where the type-clustered variants consistently exhibit the narrowest performance profiles. This indicates that spatial locality is a vital semantic carrier in visual documents; by collapsing multiple distinct components into a single type-level average, the model loses the fine-grained resolution required for the \texttt{MaxSim} operator to distinguish between specific relevant and irrelevant regions of the same type.

\begin{figure*}[!ht]
  \centering
  \includegraphics[width=1\linewidth]{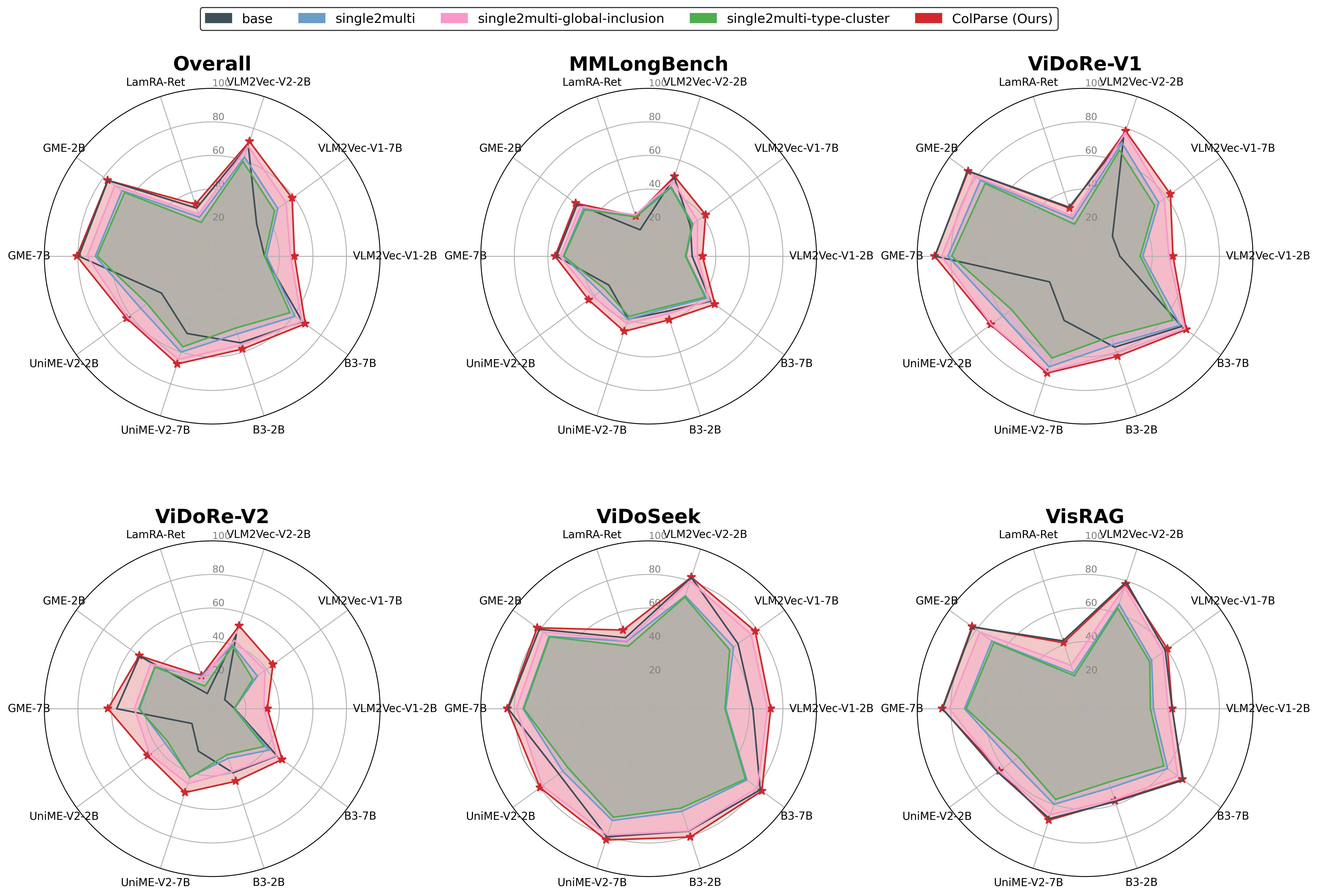}
  % \vspace{-2em}
  \caption{The performance comparison (evaluated by nDCG@5) between \ourmethod and its variants on five VDR benchmarks across ten mainstream single-vector multimodal retrieval models. Refer to \autoref{tab:variant_results_mmlongbench_vidorev1} and \autoref{tab:variant_results_vidorev2_vidoseek_visrag} for detailed result records due to the space limit.}
   \label{fig:variant_result_visualization}
   \vspace{-1em}
\end{figure*}

\onecolumn

% We recommend wrapping this in a \onecolumn environment for two-column papers.
\begin{center}
\small % Use smaller font for wide tables
\setlength\tabcolsep{2.5pt} % Adjust column separation
\renewcommand\arraystretch{1.2} % Adjust row height
\begin{longtable}{p{2.6cm}|cc|c|*{10}{c}|c}
\caption{Ablation study on MMLongBench and ViDoRe-V1 benchmarks. For each model block, we bold the best-performing method in each column (except for the base result). The average scores are shown with relative gains ($\color{RedOrange}\uparrow\color{black}$/$\color{BlueGreen}\downarrow\color{black}$) compared to the base model.}
\label{tab:variant_results_mmlongbench_vidorev1} \\
\Xhline{1.5pt}
\rowcolor{gray!20}
\multirow{2}{*}{\textbf{Method}} & \multicolumn{3}{c|}{\textbf{MMLongBench}} & \multicolumn{11}{c}{\textbf{ViDoRe-V1}} \\
\cmidrule(lr){2-4} \cmidrule(lr){5-15}
\rowcolor{gray!20}
& \textbf{Doc} & \textbf{Page} & \textbf{Avg.} & \textbf{Arxiv} & \textbf{DocV} & \textbf{InfoV} & \textbf{Shift} & \textbf{TabF} & \textbf{TatD} & \textbf{S-AI} & \textbf{S-En} & \textbf{S-HC} & \textbf{S-Gov} & \textbf{Avg.} \\
\Xhline{1.5pt}
\endfirsthead
\Xhline{1.5pt}
\rowcolor{gray!20}
\multicolumn{15}{c}{\tablename\ \thetable\ -- \textit{Continued from previous page}} \\
\Xhline{1.5pt}
\rowcolor{gray!20}
\multirow{2}{*}{\textbf{Method}} & \multicolumn{3}{c|}{\textbf{MMLongBench}} & \multicolumn{11}{c}{\textbf{ViDoRe-V1}} \\
\cmidrule(lr){2-4} \cmidrule(lr){5-15}
\rowcolor{gray!20}
& \textbf{Doc} & \textbf{Page} & \textbf{Avg.} & \textbf{Arxiv} & \textbf{DocV} & \textbf{InfoV} & \textbf{Shift} & \textbf{TabF} & \textbf{TatD} & \textbf{S-AI} & \textbf{S-En} & \textbf{S-HC} & \textbf{S-Gov} & \textbf{Avg.} \\
\Xhline{1.5pt}
\endhead
\Xhline{1.5pt}
\multicolumn{15}{r}{\textit{Continued on next page}} \\
\endfoot
\Xhline{1.5pt}
\endlastfoot
%
% Block 1: VLM2Vec-V1-2B
VLM2Vec-V1-2B & 25.62 & 26.23 & 25.93 & 17.80 & 13.98 & 39.41 & 9.18 & 36.32 & 10.56 & 16.39 & 15.96 & 23.56 & 24.11 & 20.73 \\
\hline
\rowcolor{gray!10}
\texttt{single2multi} & 25.19 & 19.59 & 22.39\blue{3.54} & 38.78 & 16.38 & 56.94 & 12.48 & 41.15 & 9.98 & 38.16 & 46.42 & 42.66 & 40.94 & 34.39\red{13.66} \\
\texttt{s2m-t-c} & 25.37 & 18.43 & 21.90\blue{4.03} & 34.95 & 15.04 & 55.09 & 15.31 & 40.07 & 9.78 & 32.83 & 43.07 & 44.84 & 37.31 & 32.83\red{12.10} \\
\rowcolor{gray!10}
\texttt{s2m-g-i} & 31.57 & 26.55 & 29.06\red{3.13} & \textbf{48.84} & 23.86 & 67.91 & 39.85 & \textbf{62.15} & 14.07 & 55.74 & \textbf{64.41} & 61.67 & 60.83 & 49.93\red{29.20} \\
\hline
\rowcolor{gray!10}
\ourmethod & \textbf{34.31} & \textbf{29.83} & \textbf{32.07}\red{6.14} & 47.66 & \textbf{28.12} & \textbf{69.23} & \textbf{47.11} & 57.05 & \textbf{20.43} & \textbf{62.24} & 63.77 & \textbf{65.51} & \textbf{62.54} & \textbf{52.37}\red{31.64} \\
\Xhline{1.5pt}
% Block 2: VLM2Vec-V1-7B
VLM2Vec-V1-7B & 23.85 & 37.63 & 30.74 & 28.07 & 17.93 & 44.47 & 2.06 & 16.78 & 5.86 & 17.93 & 25.04 & 28.90 & 14.59 & 20.16 \\
\hline
\rowcolor{gray!10}
\texttt{single2multi} & 35.75 & 29.40 & 32.58\red{1.84} & 53.54 & 26.07 & 61.73 & 41.43 & 67.64 & 22.37 & 72.54 & 59.56 & 73.52 & 64.03 & 54.24\red{34.08} \\
\texttt{s2m-t-c} & 35.62 & 28.64 & 32.13\red{1.39} & 50.56 & 24.08 & 59.42 & 35.90 & 64.84 & 18.29 & 65.39 & 58.49 & 72.13 & 63.75 & 51.29\red{31.13} \\
\rowcolor{gray!10}
\texttt{s2m-g-i} & 37.95 & 34.08 & 36.02\red{5.28} & \textbf{62.29} & 30.49 & 67.54 & 41.74 & \textbf{78.87} & 22.99 & \textbf{74.85} & 62.78 & 77.07 & 64.14 & 58.28\red{38.12} \\
\hline
\rowcolor{gray!10}
\ourmethod & \textbf{43.34} & \textbf{40.58} & \textbf{41.96}\red{11.22} & 60.47 & \textbf{34.42} & \textbf{70.39} & \textbf{53.67} & 77.12 & \textbf{31.33} & 74.81 & \textbf{69.64} & \textbf{80.79} & \textbf{75.89} & \textbf{62.85}\red{42.69} \\
\Xhline{1.5pt}
% Block 3: VLM2Vec-V2-2B
VLM2Vec-V2-2B & 48.55 & 50.34 & 49.45 & 78.98 & 38.51 & 82.21 & 64.57 & 87.64 & 44.68 & 85.06 & 82.99 & 89.89 & 87.08 & 74.16 \\
\hline
\rowcolor{gray!10}
\texttt{single2multi} & 45.89 & 41.04 & 43.47\blue{5.98} & 69.32 & 38.09 & 77.88 & 56.25 & 72.60 & 52.69 & 88.14 & 83.80 & 86.99 & 84.09 & 70.99\blue{3.17} \\
\texttt{s2m-t-c} & 44.99 & 40.97 & 42.98\blue{6.47} & 67.44 & 34.25 & 74.79 & 48.36 & 71.32 & 43.73 & 82.22 & 81.18 & 84.78 & 76.42 & 66.45\blue{7.71} \\
\rowcolor{gray!10}
\texttt{s2m-g-i} & 48.03 & 46.30 & 47.17\blue{2.28} & 79.83 & \textbf{46.57} & 82.19 & 66.78 & \textbf{87.66} & \textbf{52.85} & \textbf{93.91} & \textbf{86.02} & 91.14 & 87.46 & 77.44\red{3.28} \\
\hline
\rowcolor{gray!10}
\ourmethod & \textbf{49.49} & \textbf{50.53} & \textbf{50.01}\red{0.56} & \textbf{80.17} & 46.33 & \textbf{83.53} & \textbf{72.76} & 86.74 & 52.40 & 91.36 & 85.83 & \textbf{95.47} & \textbf{89.52} & \textbf{78.41}\red{4.25} \\
\Xhline{1.5pt}
% Block 4: LamRA-Ret
LamRA-Ret & 19.78 & 13.24 & 16.51 & 29.31 & 19.56 & 63.00 & 15.83 & 51.44 & 7.70 & 21.10 & 29.81 & 37.18 & 31.95 & 30.69 \\
\hline
\rowcolor{gray!10}
\texttt{single2multi} & 33.04 & 18.56 & 25.80\red{9.29} & 6.06 & 13.60 & 37.83 & 18.90 & 25.46 & 19.02 & 23.99 & 28.27 & 33.31 & \textbf{28.03} & 23.45\blue{7.24} \\
\texttt{s2m-t-c} & 31.74 & 17.49 & 24.62\red{8.11} & 6.13 & 10.90 & 35.44 & 16.81 & 28.56 & 15.77 & 15.45 & 27.59 & 24.45 & 19.65 & 20.08\blue{10.61} \\
\rowcolor{gray!10}
\texttt{s2m-g-i} & \textbf{33.16} & 18.48 & \textbf{25.82}\red{9.31} & 6.06 & 15.79 & 37.71 & 18.96 & 38.78 & \textbf{19.19} & \textbf{26.00} & 28.27 & 31.68 & 29.03 & 25.15\blue{5.54} \\
\hline
\rowcolor{gray!10}
\ourmethod & 30.74 & \textbf{19.50} & 25.12\red{8.61} & \textbf{17.27} & \textbf{20.61} & \textbf{58.35} & \textbf{21.39} & \textbf{39.69} & 13.34 & 25.57 & \textbf{35.27} & \textbf{43.21} & 26.29 & \textbf{30.10}\blue{0.59} \\
\Xhline{1.5pt}
% Block 5: GME-2B
GME-2B & 52.07 & 53.14 & 52.61 & 82.59 & 56.46 & 88.97 & 89.72 & 93.20 & 70.33 & 98.49 & 92.15 & 98.15 & 95.65 & 86.57 \\
\hline
\rowcolor{gray!10}
\texttt{single2multi} & 50.56 & 45.68 & 48.12\blue{4.49} & 72.46 & 39.99 & 79.82 & 70.54 & 80.91 & 68.07 & 88.91 & 91.11 & 88.85 & 89.86 & 77.05\blue{9.52} \\
\texttt{s2m-t-c} & 49.51 & 44.73 & 47.12\blue{5.49} & 70.90 & 37.34 & 77.92 & 67.59 & 79.85 & 59.12 & 86.13 & 86.87 & 87.23 & 83.35 & 73.63\blue{12.94} \\
\rowcolor{gray!10}
\texttt{s2m-g-i} & 51.81 & 48.49 & 50.15\blue{2.46} & 80.89 & 46.62 & 83.93 & 75.78 & 92.12 & 68.33 & 93.48 & \textbf{92.60} & 93.36 & 90.49 & 81.76\blue{4.81} \\
\hline
\rowcolor{gray!10}
\ourmethod & \textbf{53.06} & \textbf{54.24} & \textbf{53.65}\red{1.04} & \textbf{82.39} & \textbf{54.11} & \textbf{88.93} & \textbf{88.51} & \textbf{92.33} & \textbf{70.65} & \textbf{97.75} & 92.30 & \textbf{97.91} & \textbf{96.10} & \textbf{86.10}\blue{0.47} \\
\Xhline{1.5pt}
% Block 6: GME-7B
GME-7B & 54.01 & 55.80 & 54.91 & 87.59 & 56.05 & 91.96 & 94.25 & 93.72 & 76.26 & 99.63 & 95.45 & 99.63 & 99.06 & 89.36 \\
\hline
\rowcolor{gray!10}
\texttt{single2multi} & 53.55 & 48.17 & 50.86\blue{4.05} & 75.57 & 45.26 & 83.21 & 77.90 & 86.05 & 73.97 & 94.72 & 92.04 & 95.19 & 93.18 & 81.71\blue{7.65} \\
\texttt{s2m-t-c} & 52.97 & 48.05 & 50.51\blue{4.40} & 75.21 & 44.11 & 82.69 & 78.44 & 86.17 & 66.14 & 91.53 & 89.86 & 91.38 & 88.60 & 79.41\blue{9.95} \\
\rowcolor{gray!10}
\texttt{s2m-g-i} & 54.32 & 51.14 & 52.73\blue{2.18} & 84.29 & 52.98 & 87.36 & 82.08 & 94.79 & 74.48 & 96.65 & 91.10 & 96.92 & 93.70 & 85.44\blue{3.92} \\
\hline
\rowcolor{gray!10}
\ourmethod & \textbf{54.96} & \textbf{56.51} & \textbf{55.74}\red{0.83} & \textbf{87.35} & \textbf{57.91} & \textbf{90.76} & \textbf{95.35} & \textbf{95.44} & \textbf{75.92} & \textbf{99.63} & \textbf{94.67} & \textbf{99.63} & \textbf{98.89} & \textbf{89.56}\red{0.20} \\
\Xhline{1.5pt}
% Block 7: UniME-V2-2B
UniME-V2-2B & 18.52 & 40.10 & 29.31 & 36.52 & 12.43 & 42.41 & 14.09 & 51.11 & 7.39 & 20.23 & 32.96 & 24.21 & 19.25 & 26.06 \\
\hline
\rowcolor{gray!10}
\texttt{single2multi} & 38.47 & 33.34 & 35.91\red{6.60} & 52.64 & 27.01 & 68.58 & 49.89 & 61.20 & 38.92 & 77.36 & 73.65 & 77.92 & 79.30 & 60.65\red{34.59} \\
\texttt{s2m-t-c} & 35.13 & 30.33 & 32.73\red{3.42} & 48.29 & 21.24 & 62.26 & 43.83 & 58.58 & 27.37 & 70.42 & 66.86 & 74.85 & 67.57 & 54.13\red{28.07} \\
\rowcolor{gray!10}
\texttt{s2m-g-i} & 41.90 & 39.97 & 40.94\red{11.63} & \textbf{64.23} & 33.32 & \textbf{74.84} & 67.10 & \textbf{78.40} & \textbf{39.00} & 82.16 & \textbf{79.38} & 88.43 & 83.62 & 69.05\red{42.99} \\
\hline
\rowcolor{gray!10}
\ourmethod & \textbf{44.22} & \textbf{44.19} & \textbf{44.21}\red{14.90} & 62.39 & \textbf{37.69} & 73.33 & \textbf{71.19} & 77.45 & 38.83 & \textbf{82.50} & 75.80 & \textbf{89.35} & \textbf{85.84} & \textbf{69.44}\red{43.38} \\
\Xhline{1.5pt}
% Block 8: UniME-V2-7B
UniME-V2-7B & 33.19 & 45.72 & 39.46 & 63.23 & 24.91 & 65.25 & 11.16 & 41.54 & 14.18 & 41.89 & 40.56 & 57.44 & 42.78 & 40.29 \\
\hline
\rowcolor{gray!10}
\texttt{single2multi} & 39.83 & 39.39 & 39.61\red{0.15} & 57.08 & 32.54 & 71.38 & 64.54 & 71.82 & \textbf{49.00} & 84.00 & 84.26 & 91.18 & 87.49 & 69.33\red{29.04} \\
\texttt{s2m-t-c} & 38.66 & 36.94 & 37.80\blue{1.66} & 54.98 & 25.88 & 65.87 & 56.63 & 70.76 & 41.65 & 80.00 & 77.76 & 86.67 & 77.88 & 63.81\red{23.52} \\
\rowcolor{gray!10}
\texttt{s2m-g-i} & 41.35 & 43.08 & 42.22\red{2.76} & 64.50 & 34.55 & 76.81 & 66.35 & \textbf{81.95} & 48.84 & 85.63 & \textbf{85.38} & 91.83 & 87.90 & 72.37\red{32.08} \\
\hline
\rowcolor{gray!10}
\ourmethod & \textbf{45.90} & \textbf{48.26} & \textbf{47.08}\red{7.62} & \textbf{64.78} & \textbf{37.43} & \textbf{78.51} & \textbf{73.56} & 81.47 & 43.69 & \textbf{89.32} & 82.68 & \textbf{92.74} & \textbf{88.13} & \textbf{73.23}\red{32.94} \\
\Xhline{1.5pt}
% Block 9: B3-2B
B3-2B & 37.10 & 32.07 & 34.59 & 57.00 & 29.38 & 68.09 & 48.31 & 71.55 & 18.09 & 74.13 & 64.64 & 75.44 & 63.13 & 56.98 \\
\hline
\rowcolor{gray!10}
\texttt{single2multi} & 36.47 & 29.20 & 32.84\blue{1.75} & 46.63 & 20.13 & 60.31 & 51.58 & 56.45 & \textbf{36.55} & 70.40 & 69.80 & 73.83 & 66.85 & 55.25\blue{1.73} \\
\texttt{s2m-t-c} & 35.15 & 28.33 & 31.74\blue{2.85} & 43.88 & 18.92 & 54.65 & 41.90 & 54.30 & 26.86 & 67.36 & 61.92 & 72.05 & 60.66 & 50.25\blue{6.73} \\
\rowcolor{gray!10}
\texttt{s2m-g-i} & 38.87 & 33.51 & 36.19\red{1.60} & 56.06 & 22.52 & 64.41 & 59.98 & \textbf{73.09} & 35.98 & 78.34 & 70.48 & 80.13 & 67.17 & 60.82\red{3.84} \\
\hline
\rowcolor{gray!10}
\ourmethod & \textbf{42.06} & \textbf{37.60} & \textbf{39.83}\red{5.24} & \textbf{56.47} & \textbf{30.91} & \textbf{66.69} & \textbf{67.42} & 69.33 & 29.42 & \textbf{79.88} & \textbf{72.67} & \textbf{83.24} & \textbf{71.41} & \textbf{62.74}\red{5.76} \\
\Xhline{1.5pt}
% Block 10: B3-7B
B3-7B & 46.09 & 45.10 & 45.60 & 68.95 & 43.38 & 79.86 & 66.56 & 84.12 & 37.06 & 81.01 & 81.25 & 88.57 & 81.30 & 71.21 \\
\hline
\rowcolor{gray!10}
\texttt{single2multi} & 44.43 & 40.79 & 42.61\blue{2.99} & 58.94 & 31.91 & 71.96 & 68.80 & 73.69 & 53.07 & 87.14 & 81.75 & 88.38 & 84.76 & 70.04\blue{1.17} \\
\texttt{s2m-t-c} & 43.32 & 39.84 & 41.58\blue{4.02} & 56.41 & 29.10 & 68.99 & 53.11 & 70.75 & 44.35 & 83.69 & 77.98 & 84.69 & 78.56 & 64.76\blue{6.45} \\
\rowcolor{gray!10}
\texttt{s2m-g-i} & 45.62 & 43.88 & 44.75\blue{0.85} & 67.23 & 35.81 & 76.70 & 70.48 & \textbf{84.62} & \textbf{53.26} & \textbf{87.14} & \textbf{83.19} & 91.86 & 84.44 & 73.47\red{2.26} \\
\hline
\rowcolor{gray!10}
\ourmethod & \textbf{49.11} & \textbf{48.39} & \textbf{48.75}\red{3.15} & \textbf{67.68} & \textbf{42.17} & \textbf{79.02} & \textbf{78.06} & 81.64 & 47.60 & 85.17 & 82.04 & \textbf{92.00} & \textbf{88.73} & \textbf{74.41}\red{3.20} \\
\Xhline{1.5pt}
\end{longtable}
\end{center}

\twocolumn
\clearpage
\onecolumn

% We recommend wrapping this in a \onecolumn environment for two-column papers.
\begin{center}
\tiny % Use smaller font for wide tables
\setlength\tabcolsep{4.0pt} % Adjust column separation
\renewcommand\arraystretch{1.2} % Adjust row height
\begin{longtable}{p{2.4cm}|cccc|c|cc|c|*{6}{c}|c}
\caption{Ablation study on ViDoRe-V2, ViDoSeek, and VisRAG benchmarks. For each model block, we bold the best-performing method in each column (except for the base result). The average scores are shown with relative gains ($\color{RedOrange}\uparrow\color{black}$/$\color{BlueGreen}\downarrow\color{black}$) compared to the base model.}
\label{tab:variant_results_vidorev2_vidoseek_visrag} \\
\Xhline{1.5pt}
\rowcolor{gray!20}
\multirow{2}{*}{\textbf{Method}} & \multicolumn{5}{c|}{\textbf{ViDoRe-V2}} & \multicolumn{3}{c|}{\textbf{ViDoSeek}} & \multicolumn{7}{c}{\textbf{VisRAG}} \\
\cmidrule(lr){2-6} \cmidrule(lr){7-9} \cmidrule(lr){10-16}
\rowcolor{gray!20}
& \textbf{Bio-L} & \textbf{Eco-R} & \textbf{ESG-H} & \textbf{ESG-M} & \textbf{Avg.} & \textbf{Doc} & \textbf{Page} & \textbf{Avg.} & \textbf{Arxiv} & \textbf{Chart} & \textbf{InfoV} & \textbf{MP-Doc} & \textbf{Plot} & \textbf{Slide} & \textbf{Avg.} \\
\Xhline{1.5pt}
\endfirsthead
\Xhline{1.5pt}
\rowcolor{gray!20}
\multicolumn{16}{c}{\tablename\ \thetable\ -- \textit{Continued from previous page}} \\
\Xhline{1.5pt}
\rowcolor{gray!20}
\multirow{2}{*}{\textbf{Method}} & \multicolumn{5}{c|}{\textbf{ViDoRe-V2}} & \multicolumn{3}{c|}{\textbf{ViDoSeek}} & \multicolumn{7}{c}{\textbf{VisRAG}} \\
\cmidrule(lr){2-6} \cmidrule(lr){7-9} \cmidrule(lr){10-16}
\rowcolor{gray!20}
& \textbf{Bio-L} & \textbf{Eco-R} & \textbf{ESG-H} & \textbf{ESG-M} & \textbf{Avg.} & \textbf{Doc} & \textbf{Page} & \textbf{Avg.} & \textbf{Arxiv} & \textbf{Chart} & \textbf{InfoV} & \textbf{MP-Doc} & \textbf{Plot} & \textbf{Slide} & \textbf{Avg.} \\
\Xhline{1.5pt}
\endhead
\Xhline{1.5pt}
\multicolumn{16}{r}{\textit{Continued on next page}} \\
\endfoot
\Xhline{1.5pt}
\endlastfoot
%
% Block 1: VLM2Vec-V1-2B
VLM2Vec-V1-2B & 6.88 & 14.15 & 12.25 & 20.54 & 13.46 & 56.40 & 67.73 & 62.07 & 41.68 & 58.21 & 70.79 & 42.74 & 23.83 & 74.07 & 51.89 \\
\hline
\rowcolor{gray!10}
\texttt{single2multi} & 17.15 & 18.19 & 10.24 & 7.44 & 13.26\blue{0.20} & 50.51 & 41.52 & 46.02\blue{16.05} & 30.17 & 42.91 & 61.28 & 37.60 & 6.80 & 65.76 & 40.75\blue{11.14} \\
\texttt{s2m-t-c} & 13.46 & 18.31 & 11.58 & 8.85 & 13.05\blue{0.41} & 50.22 & 40.75 & 45.49\blue{16.58} & 27.87 & 41.64 & 59.37 & 34.82 & 6.94 & 62.76 & 38.90\blue{12.99} \\
\rowcolor{gray!10}
\texttt{s2m-g-i} & 28.49 & 26.29 & 29.18 & \textbf{39.15} & 30.78\red{17.32} & 73.38 & 68.14 & 70.76\red{8.69} & \textbf{40.39} & 52.54 & \textbf{70.43} & 47.82 & 8.96 & 76.40 & 49.42\blue{2.47} \\
\hline
\rowcolor{gray!10}
\ourmethod & \textbf{30.33} & \textbf{29.55} & \textbf{33.21} & 38.33 & \textbf{32.86}\red{19.40} & \textbf{75.23} & \textbf{70.19} & \textbf{72.71}\red{10.64} & 38.18 & \textbf{60.09} & 69.44 & \textbf{48.29} & \textbf{18.83} & \textbf{76.95} & \textbf{51.96}\red{0.07} \\
\Xhline{1.5pt}
% Block 2: VLM2Vec-V1-7B
VLM2Vec-V1-7B & 4.93 & 13.74 & 6.82 & 11.27 & 9.19 & 54.26 & 77.39 & 65.83 & 52.58 & 69.83 & 71.43 & 52.86 & 34.24 & 73.22 & 59.03 \\
\hline
\rowcolor{gray!10}
\texttt{single2multi} & 34.67 & 37.55 & 33.77 & 26.91 & 33.23\red{24.04} & 66.88 & 58.17 & 62.53\blue{3.30} & 45.06 & 50.30 & 63.65 & 50.23 & 10.83 & 74.00 & 49.01\blue{10.02} \\
\texttt{s2m-t-c} & 31.17 & 41.18 & 26.15 & 21.07 & 29.89\red{20.70} & 64.53 & 54.75 & 59.64\blue{6.19} & 43.98 & 52.75 & 61.47 & 45.75 & 10.49 & 71.33 & 47.63\blue{11.40} \\
\rowcolor{gray!10}
\texttt{s2m-g-i} & 41.08 & 37.99 & 40.11 & 36.54 & 38.93\red{29.74} & 75.46 & 75.28 & 75.37\red{9.54} & 53.97 & 64.14 & 65.37 & 55.41 & 27.98 & 77.07 & 57.32\blue{1.71} \\
\hline
\rowcolor{gray!10}
\ourmethod & \textbf{42.63} & \textbf{42.89} & \textbf{50.55} & \textbf{42.86} & \textbf{44.73}\red{35.54} & \textbf{78.34} & \textbf{78.61} & \textbf{78.48}\red{12.65} & \textbf{54.43} & \textbf{70.30} & \textbf{69.27} & \textbf{58.49} & \textbf{33.46} & \textbf{77.98} & \textbf{60.66}\red{1.63} \\
\Xhline{1.5pt}
% Block 3: VLM2Vec-V2-2B
VLM2Vec-V2-2B & 44.45 & 45.77 & 48.77 & 46.98 & 46.49 & 80.88 & 83.68 & 82.28 & 77.38 & 82.30 & 86.27 & 71.60 & 66.96 & 92.04 & 79.43 \\
\hline
\rowcolor{gray!10}
\texttt{single2multi} & 42.12 & 51.08 & 41.33 & 24.84 & 39.84\blue{6.65} & 73.81 & 67.71 & 70.76\blue{11.52} & 65.36 & 69.18 & 79.34 & 73.19 & 19.59 & 86.77 & 65.57\blue{13.86} \\
\texttt{s2m-t-c} & 41.81 & 51.83 & 37.07 & 25.78 & 39.12\blue{7.37} & 73.69 & 66.16 & 69.93\blue{12.35} & 64.18 & 66.49 & 77.61 & 64.32 & 19.58 & 85.70 & 62.98\blue{16.45} \\
\rowcolor{gray!10}
\texttt{s2m-g-i} & 44.34 & 51.99 & 40.57 & 34.53 & 42.86\blue{3.63} & 79.23 & 81.01 & 80.12\blue{2.16} & 77.03 & 74.83 & 82.98 & \textbf{78.17} & 54.13 & 90.93 & 76.35\blue{3.08} \\
\hline
\rowcolor{gray!10}
\ourmethod & \textbf{50.06} & \textbf{53.76} & \textbf{57.41} & \textbf{46.40} & \textbf{51.91}\red{5.42} & \textbf{80.94} & \textbf{83.87} & \textbf{82.41}\red{0.13} & \textbf{77.18} & \textbf{78.05} & \textbf{84.37} & 78.07 & \textbf{58.74} & \textbf{91.95} & \textbf{78.06}\blue{1.37} \\
\Xhline{1.5pt}
% Block 4: LamRA-Ret
LamRA-Ret & 10.75 & 9.65 & 6.32 & 11.18 & 9.48 & 60.17 & 28.81 & 44.49 & 11.17 & 63.50 & 59.78 & 33.57 & 29.42 & 57.59 & 42.51 \\
\hline
\rowcolor{gray!10}
\texttt{single2multi} & 11.11 & 26.32 & 20.54 & 23.76 & 20.43\red{10.95} & 53.77 & 30.82 & 42.30\blue{2.19} & 1.94 & 25.59 & 30.56 & 27.49 & 3.95 & 44.28 & 22.30\blue{20.21} \\
\texttt{s2m-t-c} & 9.36 & 15.57 & 14.49 & 17.11 & 14.13\red{4.65} & 49.02 & 29.38 & 39.20\blue{5.29} & 2.03 & 22.46 & 31.39 & 21.10 & 3.96 & 42.90 & 20.64\blue{21.87} \\
\rowcolor{gray!10}
\texttt{s2m-g-i} & 11.11 & \textbf{26.56} & 18.78 & 21.76 & 19.55\red{10.07} & 53.74 & 31.77 & 42.76\blue{1.73} & 1.91 & 28.55 & 31.15 & 32.12 & 25.09 & 44.22 & 27.17\blue{15.34} \\
\hline
\rowcolor{gray!10}
\ourmethod & \textbf{15.81} & 17.65 & \textbf{23.75} & \textbf{25.48} & \textbf{20.67}\red{11.19} & \textbf{58.55} & \textbf{39.99} & \textbf{49.27}\red{4.78} & \textbf{5.91} & \textbf{60.76} & \textbf{54.91} & \textbf{38.77} & \textbf{25.27} & \textbf{62.34} & \textbf{41.33}\blue{1.18} \\
\Xhline{1.5pt}
% Block 5: GME-2B
GME-2B & 54.25 & 50.65 & 59.44 & 49.15 & 53.37 & 81.44 & 79.62 & 80.53 & 81.37 & 81.70 & 91.31 & 85.03 & 63.81 & 93.60 & 82.80 \\
\hline
\rowcolor{gray!10}
\texttt{single2multi} & 47.25 & 43.18 & 42.80 & 35.67 & 42.23\blue{11.14} & 82.42 & 64.31 & 73.37\blue{7.16} & 69.46 & 75.14 & 82.71 & 77.65 & 19.43 & 87.64 & 68.67\blue{14.13} \\
\texttt{s2m-t-c} & 47.87 & 50.70 & 39.72 & 31.03 & 42.33\blue{11.04} & 82.06 & 63.82 & 72.94\blue{7.59} & 68.24 & 74.56 & 81.25 & 71.08 & 20.06 & 87.98 & 67.20\blue{15.60} \\
\rowcolor{gray!10}
\texttt{s2m-g-i} & 49.19 & 43.73 & 47.90 & 40.90 & 45.43\blue{7.94} & 83.46 & 72.22 & 77.84\blue{2.69} & 78.82 & 77.90 & 85.79 & 81.37 & 51.04 & 90.55 & 77.58\blue{5.22} \\
\hline
\rowcolor{gray!10}
\ourmethod & \textbf{55.14} & \textbf{52.08} & \textbf{56.24} & \textbf{51.36} & \textbf{53.71}\red{0.34} & \textbf{83.93} & \textbf{80.16} & \textbf{82.05}\red{1.52} & \textbf{81.65} & \textbf{83.87} & \textbf{91.06} & \textbf{85.89} & \textbf{63.17} & \textbf{93.76} & \textbf{83.23}\red{0.43} \\
\Xhline{1.5pt}
% Block 6: GME-7B
GME-7B & 53.66 & 54.34 & 65.38 & 54.32 & 56.93 & 83.21 & 84.18 & 83.70 & 87.20 & 82.32 & 92.92 & 88.89 & 63.36 & 94.81 & 84.92 \\
\hline
\rowcolor{gray!10}
\texttt{single2multi} & 44.39 & 42.94 & 50.47 & 35.10 & 43.23\blue{13.70} & 83.51 & 66.19 & 74.85\blue{8.85} & 75.24 & 77.08 & 84.57 & 82.41 & 20.95 & 90.67 & 71.82\blue{13.10} \\
\texttt{s2m-t-c} & 47.29 & 51.24 & 46.31 & 30.35 & 43.80\blue{13.13} & 83.03 & 65.61 & 74.32\blue{9.38} & 74.40 & 76.84 & 83.60 & 77.58 & 21.35 & 90.80 & 70.76\blue{14.16} \\
\rowcolor{gray!10}
\texttt{s2m-g-i} & 45.99 & 42.98 & 56.63 & 39.79 & 46.35\blue{10.58} & 84.04 & 73.61 & 78.83\blue{4.87} & 84.45 & 80.35 & 88.21 & 87.03 & 53.24 & 92.27 & 80.93\blue{3.99} \\
\hline
\rowcolor{gray!10}
\ourmethod & \textbf{62.40} & \textbf{59.73} & \textbf{68.43} & \textbf{57.90} & \textbf{62.12}\red{5.19} & \textbf{84.30} & \textbf{83.94} & \textbf{84.12}\red{0.42} & \textbf{87.12} & \textbf{83.94} & \textbf{92.84} & \textbf{89.62} & \textbf{62.53} & \textbf{94.97} & \textbf{85.17}\red{0.25} \\
\Xhline{1.5pt}
% Block 7: UniME-V2-2B
UniME-V2-2B & 9.50 & 15.78 & 15.51 & 19.03 & 14.96 & 54.24 & 77.98 & 66.11 & 61.19 & 65.48 & 76.76 & 59.65 & 45.25 & 77.50 & 64.31 \\
\hline
\rowcolor{gray!10}
\texttt{single2multi} & 37.44 & 47.15 & 27.88 & 26.68 & 34.79\red{19.83} & 67.36 & 59.21 & 63.29\blue{2.82} & 47.81 & 54.17 & 71.61 & 61.52 & 9.44 & 76.72 & 53.55\blue{10.76} \\
\texttt{s2m-t-c} & 35.32 & 42.69 & 28.26 & 25.47 & 32.94\red{17.98} & 63.55 & 55.59 & 59.57\blue{6.54} & 44.16 & 48.21 & 67.30 & 52.26 & 9.72 & 72.55 & 49.03\blue{15.28} \\
\rowcolor{gray!10}
\texttt{s2m-g-i} & \textbf{44.64} & \textbf{48.14} & 46.83 & 44.53 & 46.04\red{31.08} & 78.62 & 77.90 & 78.26\red{12.15} & 60.30 & 60.76 & \textbf{77.30} & \textbf{68.15} & 30.63 & \textbf{82.23} & \textbf{63.23}\blue{1.08} \\
\hline
\rowcolor{gray!10}
\ourmethod & 44.52 & 45.86 & \textbf{48.15} & \textbf{51.48} & \textbf{47.50}\red{32.54} & \textbf{80.50} & \textbf{79.46} & \textbf{79.98}\red{13.87} & \textbf{58.51} & \textbf{63.17} & 74.52 & 66.58 & \textbf{35.57} & 80.21 & 63.09\blue{1.22} \\
\Xhline{1.5pt}
% Block 8: UniME-V2-7B
UniME-V2-7B & 26.77 & 23.69 & 24.68 & 31.17 & 26.58 & 78.25 & 82.25 & 80.25 & 60.60 & 79.43 & 80.61 & 64.94 & 45.35 & 82.17 & 68.85 \\
\hline
\rowcolor{gray!10}
\texttt{single2multi} & 45.06 & 53.38 & 44.29 & 27.79 & 42.63\red{16.05} & 72.30 & 67.77 & 70.04\blue{10.21} & 52.99 & 64.35 & 74.06 & 69.25 & 16.48 & 82.14 & 59.88\blue{8.97} \\
\texttt{s2m-t-c} & 43.60 & \textbf{58.23} & 42.39 & 28.00 & 43.06\red{16.48} & 70.92 & 65.12 & 68.02\blue{12.23} & 52.28 & 59.47 & 69.08 & 62.67 & 16.85 & 80.61 & 56.83\blue{12.02} \\
\rowcolor{gray!10}
\texttt{s2m-g-i} & 50.51 & 53.70 & 46.85 & 37.13 & 47.05\red{20.47} & 77.78 & 80.58 & 79.18\blue{1.07} & 60.99 & 68.25 & 76.65 & \textbf{72.59} & 34.78 & \textbf{85.05} & 66.39\blue{2.46} \\
\hline
\rowcolor{gray!10}
\ourmethod & \textbf{54.95} & 50.07 & \textbf{54.92} & \textbf{50.14} & \textbf{52.52}\red{25.94} & \textbf{81.16} & \textbf{83.32} & \textbf{82.24}\red{1.99} & \textbf{61.90} & \textbf{77.80} & \textbf{78.41} & 71.89 & \textbf{44.43} & 84.68 & \textbf{69.85}\red{1.00} \\
\Xhline{1.5pt}
% Block 9: B3-2B
B3-2B & 38.41 & 31.80 & 45.23 & 45.10 & 40.14 & 78.56 & 74.87 & 76.72 & 51.75 & 66.86 & 70.43 & 45.73 & 36.69 & 77.81 & 58.21 \\
\hline
\rowcolor{gray!10}
\texttt{single2multi} & 36.98 & \textbf{45.67} & 23.09 & 18.57 & 31.08\blue{9.06} & 67.88 & 60.70 & 64.29\blue{12.43} & 42.15 & 56.31 & 60.69 & 51.81 & 11.43 & 74.38 & 49.46\blue{8.75} \\
\texttt{s2m-t-c} & 34.49 & 41.79 & 18.12 & 20.80 & 28.80\blue{11.34} & 65.79 & 58.75 & 62.27\blue{14.45} & 39.49 & 53.55 & 56.26 & 43.83 & 11.33 & 71.09 & 45.93\blue{12.28} \\
\rowcolor{gray!10}
\texttt{s2m-g-i} & 40.64 & 44.71 & 33.32 & 38.76 & 39.36\blue{0.78} & 75.41 & 77.59 & 76.50\blue{0.22} & \textbf{51.45} & \textbf{64.51} & 62.71 & \textbf{55.60} & 30.12 & 78.04 & 57.07\blue{1.14} \\
\hline
\rowcolor{gray!10}
\ourmethod & \textbf{45.03} & 39.23 & \textbf{48.20} & \textbf{49.30} & \textbf{45.44}\red{5.30} & \textbf{79.98} & \textbf{80.79} & \textbf{80.39}\red{3.67} & 51.00 & 62.94 & \textbf{67.06} & 53.86 & \textbf{32.02} & \textbf{80.01} & \textbf{57.82}\blue{0.39} \\
\Xhline{1.5pt}
% Block 10: B3-7B
B3-7B & 47.29 & 44.81 & 50.84 & 48.05 & 47.75 & 82.07 & 82.26 & 82.17 & 65.83 & 76.77 & 84.54 & 68.55 & 52.86 & 85.75 & 72.38 \\
\hline
\rowcolor{gray!10}
\texttt{single2multi} & 45.33 & 52.80 & 39.17 & 30.33 & 41.91\blue{5.84} & 77.57 & 66.96 & 72.27\blue{9.90} & 54.91 & 67.44 & 72.97 & 68.97 & 17.68 & 82.51 & 60.75\blue{11.63} \\
\texttt{s2m-t-c} & 42.29 & 50.60 & 34.12 & 25.95 & 38.24\blue{9.51} & 76.48 & 66.12 & 71.30\blue{10.87} & 52.95 & 61.93 & 71.13 & 62.39 & 18.19 & 81.25 & 57.97\blue{14.41} \\
\rowcolor{gray!10}
\texttt{s2m-g-i} & 49.68 & \textbf{53.53} & 48.27 & 38.41 & 47.47\blue{0.28} & 81.14 & 79.17 & 80.16\blue{2.01} & 64.55 & 71.83 & 77.07 & 72.23 & 39.18 & 84.84 & 68.28\blue{4.10} \\
\hline
\rowcolor{gray!10}
\ourmethod & \textbf{53.72} & 49.50 & \textbf{52.40} & \textbf{50.15} & \textbf{51.44}\red{3.69} & \textbf{83.15} & \textbf{83.60} & \textbf{83.38}\red{1.21} & \textbf{65.97} & \textbf{75.92} & \textbf{80.19} & \textbf{73.87} & \textbf{48.05} & \textbf{86.13} & \textbf{71.69}\blue{0.69} \\
\Xhline{1.5pt}
\end{longtable}
\end{center}

\twocolumn
\clearpage

\subsection{Hyperparameter Study}
\label{app:hyperparameter_study}

\subsubsection{Effect of Balancing Factor}
\label{app:hyperparameter_balancing_factor}

\begin{figure}[!ht]
  \centering
  \includegraphics[width=1\linewidth]{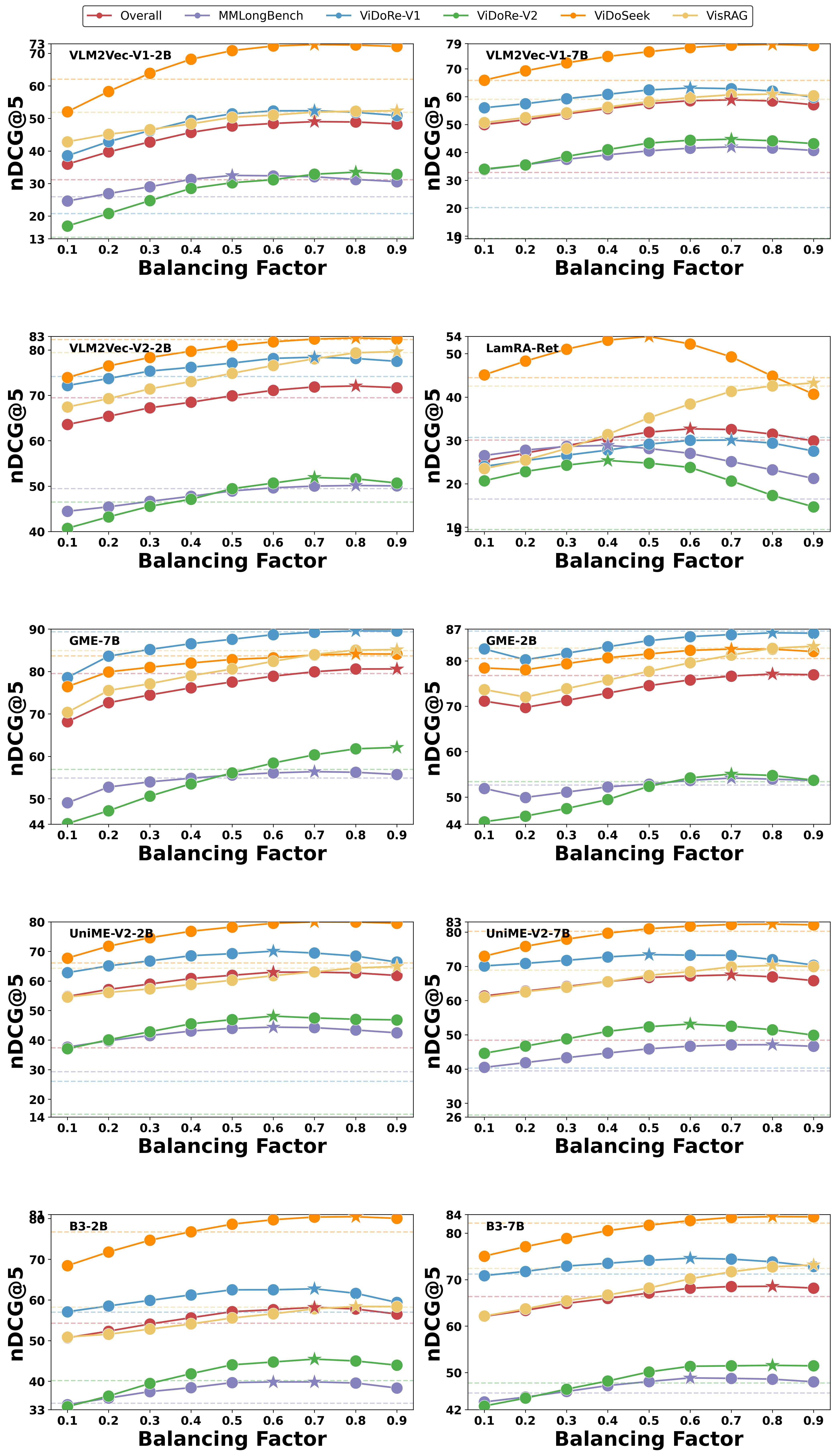}
  % \vspace{-2em}
  \caption{The comparison of the \textit{model-level} performance using \ourmethod across difference balancing factors. The dash lines refer to the base results; and the star points refer to the best-performing balancing factors. }
   \label{fig:balancing_factor_all_line}
   \vspace{-1em}
\end{figure}

\subsubsection{Effect of Document Parsing Model}
\label{app:hyperparameter_document_parsing_model}

\begin{table}[!h]
\caption{Inference efficiency comparison of MinerU2.5. The results for MinerU2.5 and baselines are tested on the A100(80G) machine. The best and runner-up results in each column are \textbf{bolded} and \underline{underlined}, respectively.}
\label{tab:mineru25_efficiency_comparison_updated}
\renewcommand\tabcolsep{7.5pt}
\renewcommand\arraystretch{1.1}
\centering
\footnotesize 
\begin{tabular}{l c c c} 
    \Xhline{1.2pt}
    \rowcolor{CadetBlue!20} 
    \textbf{Model} & \textbf{Para.} & \textbf{Tokens/sec} & \textbf{Pages/sec} \\
    \Xhline{1pt}
    MinerU2-VLM & 0.9B & \textbf{3089} & \textbf{2.84} \\
    \rowcolor{gray!10}
    dots.ocr & 3.0B & 311 & 0.28 \\
    MonkeyOCR-pro-3B & 3.7B & 520 & 0.47 \\
    \rowcolor{gray!10}
    MonkeyOCR-pro-1.2B & 1.9B & 589 & 0.53 \\
    Nanonets-OCR-s & 3.7B & 605 & 0.55 \\
    \hline
    \rowcolor{gray!10}
    MinerU2.5 & 1.2B & \underline{2422} & \underline{2.25} \\
    \Xhline{1.2pt}
\end{tabular}
\end{table}

\begin{table*}[!ht]
\caption{Document parsing performance on OmniDocBench \cite{ouyang2025omnidocbench} across multiple tasks. {The best and runner-up results are \textbf{bolded} and \underline{underlined}, respectively.}}
\label{tab:mineru25_performance_comparison}
\renewcommand\tabcolsep{5.5pt}
\renewcommand\arraystretch{1.1}
\centering
\footnotesize 
\begin{tabular}{l|ccccccc} 
    \Xhline{1.2pt}
    \rowcolor{CadetBlue!20} 
    \textbf{Methods} & \textbf{Para.} & \textbf{Overall}↑ & \textbf{Text\textsuperscript{Edit}}↓ & \textbf{Formula\textsuperscript{CDM}}↑ & \textbf{Table \textsuperscript{TEDS}}↑ & \textbf{Table \textsuperscript{TEDS-S}}↑ & \textbf{Read Order\textsuperscript{Edit}}↓ \\
    \Xhline{1pt}
    
    MinerU2-VLM & 0.9B & $85.56$ & $0.078$ & $80.95$ & $83.54$ & $87.66$ & $0.086$ \\
    \rowcolor{gray!10} dots.ocr & 3B & $88.41$ & \underline{$0.048$} & $83.22$ & \underline{$86.78$} & $90.62$ & \underline{$0.053$} \\
    MonkeyOCR-pro-3B & 3.7B & \underline{$88.85$} & $0.075$ & \underline{$87.25$} & \underline{$86.78$} & \underline{$90.63$} & $0.128$ \\
    \rowcolor{gray!10} MonkeyOCR-pro-1.2B & 1.9B & $86.96$ & $0.084$ & $85.02$ & $84.24$ & $89.02$ & $0.130$ \\
    Nanonets-OCR-s & 3.7B & $85.59$ & $0.093$ & $85.90$ & $80.14$ & $85.57$ & $0.108$ \\
    \Xhline{1.2pt}
    \hline
    \rowcolor{gray!10} \textbf{MinerU2.5} & 1.2B & $\mathbf{90.67}$ & $\mathbf{0.047}$ & $\mathbf{88.46}$ & $\mathbf{88.22}$ & $\mathbf{92.38}$ & $\mathbf{0.044}$ \\
    \Xhline{1.2pt}
    
\end{tabular}
\end{table*}

The evaluation results demonstrate the superior balance of efficiency and accuracy achieved by MinerU2.5 compared to existing specialized vision-language models. \autoref{tab:mineru25_efficiency_comparison_updated} quantifies inference efficiency on A100 (80G) hardware, utilizing Tokens/sec to measure generation speed and Pages/sec to evaluate end-to-end throughput. The findings show that while the 0.9B MinerU2-VLM maintains the highest processing speed, MinerU2.5 serves as the runner-up with 2422 tokens/s and 2.25 pages/s, both of which significantly outperform 3B-parameter baselines such as dots.ocr and MonkeyOCR-pro. Simultaneously, Table \autoref{tab:mineru25_performance_comparison} benchmarks parsing accuracy across multiple categories on OmniDocBench, employing composite Overall scores, Edit Distances for text and reading order, and structural similarity metrics (CDM and TEDS) for formulas and tables. MinerU2.5 achieves state-of-the-art performance across all six indicators, recording a top Overall score of 90.67 and the lowest error rates in text and layout recognition. MonkeyOCR-pro-3B and dots.ocr alternate as runner-up models across structural and textual tasks, yet MinerU2.5 remains the only method to consistently lead in every evaluated dimension of document parsing quality.

\subsubsection{Efficiency Analysis}
\label{app:efficiency_analysis}

\autoref{tab:parsed_vector_num_statistics} summarizes the average number of parsed vectors per document across 24 datasets in five VDR benchmarks.

\begin{table*}[htbp]
\centering
\caption{Summary of the average number of parsed vectors per document across 24 datasets in five VDR benchmarks. The values represent the number of layout-informed sub-image embeddings ($k$) generated by the document parser (MinerU2.5).}
\label{tab:parsed_vector_num_statistics}
\small % 稍微缩小字体以适应页面
\begin{tabular}{llc}
\toprule
% --- 设置表头背景颜色 ---
\rowcolor{CadetBlue!20} \textbf{Benchmark} & \textbf{Dataset} & \textbf{Avg. \#Vectors ($k$)} \\ 
\midrule

% MMLongBench
\multirow{2}{*}{MMLongBench} & MMLongBench-doc & 6.04 \\
                             & MMLongBench-page & 6.04 \\
\midrule

% ViDoSeek
\multirow{2}{*}{ViDoSeek}    & ViDoSeek-doc & 5.60 \\
                             & ViDoSeek-page & 5.77 \\
\midrule

% VisRAG
\multirow{6}{*}{VisRAG}      & VisRAG\_ArxivQA & 1.97 \\
                             & VisRAG\_ChartQA & 2.98 \\
                             & VisRAG\_InfoVQA & 4.20 \\
                             & VisRAG\_MP-DocVQA & 5.82 \\
                             & VisRAG\_PlotQA & 2.06 \\
                             & VisRAG\_SlideVQA & 4.66 \\
\midrule

% ViDoRe-v2
\multirow{4}{*}{ViDoRe-v2}   & ViDoRe\_biomedical\_lectures\_v2 & 3.88 \\
                             & ViDoRe\_economics\_reports\_v2 & 5.89 \\
                             & ViDoRe\_esg\_reports\_human\_labeled\_v2 & 6.92 \\
                             & ViDoRe\_esg\_reports\_v2\_multilingual & 6.91 \\
\midrule

% ViDoRe-V1
\multirow{10}{*}{ViDoRe-V1}  & ViDoRe\_arxiviva & 1.99 \\
                             & ViDoRe\_docvqa & 5.64 \\
                             & ViDoRe\_infovqa & 4.52 \\
                             & ViDoRe\_shiftproject & 8.97 \\
                             & ViDoRe\_syntheticDocQA\_artificial\_intelligence & 5.71 \\
                             & ViDoRe\_syntheticDocQA\_energy & 5.07 \\
                             & ViDoRe\_syntheticDocQA\_government\_reports & 5.98 \\
                             & ViDoRe\_syntheticDocQA\_healthcare\_industry & 6.11 \\
                             & ViDoRe\_tabfquad & 2.10 \\
                             & ViDoRe\_tatdqa & 8.58 \\
\bottomrule
\end{tabular}
\end{table*}

%%%%%%%%%%%%%%%%%%%%%%%%%%%%%%%%%%%%%%%%%%%%%%%%%%%%%%%%%%%%%%%%%%%%%%%%%%%%%%%
%%%%%%%%%%%%%%%%%%%%%%%%%%%%%%%%%%%%%%%%%%%%%%%%%%%%%%%%%%%%%%%%%%%%%%%%%%%%%%%

\end{document}